\documentclass{article}

    \usepackage[final]{neurips_2023}
\usepackage{soul}
\usepackage{hyperref}
\usepackage{url}
\usepackage{booktabs}       % professional-quality tables
\usepackage{amsfonts}       % blackboard math symbols
\usepackage{nicefrac}       % compact symbols for 1/2, etc.
\usepackage{microtype}      % microtypography
\usepackage{xcolor}         % colors

\usepackage[utf8]{inputenc} % allow utf-8 input
\usepackage[T1]{fontenc}    % use 8-bit T1 fonts

\usepackage{graphicx}
\usepackage{amsthm}
\usepackage{amssymb}
\usepackage{amsmath}
\usepackage{mathtools}
\usepackage{dirtytalk}
\usepackage{subcaption}
\usepackage{wrapfig}
\usepackage{multirow}
\usepackage{placeins}
\usepackage[noend,ruled,vlined,linesnumbered]{algorithm2e}  % algo2e,

\SetCommentSty{mycommfont}

\usepackage{color}
\usepackage{soul}
\usepackage{ulem}
\usepackage[capitalise]{cleveref}

\theoremstyle{definition}
\newtheorem{definition}{Definition}%[section]
\newtheorem{assumption}{Assumption}
\newtheorem{theorem}{Theorem}
\newtheorem{proposition}{Proposition}
\newtheorem{lemma}{Lemma}%[section]
\newtheorem{corollary}{Corollary}%[section]

\newcommand\cleanversion{}

\ifdefined\cleanversion
    \newcommand{\ra}[1]{#1}

    \newcommand{\rd}[1]{}
    
    \newcommand\eli[1]{}
    \newcommand\elirm[1]{}
    \newcommand\ido[1]{}
\else
    \newcommand{\ra}[1]{\hl{#1}}

    \newcommand{\rd}[1]{{\textcolor{red}{\sout{{#1}}}}}
    
    \definecolor{blue}{rgb}{0.435, 0.659, 0.863}
    \newcommand\eli[1]{\textcolor{blue}{\textbf{EM:} #1 }}
    \newcommand\elirm[1]{{\st{\textcolor{blue}{\textbf{EM:} #1 }}}}
    \definecolor{red}{rgb}{0.95, 0.5, 0.5}
    \newcommand\ido[1]{\textcolor{red}{\textbf{IG:} #1 }}
\fi

\newcommand{\repobasemaml}{https://github.com/ido90/RoML-maml}%{https://github.com/fictivename/RoML-maml}
\newcommand{\repobasepearl}{https://github.com/ido90/RoML-pearl}%{https://github.com/fictivename/RoML-pearl}
\newcommand{\repobasecesor}{https://github.com/fictivename/CeSoR_PPO}
\newcommand{\repobasepaired}{https://github.com/fictivename/dcd_mujoco}
\newcommand{\repobase}{https://github.com/ido90/RobustMetaRL}%{https://github.com/fictivename/RoML-varibad}
\newcommand{\repomaml}[1]{\href{\repobasemaml}{\underline{#1}}}
\newcommand{\repopearl}[1]{\href{\repobasepearl}{\underline{#1}}}
\newcommand{\repocesor}[1]{\href{\repobasecesor}{\underline{#1}}}
\newcommand{\repopaired}[1]{\href{\repobasepaired}{\underline{#1}}}
\newcommand{\repo}[1]{\href{\repobase}{\underline{#1}}}

\newcommand{\cempypi}[1]{\href{https://pypi.org/project/cross-entropy-method/}{\underline{#1}}}

\newcommand{\mytitle}{Train Hard, Fight Easy:\\Robust Meta Reinforcement Learning}

\definecolor{bright}{rgb}{0.8, 0.1, 0}

\title{\mytitle}

\author{
  Ido Greenberg \\ %\thanks{Use footnote for providing further information about author (webpage, alternative address)---\emph{not} for acknowledging funding agencies.} \\
  Technion, Nvidia Research \\
  \texttt{gido@campus.technion.ac.il}
  \And
  Shie Mannor \\
  Technion, Nvidia Research \\
  \texttt{shie@ee.technion.ac.il}
  \And
  Gal Chechik \\
  Bar Ilan University, Nvidia Research \\
  \texttt{gchechik@nvidia.com}
  \And
  Eli Meirom \\
  Nvidia Research \\
  \texttt{emeirom@nvidia.com}
}

\begin{document}

\maketitle

\begin{abstract}
A major challenge of reinforcement learning (RL) in real-world applications is the variation between environments, tasks or clients.
Meta-RL (MRL) addresses this issue by learning a meta-policy that adapts to new tasks.
Standard MRL methods optimize the average return over tasks, but often suffer from poor results in tasks of high risk or difficulty. This limits system reliability since test tasks are not known in advance.
In this work, we define a robust MRL objective with a controlled robustness level.
Optimization of analogous robust objectives in RL is known to lead to both \textbf{biased gradients} and \textbf{data inefficiency}.
We prove that the gradient bias disappears in our proposed MRL framework.
The data inefficiency is addressed via the novel Robust Meta RL algorithm (\textbf{\textit{RoML}}).
% We prove that the former disappears in MRL, and address the latter via the novel Robust Meta RL algorithm (\textbf{\textit{RoML}}).
RoML is a meta-algorithm that generates a robust version of any given MRL algorithm, by identifying and over-sampling harder tasks throughout training.
We demonstrate that RoML achieves robust returns on multiple navigation and continuous control benchmarks. % learns substantially different meta-policies and
\end{abstract}

%%%%%%%%%%%%%%%%%%%%%%%%%%%%%%%%%%%%%%%%%%%%

\section{Introduction}
\label{sec:intro}

\begin{wrapfigure}[16]{R}{0.55\textwidth}
% \begin{figure}[t]
% \vspace{-23pt}
\vspace{-15pt}
\centering
\includegraphics[width=1.\linewidth]{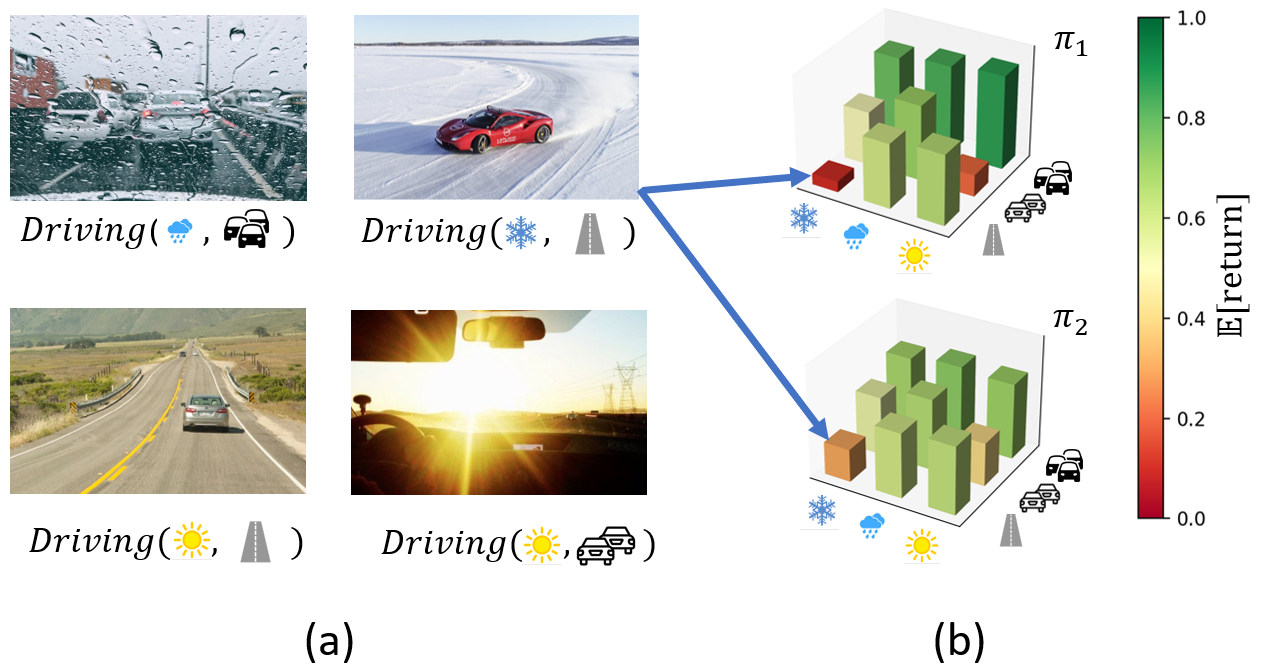}
\caption{\small (a) An illustration of driving tasks, characterized by various weather conditions and traffic density. (b) The returns of two meta-policies $\pi_1,\pi_2$ on these tasks. $\pi_1$ has a higher average return, but $\pi_2$ is more robust to high-risk tasks.
The task space is discretized only for illustration purposes.}
\vspace{-10pt}
\label{fig:problem_illustration}
% \end{figure}
\end{wrapfigure}

Reinforcement learning (RL) has achieved impressive results in a variety of applications in recent years, including cooling systems control \citep{rl_cooling} and conversational chatbots \citep{dialog_manager}.
A significant challenge in extending this success to mass production is the variation between instances of the problem, e.g., different cooling systems or different chatbot end-users.
Meta-RL (MRL) addresses this challenge by learning a ``meta-policy'' that quickly adapts to new tasks \citep{learning2learn,maml}.
In the examples above, MRL would maximize the average return of the adapted policy for a new cooling system or a new end-user.

However, optimizing the return of the average client might not suffice, as certain clients may still experience low or even negative returns.
% If 10\% of the new clients do not experience any reduction in the average cooling cost over time, it may reduce the number of current clients and discourage adoption of the technology by new ones -- even if its average return is high.
If 10\% of the clients report poor performance, it may deter potential clients from adopting the new technology -- even if its average return is high.
% or even pose a barrier to the adoption of the new technology
This highlights the need for MRL systems that provide robust returns across tasks.
% This motivates the development of MRL systems with robust returns over tasks.
% The importance of robustness is further emphasized
Robustness is further motivated by risk sensitivity in many natural RL applications, such as medical treatment~\citep{rl_healthcare} and driving~\citep{safe_rl_for_driving}.
For example, as illustrated in \cref{fig:problem_illustration}, an agent should drive safely at \textit{any} road profile -- even if at some roads the driving would be more cautious than necessary.

A common approach for risk-averse optimization is the max-min objective~\citep{task_robust}; in MRL, that would mean searching for a meta-policy with the highest expected-return in the worst possible task.
This expresses the most extreme risk aversion, which only attends to the one worst case out of all the possible outcomes. Furthermore, in certain problems, worst cases are inevitable (e.g., in certain medical treatments, a fatal outcome cannot be avoided), thus optimizing the minimum return might not provide any meaningful result.
A more general objective is the average return over the worst $\alpha$ quantiles ($0\le\alpha\le1$), also known as the Conditional Value-at-Risk (CVaR).
% : $\mathtt{CVaR}_\alpha(X) = \mathbb{E}\left[ X \,|\, X\le q_\alpha(X) \right]$ (where $q_\alpha(X)=\inf\{x\,|\,F_X(x)\ge\alpha\}$ is the $\alpha$ quantile, $X$ is a random variable and $F_X$ is its CDF). \eli{it hurts my eyes to see such definitions in the introduction. Push any formal definition of a known concept (e.g., quantile) to the method section} 
Notice that the CVaR is a generalization of both the mean (for $\alpha=1$) and the minimum (for $\alpha=0$).

CVaR is a coherent risk measure used for risk management in various fields~\citep{cvar_beyond_finance}, including banking regulation~\citep{expected_shortfall_banks} and RL~\citep{gcvar,cvar_options}.
% CVaR optimization has long been researched in RL \citep{gcvar,cvar_options}, where it is equivalent to robust optimization under uncertain perturbations~\citep{cvar_robust}.
% In RL, CVaR optimization is equivalent to robust optimization under uncertain perturbations~\citep{cvar_robust}, and has long been researched \citep{gcvar,cvar_options}.
% It is also widely researched in the context of RL \citep{gcvar,cvar_options}, where CVaR optimization is known to be equivalent to robust optimization under uncertain perturbations~\citep{cvar_robust}.
In this work, we extend CVaR optimization to MRL, by replacing the standard MRL objective
% we define the CVaR objective for the MRL problem:
% we replace the standard MRL objective,
%
\begin{equation}
\label{eq:obj_mean}
    % \arg\!\max_\pi J^\theta(R), \quad J^\theta(R) = \mathbb{E}_{\tau\sim D,\,R\sim P_\tau^\theta} [R] ,
    \arg\!\max_\pi J^\theta(R), \quad J^\theta(R) = \mathbb{E}_{\tau,R} [R] ,
\end{equation}
with a CVaR objective, which measures the robustness of a policy to high-risk tasks:
% with a CVaR objective, intended to encourage policies that are robust to high-risk tasks:
%
% \vspace{7pt}
\begin{equation}
\label{eq:obj_cvar}
    % \arg\!\max_\pi \mathbb{E}_{\tau\sim D} [V_\tau^\pi \ | \ V_\tau^\pi \le q_\alpha(V_\tau^\pi)]
    % \arg\!\max_\pi J_\alpha^\theta(R), \ \ J_\alpha^\theta(R) = \mathtt{CVaR}^\alpha_{\tau\sim D} \left[\mathbb{E}_{R\sim P_\tau^\theta}[R]\right] .
    \arg\!\max_\pi J_\alpha^\theta(R), \quad J_\alpha^\theta(R) = \mathtt{CVaR}^\alpha_{\tau} \left[\mathbb{E}_{R}[R]\right] .
\end{equation}
In both equations, $\tau$ is a random task and $R$ is the random return of policy $\pi_\theta$ in $\tau$.
\ra{Intuitively, the CVaR return expresses robustness to the selected task, in analogy to robustness to the realized model in standard RL. To further motivate} \cref{eq:obj_cvar}, note that CVaR optimization in RL is equivalent to robust optimization under uncertain perturbations~\citep{cvar_robust}.

%%%%%%%%%%%%%%%%%%%%%%%%%%%%%

In \cref{sec:cvar_mrl}, we follow the standard approach of policy gradient (PG) for $\mathtt{CVaR}_\alpha$ optimization in RL, and apply it to MRL.
That is, for every batch of $N$ trajectories, we apply the learning step to the $\alpha N$ trajectories with the lowest returns.
This standard approach, CVaR-PG, is known to suffer from a major limitation: in an actor-critic framework, the critic leads to a biased gradient estimator -- to the extent that it may point to the opposite direction \citep{gcvar}.
% \st{in contrast to mean-optimizing PG, where a critic is often used to reduce the variance of the policy gradients,}
This limitation is quite severe: many CVaR-PG implementations \citep{gcvar,cesor} rely on vanilla PG without a critic (REINFORCE, \citet{REINFORCE}); others pay the price of gradient bias -- in favor of advanced actor-critic methods that reduce the gradient variance \citep{cvar_trpo}.
% This limitation is quite severe, to the extent that certain works
% The severity of this limitation encouraged other works to violate it and use more advanced algorithms -- even at the cost of biased gradients~\citep{cvar_trpo}.

This limitation is particularly concerning in \textit{meta} RL, where high complexity and noise require more sophisticated algorithms than REINFORCE.
Fortunately, \textbf{\cref{sec:cvar_mrl} eliminates this concern:
in MRL, in contrast to RL, the CVaR policy gradient is proven to remain unbiased regardless of the choice of critic.}
% as proven in \cref{sec:cvar_mrl}, in contrast to standard RL, in MRL the CVaR policy gradients remain unbiased regardless of the choice of a critic's baseline function.
Hence, our proposed method -- CVaR Meta Learning (\textbf{\textit{CVaR-ML}}) -- can be safely applied on top of any MRL algorithm.
% SOTA algorithms for (mean-)MRL, permitting variance reduction while keeping the gradients unbiased.
This makes CVaR-ML a \textit{meta-algorithm}: given an arbitrary MRL algorithm, CVaR-ML generates a robust version of it.

Nevertheless, in CVaR optimization methods, another source of gradients variance and sample inefficiency is the large proportion of data not being utilized.
% the fact that a large proportion of the data is not utilized.
Every iteration, we rollout trajectories for $N$ tasks, but only use $\alpha N$ of them for training.
To mitigate this effect, we introduce in \cref{sec:roml} the Robust Meta RL algorithm (\textbf{\textit{RoML}}).
RoML assumes that tasks can be selected during training. It learns to identify tasks with lower returns and over-samples them. %, as illustrated in \cref{fig:roml}.
By training on high-risk tasks, the meta-agent learns policies that are robust to them without discarding data. Hence, \textbf{RoML increases the sample efficiency by a factor of up to $\pmb{\alpha^{-1}}$}.
Unlike common adversarial methods, which search for the worst-case sample (task) that minimizes the return~\citep{task_robust}, RoML lets the user specify the desired level of robustness $\alpha$, and addresses the entire $\alpha$-tail of the return distribution.
% In addition, RoML can handle infinite task spaces.

We test our algorithms on several domains.
\cref{sec:kd} considers a navigation problem, where both CVaR-ML and RoML obtain better CVaR returns than their risk-neutral baseline. Furthermore, they learn substantially different navigation policies.
\cref{sec:mujoco} considers several continuous control environments with varying tasks. These environments are challenging for CVaR-ML, which entirely fails to learn. Yet, RoML preserves its effectiveness and consistently improves the robustness of the returns.
In addition, \cref{sec:sine} demonstrates that under certain conditions, RoML can be applied to supervised settings as well -- providing robust supervised meta-learning.

As a meta-algorithm, in each experiment RoML improves the robustness of its baseline algorithm -- using the same hyper-parameters as the baseline.
The \textit{average} return is also improved in certain experiments, indicating that even the risk-neutral objective of \cref{eq:obj_mean} may benefit from robustness.

\textbf{Contribution:}
(a) We propose a principled CVaR optimization framework for robust meta-RL.
While the analogous problem in standard RL suffers from biased gradients and data inefficiency, we (b) prove theoretically that MRL is immune to the former, and (c) address the latter via the novel Robust Meta RL algorithm (RoML).
Finally, (d) we demonstrate the robustness of RoML experimentally.

%%%%%%%%%%%%%%%%%%%%%%%%%%%%%%%%%%%%%%%%%%%%%%%%%
%%%%%%%%%%%%%%%%%%%%%%%%%%%%%%%%%%%%%%%%%%%%%%%%%

\section{Related Work}
\label{sec:related_work}

%%%% MRL

\textbf{Meta-RL} for the \textbf{average task} is widely researched, including methods based on gradients \citep{maml,mrl_for_structured_exploration}, latent memory \citep{varibad,pearl} and offline meta learning \citep{offline_mrl_exploration,offline_mrl}.
It is used for applications ranging from robotics \citep{mrl_robotics} to education \citep{mrl_education}.
% In all these works, the average task is optimized.
%%%% Robust ML
Adversarial meta learning was studied for minimax optimization of the \textbf{lowest-return task}, in supervised meta learning \citep{task_robust,adversarial_meta_classification} and MRL \citep{adversarial_mrl}.
Other works studied the robustness of MRL to distributional shifts \citep{robust_off_policy_mrl,distributionally_adaptive}.
%%%%
% To the best of our knowledge,
However, the \textbf{CVaR task} objective has not been addressed yet in the framework of MRL.

% \citet{task_robust} proposed a task-sampling method for supervised meta-learning that focuses on the worst-case out of a finite set of tasks.
% \citep{robust_maml}

%%%% CVaR in RL

\textbf{Risk-averse RL.}
In \textit{standard} RL, risk awareness is widely studied for both safety~\citep{safe_RL,bfar} and robustness~\citep{bayesian_robust}.
CVaR specifically was studied using PG \citep{gcvar,cvar_trpo,cvar_options,risk_pg_convergence}, value iteration \citep{cvar_robust} and distributional RL \citep{distributional_rl_risk,automatic_risk_adaptation,drl_static_cvar}.
CVaR optimization was also shown equivalent to mean optimization under robustness \citep{cvar_robust}, motivating robust-RL methods \citep{robust_adversarial,cvar_adversarial}.
%%%%
In this work, we propose a \textit{meta-learning} framework and algorithms for CVaR optimization, and point to both similarities and differences from the standard RL setting.

\textbf{Sampling.}
In \cref{sec:roml}, we use the cross-entropy method \citep{CE_tutorial} to sample high-risk tasks for training.
The cross-entropy method has been studied in standard RL for both optimization \citep{ce_for_policy_search,cem_gd} and sampling \citep{cesor}.
Sampling in RL was also studied for regret minimization in the framework of Unsupervised Environment Design~\citep{PAIRED,robust_PLR}; and for accelerated curriculum learning in the framework of Contextual RL~\citep{self_paced_1,self_paced_2}.
% By contrast, our sampling approach is (a) dedicated for optimization of the CVaR risk measure, and (b) applicable to MRL.
By contrast, we address MRL (where the current task is unknown to the agent, unlike Contextual RL), and optimize the CVaR risk measure instead of the mean.

%%%%
% However, to the best of our knowledge, neither the CVaR objective nor the cross-entropy method were studied in the MRL framework.

%%%%%%%%%%%%%%%%%%%%%%%%%%%%%%%%%%%%%%%%%%%%%%%%%
%%%%%%%%%%%%%%%%%%%%%%%%%%%%%%%%%%%%%%%%%%%%%%%%%

\section{Preliminaries}
\label{sec:preliminaries}

% \textbf{Meta RL -- problem setup}: \eli{I suggest breaking up to parts, ... where X is XXX, y is YYY... . It is hard to do to the zip().}
\textbf{MRL.}
Consider a set of Markov Decision Processes (MDPs) $\{(S,A,\tau,\mathcal{P}_\tau,\mathcal{P}_{0,\tau},\gamma)\}_{\tau\in\Omega}$, %corresponding to states, actions, task instance, state-transition and reward distribution, initial state distribution and discount factor \eli{not sure it is not ambiguous in English}. Notice that the transitions, rewards and initialization ($P_\tau,P_{0,\tau}$) depend on the task $\tau$.
where the distribution of transitions and rewards $\mathcal{P}_\tau$ and the initial state distribution $\mathcal{P}_{0,\tau}$ both depend on task $\tau \in \Omega$.
% The task is drawn from a probability space with sample set $\Omega$ and distribution $D$.
The task itself is drawn from a distribution $\tau\sim D$ over a general space $\Omega$, and is not known to the agent.
The agent can form a belief regarding the current $\tau$ based on the task history $h$,
% of experience in the current task,
which consists of repeating triplets of states, actions and rewards \citep{varibad}.
Thus, the meta-policy $\pi_\theta(a;s,h)$ ($\theta\in\Theta$) maps the current state $s\in S$ and the history $h\in \prod (S\times A\times \mathbb R)$ (consisting of state-action-reward triplets) to a probability distribution over actions.
% Intuitively, the dependence on $h$ allows the agent to (implicitly) form a belief regarding the properties of the current task, and update the policy accordingly~\citep{varibad}.

% Given a task $\tau$,
A meta-rollout is defined as a sequence of $K \ge 1$ episodes of length $T\in\mathbb N$ over a single task $\tau$: $\Lambda = \{\{(s_{k,t},\,a_{k,t},\,r_{k,t})\}_{t=1}^T\}_{k=1}^K$.
For example, in a driving problem, $\tau$ might be a geographic area or type of roads, and $\Lambda$ a sequence of drives on these roads.
The return of the agent over a meta-rollout is defined as $R(\Lambda) = \frac{1}{K}\sum_{k=1}^K \sum_{t=0}^T \gamma^{t} r_{k,t}$, where $r_{k,t}$ is the (random variable) reward at step $t$ in episode $k$.
Given a task $\tau$ and a meta-policy $\pi_\theta$, we denote by $P_\tau^\theta(x)$ the conditional PDF of the return $R$.  % f_R(x\,|\,\tau,\pi_\theta)\textbf{}
With a slight abuse of notation, we shall use $P_\tau^\theta(\Lambda)$ to also denote the PDF of the meta-rollout itself.
The standard MRL objective is to maximize the expected return $J^\theta(R) = \mathbb{E}_{\tau,R} [R]$. % (\cref{eq:obj_mean}).
% When $\pi=\pi_\theta$, we may denote $J^\theta(R) = J^\theta(R)$.

% While the setting above is mathematically equivalent to an MDP with horizon $\tilde{T}=K\cdot T$ and state space that includes the whole history $\tilde{S}=(S\times A\times \mathbb R)^*$, the increased complexity of this presentation requires a regularization

While the meta policy $\pi_\theta(s,a;h)$ is history-dependent, it can still be learned using standard policy gradient (PG) approaches, by considering $h$ as part of an extended state space $\tilde{s}=(s,h)\in\tilde{S}$.
Then, the policy gradient can be derived directly: % from \cref{eq:obj_mean}:
\begin{align}
\label{eq:MPG}
    \nabla_\theta J^\theta(R) = \int_\Omega D(z) \int_{-\infty}^{\infty} (x-b) \nabla_\theta P_z^\theta(x) \cdot dx \cdot dz ,
    % \nabla_\theta \mathbb{E}_{\tau,\,R} [R] = \int_\Omega \int_S \int_A (R-b) \cdot \nabla_\theta \pi_\theta(s,a;h) \cdot da \cdot d\rho^\pi(s) \cdot dD(\tau) ,
    % \nabla_\theta \hat{J}_\alpha(\{\tau_i\}_{i=1}^N;\, \pi_\theta) = \frac{1}{\alpha N} \sum_{i=1}^N w_i \cdot \pmb{1}_{R(\tau_i)\le \hat{q}_\alpha} \left(R(\tau_i)-\hat{q}_\alpha\right) \sum_{t=0}^T\nabla_\theta\log \pi_\theta(a_{i,t};s_{i,t}),
\end{align}
where $D$ is a probability measure over the task space $\Omega$, $P_z^\theta(x)$ is the PDF of $R$ (conditioned on $\pi_\theta$ and $\tau=z$), and $b$ is any arbitrary baseline that is independent of $\theta$~\citep{finite_pg}.  % =f_R(x\,|\,\tau=z,\,\pi_\theta)
% Note that \cref{eq:MPG} integrates over the returns directly rather than through states and actions.  % that's kind of more standard for finite mdp...
While a direct gradient estimation via Monte Carlo sampling is often noisy, its variance can be reduced by an educated choice of baseline $b$. In the common actor-critic framework \mbox{\citep{A2C}}, a learned value function $b=V(s;h)$ is used.
This approach is used in many SOTA algorithms in deep RL, e.g., PPO~\citep{ppo}; and by proxy, in MRL algorithms that rely on them, e.g., VariBAD~\citep{varibad}.

A major challenge in MRL is the extended state space $\tilde{S}$, which now includes the whole task history.
Common algorithms handle the task history via a low-dimensional embedding that captures transitions and reward function \citep{varibad}; or using additional optimization steps w.r.t.~task history \citep{maml}.
Our work does not compete with such methods, but rather builds upon them: our methods operate as meta-algorithms that run on top of existing MRL baselines.

% Indeed, the important efforts in this area are independent of our work: our methods operate as meta-algorithms that run on top of an existing MRL baseline. Thus, if the baseline algorithm handles the complexity of $\tilde{S}$, so do our methods.
% % Our methods operate as meta-algorithms, and generate a robust version of a given MRL algorithm; in particular, any algorithm that handles the complexity of $\tilde{S}$ can be easily incorporated into our methods.
% Common algorithms handle the task history via a low-dimensional embedding that captures transitions and reward function \citep{varibad}; or using additional optimization steps w.r.t.~task history \citep{maml}.

% The major challenge in this PG formulation of MRL comes from the increased complexity of the extended state space $\tilde{S}$, which now includes the whole task history. Different approaches can be used to regularize the dependence on $h$, e.g., learning a low-dimensional embedding that represents the transitions and reward function~\citep{varibad}.
% This discussion is orthogonal to our work, since our methods below operate as robust meta-algorithms that can be applied on top of existing algorithms in MRL.

% Note that in another MRL framework, $\pi_\theta$ is further optimized on inference time~\citep{maml}; our algorithms below are applicable to that framework as well (see \cref{sec:sine}), but for simplicity, we limit the analysis to the setting described above.

%%%%%%%%%%%%%%%%%%%%%%%%%%%

\textbf{CVaR-PG.}
Before moving on to CVaR optimization in MRL, we first recap the common PG approach for standard (non-meta) RL.
% For a random variable $X$ with CDF $F_X$ and $\alpha$-quantile $q_\alpha(X)=\inf\{x\,|\,F_X(x)\ge\alpha\}$, the CVaR is defined as $\mathtt{CVaR}_\alpha(X) = \mathbb{E}\left[ X \,|\, X\le q_\alpha(X) \right]$.
For a random variable $X$ and $\alpha$-quantile $q_\alpha(X)$, the CVaR is defined as $\mathtt{CVaR}_\alpha(X) = \mathbb{E}\left[ X \,|\, X\le q_\alpha(X) \right]$.
% Removing the tasks $\tau \in \Omega$ from all notations, we are left with an MDP
For an MDP $(S,A,P,P_0,\gamma)$, the CVaR-return objective is $ \tilde{J}_\alpha^\theta(R) = \mathtt{CVaR}^{\alpha}_{R\sim P^\theta}[R] = \int_{-\infty}^{q_\alpha^\theta(R)} x \cdot P^\theta(x) \cdot dx $, whose corresponding policy gradient is~\citep{gcvar}:
\begin{equation}
\label{eq:PG_rl}
    \nabla_\theta \tilde{J}_\alpha^\theta(R) = \int_{-\infty}^{q_\alpha^\theta(R)} (x - q_\alpha^\theta(R)) \cdot \nabla_\theta P^\theta(x) \cdot dx.
\end{equation}
Given a sample of $N$ trajectories $\{\{(s_{i,t},\,a_{i,t})\}_{t=1}^T\}_{i=1}^N$ with returns $\{R_i\}_{i=1}^N$, the policy gradient can be estimated by~\citep{gcvar,cvar_trpo}:  % cvar_trpo
\begin{align}
\label{eq:PG_est_rl}
\begin{split}
    \nabla_\theta &\tilde{J}_\alpha^\theta(R) \approx %\\
    \frac{1}{\alpha N} \sum_{i=1}^N \pmb{1}_{R_i\le \hat{q}_\alpha^\theta} \cdot (R_i - \hat{q}_\alpha^\theta) \cdot \sum_{t=1}^T \nabla_\theta \log \pi_\theta(a_{i,t};s_{i,t}),
\end{split}
\end{align}
where $\hat{q}_\alpha^\theta$ is an estimator of the current return quantile. %, and $\nabla_\theta \log \pi_\theta(traj) = \sum_{t=0}^T \nabla_\theta \log \pi_\theta(traj_t)$.

% (e.g., to reduce variance)
Notice that in contrast to mean-PG, in CVaR optimization the baseline \textit{cannot} follow an arbitrary critic, but should  approximate the total return quantile $q_\alpha^\theta(R)$.
\citet{gcvar} showed that any baseline $b \ne q_\alpha^\theta(R)$ inserts bias to the CVaR gradient estimator, potentially to the level of pointing to the opposite direction (as discussed in \cref{app:pg_rl} and \cref{fig:pg}).
As a result, CVaR-PG methods in RL either are limited to basic REINFORCE with a constant baseline~\citep{cesor}, or use a critic for variance reduction at the cost of biased gradients~\citep{cvar_trpo}.

% (old retrospective discussion - can be deleted)
% We should highlight an elusive yet major difference between \cref{prop:unbiased} and similar results in CVaR-optimization for (non-meta) RL.
% Consider for example \citet{gcvar}, which provided an algorithm for CVaR-optimization on top of policy gradient methods (CVaR-PG).
% In policy gradient (PG), the gradients of actions are prioritized according to their value advantage $R-b$; where in general, the return $R$ can be replaced by the recursive proxy $r+V(s')$, corresponding to the value of the following state; and the baseline $b$ often corresponds to the current state value $V(s)$. This results in the temporal-difference advantage $r+V(s')-V(s)$, whose variance is lower than that of $R$, and thus denoises and accelerates the learning.
% However, \citet{gcvar} show that in CVaR-PG, unless the baseline corresponds to the $\alpha$-quantile of the total returns ($b = q_\alpha(R)$), the policy gradient becomes biased. Hence, we must use either a vanilla PG method (essentially REINFORCE, with no state-dependent baseline) or biased gradients.
% Note that the severity of the baseline limitation does bring certain algorithms to break it at the cost of biased gradients~\citep{other_cvar_pg}.

% As discussed in \cref{app:unbiased}, the constraint $b=q_\alpha(R)$ is rooted in the calculation of policy gradients over \textit{only one part} of the returns distribution (the $\alpha$-tail).

% \textbf{Sample inefficiency in CVaR optimization}:
Another major source of gradient-variance in CVaR-PG is its reduced sample efficiency: notice that \cref{eq:PG_est_rl} only exploits $\approx \alpha N$ trajectories out of each batch of $N$ trajectories (due to the term $\pmb{1}_{R_i\le \hat{q}_\alpha^\theta}$), hence results in estimation variance larger by a factor of $\alpha^{-1}$.

%%%%%%%%%%%%%%%%%%%%%%%%%%%%%%%%%%%%%%%%%%%%%%%%%
%%%%%%%%%%%%%%%%%%%%%%%%%%%%%%%%%%%%%%%%%%%%%%%%%

% \section{The Naive Method and Its Inefficiency}
\section{CVaR Optimization in Meta-Learning}
\label{sec:cvar_mrl}

In this section, we show that \textbf{unlike standard RL, CVaR-PG in MRL permits a flexible baseline \textit{without} presenting biased gradients}.
Hence, policy gradients for \textit{CVaR} objective in \textit{MRL} is substantially different from both \textit{mean}-PG in MRL (\cref{eq:MPG}) and CVaR-PG in \textit{RL} (\cref{eq:PG_rl}). 
% We are ready now to calculate the policy gradient for CVaR optimization in MRL, and derive a sample-based algorithm.

% In this section, we introduce a method for optimization of the CVaR objective of \cref{eq:obj_cvar} in MRL (\textbf{CVaR-ML}).
% This method follows the spirit of the CVaR-PG algorithm in standard (non-meta) RL: similarly to \cref{eq:PG_est_rl}, every iteration, we apply the learning rule only to the tail of the sampled batch.
% % (cite Surya Ganguli too -- learning from hardest examples is beneficial?)
% However, in MRL we consider a batch of tasks rather than a batch of trajectories. As it turns out, this distinction has a substantial effect on the derived algorithm. Specifically, when applied on top of a PG algorithm, CVaR-ML is more permissive than \cref{eq:PG_rl} in terms of the baseline that can be used. Hence, CVaR-ML is not restricted to the basic REINFORCE and can be implemented on top of any PG algorithm, including SOTA methods such as PPO~\citep{ppo}.
% Note that this makes CVaR-ML a \textit{meta-algorithm}, in the sense that given an arbitrary MRL algorithm, it yields a robust version of it for CVaR optimization.

To derive the policy gradient, we first define the policy value per task and the tail of tasks.

\begin{definition}
\label{def:omega_alpha}
The value of policy $\pi_\theta$ in task $\tau$ is denoted by $ V_\tau^\theta = \mathbb{E}_{R\sim P_\tau^\theta}[R] $. Notice that $V_\tau^\theta$ depends on the random variable $\tau$.
We define the $\alpha$-tail of tasks w.r.t.~$\pi_\theta$ as the tasks with the lowest values: $\Omega_\alpha^\theta = \{ z\in\Omega \,|\, V_z^\theta \le q_\alpha(V_\tau^\theta) \}$.
\end{definition}
% currently using $\tau$ for the task r.v. and $z$ for its value. open to suggestions (I didn't like Z/z or $\mathcal{T}/\tau$).

% We begin with \cref{theorem:mpg} for calculation of the gradient.
% To simplify integral calculations, we assume that the return distribution is continuous (i.e., without atoms), and that the meta-policy value is continuous in the task.

\begin{assumption}
\label{assumption}
To simplify integral calculations, we assume that for any $z\in\Omega$ and $\theta\in\Theta$, $R$ is a continuous random variable (i.e., its conditional PDF $P_z^\theta(x)$ has no atoms).  % f_R(x\,|\,\tau=z,\,\pi_\theta)
We also assume that $v(z) = V_z^\theta$ is a continuous function for any $\theta\in\Theta$.
\end{assumption}
% other assumptions? why did Aviv add bounded values and gradients?

\begin{theorem}[Meta Policy Gradient for CVaR]
\label{theorem:mpg}
Under \cref{assumption}, the policy gradient of the CVaR objective in \cref{eq:obj_cvar} is
\begin{equation}
\label{eq:cvar_mpg}
\nabla_\theta J_\alpha^\theta(R) = \int_{\Omega_\alpha^\theta} D(z) \int_{-\infty}^{\infty} (x-b) \nabla_\theta P_z^\theta(x) \cdot dx \cdot dz ,
\end{equation}
where $b$ is \textit{any} arbitrary baseline independent of $\theta$.
\end{theorem}
\begin{proof}[Proof intuition (the formal proof is in \cref{app:pg})]
\ra{In RL}, the CVaR objective measures the $\alpha$ lowest-return trajectories.
When the policy is updated, the cumulative probability of these trajectories changes and no longer equals $\alpha$. Thus, the new CVaR calculation must add or remove trajectories (as visualized in \cref{fig:pg} in the appendix). This adds a term in the gradient calculation, which causes the bias in CVaR-PG.
By contrast, in MRL, the CVaR measures the $\alpha$ lowest-return \textit{tasks} $\Omega_\alpha^\theta$. Since the task distribution does not depend on the policy, the probability of these tasks is not changed -- but only the way they are handled by the agent (\cref{fig:mpg}). Thus, no bias term appears in the calculation.
Note that $\Omega_\alpha^\theta$ does change throughout the meta-learning -- due to changes in task \textit{values} (rather than task probabilities); this is a different effect and is not associated with gradient bias.
% In RL, the CVaR-PG focuses on the $\alpha$ low-return trajectories and overlooks high-return trajectories.
% Hence, if the relative return $(R-b)$ were positive in the $\alpha$-tail, the gradient would increase the tail probability, i.e., the probability of unsuccessful actions. Thus, the baseline must guarantee that $(R-b)$ is negative on the tail, as visualized in \cref{fig:pg} in the appendix.
% % Hence, the value baseline $b$ must guarantee that $(R-b)$ is negative for all of the $\alpha$-tail -- otherwise the gradient would increase the tail probability, i.e., the probability of unsuccessful actions. This is visualized in \cref{fig:pg} in the appendix.
% In MRL, on the other hand, the focus is on the subset $\Omega_\alpha^\theta$ of \textit{tasks}.
% The gradient expresses the entire return distribution for these tasks, hence actions that led to high returns are always preferred over ones that led to low returns (\cref{fig:mpg}).
% The task distribution $D$ itself does not depend on $\theta$, so there cannot be a side effect of increasing the task-tail probability.
\end{proof}

According to \cref{theorem:mpg}, the CVaR PG in MRL permits \textit{any} baseline $b$.
As discussed in \cref{sec:preliminaries}, this flexibility is necessary, for example, in any actor-critic framework.

% Formally,
To estimate the gradient from meta-rollouts of the tail tasks, we transform the integration of \cref{eq:cvar_mpg} into an expectation:
\begin{corollary}
\label{corollary:mpg}
    \cref{eq:cvar_mpg} can be written as
    \begin{align}
    \begin{split}
    \label{eq:cvar_mpg2}
    &\nabla_\theta J_\alpha^\theta(R) = \mathbb{E}_{\tau \sim D}\left[ \mathbb{E}_{\Lambda \sim P_\tau^\theta} \left[ g(\Lambda) \right]
\ \Big|\ V_\tau^\theta \le q_\alpha(V_\tau^\theta) \right] ,
    \end{split}
    \end{align}
    % where $g(\Lambda) = (R(\Lambda)-b) \sum_{k=1}^K \sum_{t=1}^T \nabla_\theta \log \pi_\theta (a_{k,t};\,h_{k,t})$.
    where $g(\Lambda) = (R(\Lambda)-b) \sum_{\substack{1\le k\le K,\\1\le t\le T}} \nabla_\theta \log \pi_\theta (a_{k,t};\,\tilde{s}_{k,t})$; and $\tilde{s}_{k,t} = (s_{k,t}, h_{k,t})$ is the extended state (that includes all the task history $h_{k,t}$ until trajectory $k$, step $t$).
\end{corollary}
\begin{proof}
We apply the standard log trick $\nabla_\theta P_z^\theta = P_z^\theta \cdot \nabla_\theta \log P_z^\theta$ to \cref{eq:cvar_mpg}, after substituting the meta-rollout PDF:
$P_z^\theta(\Lambda) = \prod_{k=1}^K\left[ P_{0,z}(s_{k,0}) \cdot \prod_{t=1}^T P_z (s_{k,t+1},r_{k,t}\,|\,s_{k,t},a_{k,t}) \pi_\theta(a_{k,t};\,\tilde{s}_{k,t}) \right] $ .
% Given a task $\tau$, we write the PDF of a meta-rollout $\Lambda$:
%
% \begin{align*}
% P_z^\theta(\Lambda) = &\prod_{k=1}^K\left[ P_{0,z}(s_{k,0}) \cdot %\phantom{\prod_{t=1}^T} \right. \left.
% \prod_{t=1}^T P_z (s_{k,t+1},r_{k,t}\,|\,s_{k,t},a_{k,t}) \pi_\theta(a_{k,t};\,\tilde{s}_{k,t}) \right] .
% \end{align*}
%
% and simply apply the standard log trick $\nabla_\theta P_z^\theta = P_z^\theta \cdot \nabla_\theta \log P_z^\theta$ to \cref{eq:cvar_mpg}.
\end{proof}

For a practical Monte-Carlo estimation of \cref{eq:cvar_mpg2}, given a task $z_i$, we need to estimate whether $V_{z_i}^\theta \le q_\alpha(V_\tau^\theta)$.
To estimate $V_{z_i}^\theta$, we can generate $M$ i.i.d meta-rollouts with returns $\{R_{i,m}\}_{m=1}^M$, and calculate their average return $\hat{V}_{z_i}^\theta = R_i = \sum_{m=1}^M R_{i,m}/M$.
Then, the quantile $q_\alpha(V_\tau^\theta)$ can be estimated over a batch of tasks $\hat{q}_\alpha = q_\alpha(\{R_i\}_{i=1}^N)$.
If $\hat{V}_{z_i}^\theta \le \hat{q}_\alpha$, we use \textit{all} the meta-rollouts of $z_i$ for the gradient calculation (including meta-rollouts that by themselves have a higher return $R_{i,m} > \hat{q}_\alpha$).
% Recall that each meta-rollout includes $K$ trajectories of length $T$, where $K,T\in\mathbb{N}$ are determined by the problem setup, and the meta-rollout's return is itself the average return over the trajectories.
Notice that we use $M$ i.i.d meta-rollouts, each consisting of $K$ episodes (the episodes within a meta-rollout are \textit{not} independent, due to agent memory).

Putting it together, we obtain the sample-based gradient estimator of \cref{eq:cvar_mpg2}: %is given by \cref{eq:PG_est_mrl}:
\begin{align}
\label{eq:PG_est_mrl}
\begin{split}
    &\nabla_\theta J_\alpha^\theta(R) \approx 
    \frac{1}{\alpha N} \sum_{i=1}^N \pmb{1}_{R_i\le \hat{q}_\alpha^\theta} \sum_{m=1}^M g_{i,m}, \\
    &g_{i,m} \coloneqq (R_{i,m} - b) \sum_{k=1}^K \sum_{t=1}^T \nabla_\theta \log \pi_\theta(a_{i,m,k,t};\,\tilde{s}_{i,m,k,t}) ,  % \frac{1}{K}
\end{split}
\end{align}
where $a_{i,m,k,t},\,\tilde{s}_{i,m,k,t}$ are the action and the state-and-history at task $i$, meta-rollout $m$, trajectory $k$ and step $t$.
% Note that $h_{i,m,k,t}$ includes $(k-1)T+(t-1)$ state-action pairs, in addition to the current state (which allows us to drop the current state itself from the notation).
% where $R_i = \sum_{m=1}^M R_{i,m}/M$ is the average return of task $i$, $\hat{q}_\alpha^\theta = q_\alpha(\{R_i\})$ is the quantile estimator over tasks, and
% $$ g_{i,m} = \frac{1}{K}\sum_{k=1}^K \sum_{t=1}^T \nabla_\theta \log \pi_\theta(a_{i,m,k,t};h_{i,m,k,t}) .$$

\begin{wrapfigure}[16]{L}{0.52\textwidth}
\vspace{-1pt}
\hspace{3pt}
\begin{minipage}{\linewidth}
\begin{algorithm}[H]
\caption{CVaR Meta Learning (CVaR-ML)}
\label{algo:cvarml}
% \setstretch{1.1}
\DontPrintSemicolon
\SetAlgoNoLine
\SetNoFillComment

{\bf Input}: Meta-learning algorithm (\cref{def:mrl_baseline}); robustness level $\alpha\in(0,1]$; task distribution $D$; $N$ tasks per batch; $M$ meta-rollouts per task\;
 \BlankLine
 \While{not finished training}{
 
	\tcp{Sample tasks}
	Sample $\{z_{i}\}_{i=1}^{N} \sim D$\; \label{line-ref:sample}

    \tcp{Run meta-rollouts}
    $\{\{\Lambda_{i,m}\}_{m=1}^M\}_{i=1}^N \leftarrow \text{rollouts}(\{z_{i}\}_{i=1}^{N},\ M)$\; \label{line-ref:rollout}
    $ R_{i,m} \leftarrow \text{return}(\Lambda_{i,m}), \quad \forall i,\,m $\;

	\tcp{Compute sample quantile}
    $R_i \leftarrow \text{mean}(\{R_{i,m}\}_{m=1}^M), \quad \forall i $\; %$\in \{1,\ldots,N\}$\;
	$\hat{q}_\alpha \leftarrow \text{quantile}(\{R_i\}_{i=1}^N,\, \alpha)$\;
	
    \tcp{Meta-learning algorithm train step}
	$\mathtt{ML} \big(\ \{ \Lambda_{i,m} \ |\  R_i \le \hat{q}_\alpha, \ 1\le m\le M \} \ \big)$\; \label{line-ref:learn}
 }
\end{algorithm}
\end{minipage}
\end{wrapfigure}

The procedure described above follows the principles of CVaR-PG in (non-meta) RL, as the learning rule is only applied to the tail of the sampled batch.
However, in MRL we consider a batch of tasks rather than a batch of trajectories. As discussed in \cref{theorem:mpg} and its proof, this distinction has a substantial impact on the gradient and the resulting algorithm.
Specifically, \cref{eq:PG_est_mrl} allows for greater flexibility than \cref{eq:PG_rl}, as it permits any baseline $b$ that does not depend on $\theta$.
This allows gradient calculation using any PG algorithm, including SOTA methods such as PPO \citep{ppo} (which are already used in MRL methods such as VariBAD~\citet{varibad}).
% This is in contrast to being restricted to basic REINFORCE methods, as in CVaR optimization in standard RL.
Therefore, in contrast to standard RL, CVaR optimization is not restricted to basic REINFORCE.
% This paves the way to a gradient calculation that is not restricted to basic REINFORCE methods and can be implemented by any PG algorithm, including SOTA methods such as PPO~\citep{ppo} (which are themselves used in MRL methods such as VariBAD~\citep{varibad}).

Our CVaR Meta Learning method (\textbf{CVaR-ML}, \cref{algo:cvarml}) leverages this property to operate as a \textit{meta-algorithm}, providing a robust version for any given baseline algorithm, such as \citet{maml,varibad}:
\begin{definition}%[Baseline MRL algorithm]
\label{def:mrl_baseline}
    A \textit{baseline MRL algorithm} learns a meta-policy $\pi_\theta$ using a training step $\mathtt{ML}$. Given a batch of meta-rollouts $\{\Lambda_{i}\}$, $\mathtt{ML}(\{\Lambda_{i}\})$ updates $\pi_\theta$.
    %list of tasks $Z=\{z_i\}_{i=1}^N$, $ML(Z)$ uses $\pi_\theta$ to run a meta-rollout of $M$ trajectories for each task. Then it updates $\pi_\theta$ and returns the meta-rollouts returns $\{\{R_{i,m}\}_{m=1}^M\}_{i=1}^N$.
\end{definition}

% CVaR-ML receives an arbitrary meta-learning algorithm as an input, and yields a robust version of it for CVaR optimization.
% In that sense, CVaR-ML is a \textit{meta-algorithm}.
Notice that CVaR-ML only handles task filtering, and uses the baseline training step $\mathtt{ML}$ as a black box (\cref{line-ref:learn}).
Hence, it can be used with \textit{any} MRL baseline -- not just PG methods.
% while the analysis above focuses on PG methods, CVaR-ML can be fed with \textit{any} meta-RL baseline.
In fact, by using a supervised meta-learning baseline, CVaR-ML can be applied to the supervised setting as well with minimal modifications, namely, replacing meta-rollouts with examples and returns with losses.
% In fact, by feeding a supervised meta-learning algorithm, CVaR-ML can be applied to the supervised setting as well under minimal modifications (namely, replacing the meta-rollouts and the returns with examples and losses, respectively).

%%%%%%%%%%%%%%%%%%%%%%%%%%%%%%%%%%%%%%%%%%%%%%%%%
%%%%%%%%%%%%%%%%%%%%%%%%%%%%%%%%%%%%%%%%%%%%%%%%%

\section{Efficient CVaR-ML}
\label{sec:roml}

% \begin{wrapfigure}{r}{0.41\textwidth}
\begin{figure}[b]
\centering
% \vspace{-13pt}
\begin{minipage}{0.56\textwidth}  % 0.43
% \begin{subfigure}{.86\linewidth}
  \centering
  \includegraphics[width=1.\linewidth]{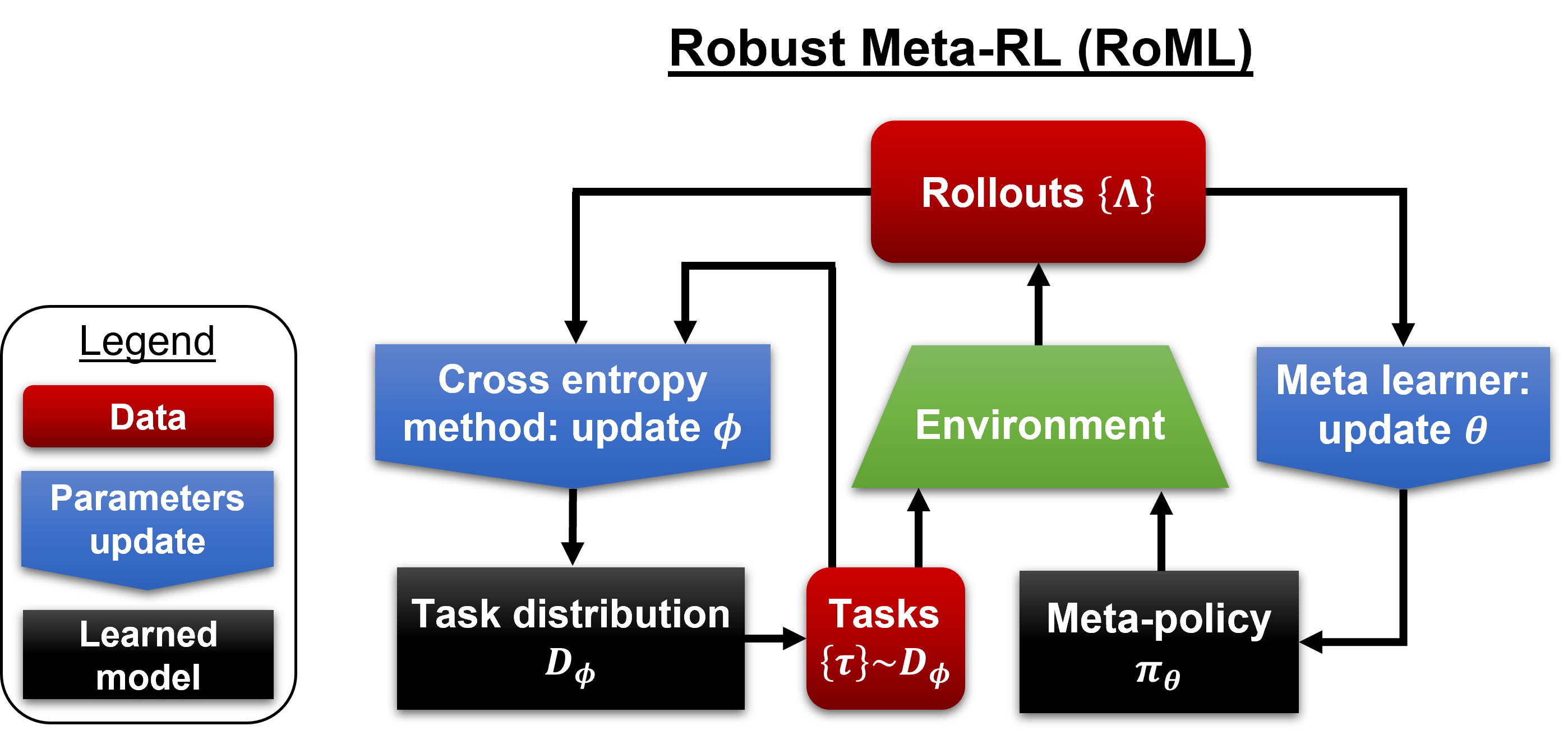}
  %\caption{}
  \label{fig:roml_diagram}
% \end{subfigure} \\ \vspace{-6pt}
\end{minipage} \hspace{10pt}
\begin{minipage}{0.37\textwidth}
% \begin{subfigure}{.58\linewidth}
  \centering
  \includegraphics[width=1.\linewidth]{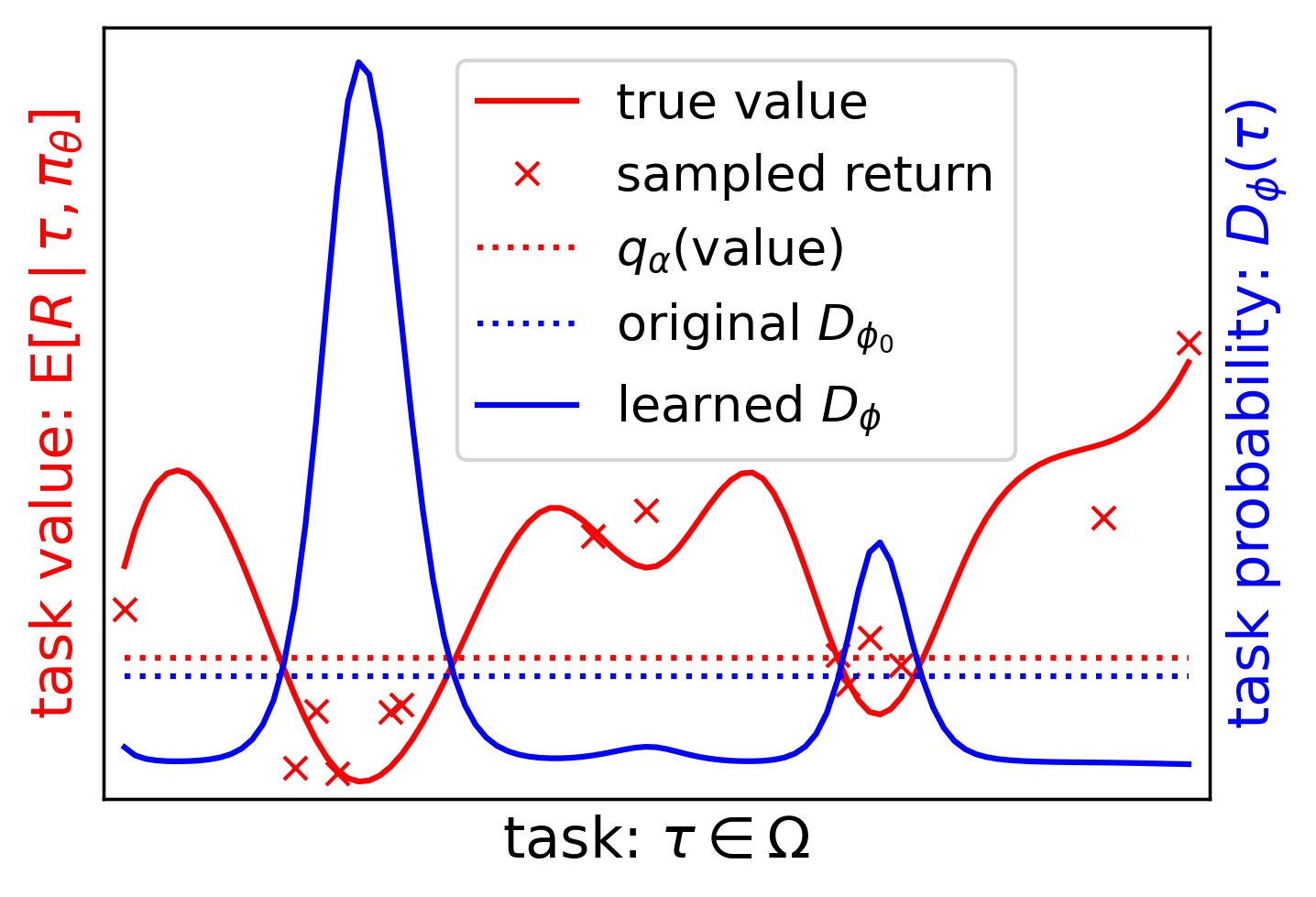}
  %\caption{}
  \label{fig:roml_illustration}
% \end{subfigure}
\end{minipage}
% \vspace{-20pt}
\caption{\small \textbf{Left}: RoML uses the cross entropy method to modify the task distribution $D_\phi$, which is used to generate the next meta-rollouts. \textbf{Right}: illustration of an arbitrary point of time in training: the task distribution $D_\phi$ (blue) is learned according to the task values of the current meta-policy $\pi_\theta$ (red). Since low-return tasks are over-sampled, the learned meta-policy is more robust to the selection of task.}
% legend: blue = algorithmic module; black = learned model; red = data
\label{fig:roml}
\end{figure}
% \end{wrapfigure}

\cref{theorem:mpg} guarantees unbiased gradients when using \cref{algo:cvarml}; however, it does not bound their variance.
% While \cref{theorem:mpg} guarantees unbiased gradients when using the actor-critic framework in \cref{algo:cvarml}, there is no guarantee for the variance of the gradients.
In particular, \cref{line-ref:learn} applies the learning step to a subset of only $\alpha N M$ meta-rollouts out of $N M$, which increases the estimator variance by a factor of $\alpha^{-1}$ compared to mean optimization.
This could be prevented if we knew the set $\Omega_\alpha^\theta$ of tail-tasks (for the current $\pi_\theta$), and sampled only these tasks, using the distribution $D_\alpha^\theta(z) = \alpha^{-1} \pmb{1}_{V_z^\theta \le q_\alpha(V_\tau^\theta)} D(z)$.
% Intuitively, we could increase our sample efficiency back by a factor of $\alpha^{-1}$.
\cref{prop:efficiency} shows that this would indeed recover the sample efficiency. %back by a factor of $\alpha^{-1}$.
% formalizes this intuition.

\begin{proposition}[Variance reduction]
\label{prop:efficiency}
Denote by $G$ the estimator of $\nabla_\theta J_\alpha^\theta(R)$ in \cref{eq:PG_est_mrl}, assume there is no quantile error ($\hat{q}_\alpha^\theta = q_\alpha^\theta$), and denote $\mathbb{E}_{D}[\cdot] = \mathbb{E}_{z_i\sim D, \, R_{i,m}\sim P_{z_i}^\theta}[\cdot]$.
Then, switching the task sample distribution to $D_\alpha^\theta$ leads to a variance reduction of factor $\alpha$:
\begin{align*}
    % \label{eq:variance_reduction}
    \begin{split}
    \mathbb{E}_{D_\alpha^\theta} [\alpha G] &= \mathbb{E}_{D} [G], \quad
    \mathrm{Var}_{D_\alpha^\theta} (\alpha G) \le \alpha \cdot \mathrm{Var}_D (G) .
    \end{split}
\end{align*}
\end{proposition}
\begin{proof}[Proof sketch (the complete proof is in \cref{app:efficiency})]
We calculate the expectation and variance directly.
$G$ is proportional to $\pmb{1}_{R_i \le q_\alpha^\theta}$ (\cref{eq:PG_est_mrl}). The condition $R_i \le q_\alpha^\theta$ leads to multiplication by the probability $\alpha$ when sampled from $D$ (where it corresponds to the $\alpha$-tail); but not when sampled from $D_\alpha^\theta$ (where it is satisfied w.p.~1).
This factor $\alpha$ cancels out the ratio between the expectations of $G$ and $\alpha G$ (thus the expectations are identical) -- but not the ratio $\alpha^2$ between their variances.
% If sampling from $D$ resulted in $\alpha N$ tail-samples ($R_i\le q_\alpha^\theta$), we could simply use the known $\sim N^{-1}$ scaling of the variance of $N$ i.i.d variables' mean.
% However, the number of tail-samples is itself random (under the assumption $\hat{q}_\alpha^\theta = q_\alpha^\theta$). Thus, we calculate expectation and variance directly.
% $G$ is proportional to $\pmb{1}_{R_i \le q_\alpha^\theta}$, which decomposes to the two cases $\pmb{1}_{R_i \le q_\alpha^\theta} \in \{0,\,1\}$ according to the total probability rule.
% The term of $\pmb{1}_{R_i \le q_\alpha^\theta} = 0$ vanishes.
% The term of $\pmb{1}_{R_i \le q_\alpha^\theta} = 1$ is multiplied by the probability $\alpha$ when sampled from $D$, but not when sampled from $D_\alpha^\theta$.
% % When sampling from $D$, the total probability rule decomposes the expectation into two terms: one corresponds to the tail (multiplied by the probability $\alpha$), and the other term (multiplied by $\pmb{1}_{R_i \le q_\alpha^\theta}=0$).
% This multiplicative factor of $\alpha$ cancels out the ratio between the expected value of $G$ and $\alpha G$ (thus the expectations are identical) -- but not the ratio $\alpha^2$ between their variances.
\end{proof}

Following the motivation of \cref{prop:efficiency}, we wish to increase the number of train tasks that come from the tail distribution $\Omega_\alpha^\theta$.
To that end, we assume to have certain control over the sampling of training tasks.
This assumption is satisfied in most simulated environments, as well as many real-world scenarios. For example, when training a driver, we choose the tasks, roads and times of driving throughout training.
In this section, we propose a method to make these choices.

We begin with parameterizing the task distribution $D$: we consider a parametric family $D_\phi$ such that $D=D_{\phi_0}$. Then, we wish to modify the parameter $\phi$ so that $D_\phi$ aligns with $D_\alpha^\theta$ as closely as possible. 
% such that the sample distribution $D_\phi$ assigns more probability mass to tasks with lower returns. Specifically, we aim to align $D_\phi$ with the $\alpha$-tail of $D_{\phi_0}$, defined by
% %
% \begin{equation}
% \label{eq:D_tail}
%     D_\alpha^\pi(z) = \alpha^{-1} \pmb{1}_{V_z^\pi\le q_\alpha(V_\tau^\pi)} D(z) .
% \end{equation}
% %
To that end, we use the Cross Entropy Method (CEM,~\citet{CE_tutorial}), which searches for $\phi^*$ that minimizes the KL-divergence (i.e., cross entropy) between the two:
\begin{align}
\label{eq:CEM0}
\begin{split}
\phi^* \in \mathrm{argmin}_{\phi'} \; D_{KL}\left( D_\alpha^\theta \,||\, D_{\phi'} \right) %\\ 
&= \ \mathrm{argmax}_{\phi'} \; \mathbb{E}_{z\sim D_{\phi_0}} \left[\pmb{1}_{V_z^\theta\le q_\alpha(V_\tau^\theta)} \log D_{\phi'}(z)\right] \\
&= \ \mathrm{argmax}_{\phi'} \; \mathbb{E}_{z\sim D_{\phi}} \left[w(z)\;\pmb{1}_{V_z^\theta\le q_\alpha(V_\tau^\theta)} \log D_{\phi'}(z)\right] ,
\end{split}
\end{align}
% the constant $\alpha^{-1}$ was deleted to make room
%
where $w(z) = \frac{D_{\phi_0}(z)}{D_{\phi}(z)}$ is the importance sampling weight corresponding to $z \sim D_{\phi}$.
Note that \cref{eq:CEM0} has a closed-form solution for most of the standard families of distributions \citep{CE_tutorial}.
% Note that \cref{eq:CEM0} often reduces to a simple closed-form calculation (e.g., for the Gaussian family $D_\phi = \mathcal{N}(\mu,\sigma^2)$, $\phi^*$ is simply the mean and variance of the tail samples).

\begin{wrapfigure}[29]{L}{0.55\textwidth}
\vspace{-14pt}
\hspace{3pt}
\begin{minipage}{\linewidth}
\begin{algorithm}[H]
\caption{Robust Meta RL (RoML)}
\label{algo:roml}
% \setstretch{1.1}
\DontPrintSemicolon
\SetAlgoNoLine
\SetNoFillComment

{\bf Input}: Meta-learning algorithm (\cref{def:mrl_baseline}); robustness level $\alpha\in(0,1]$; parametric task distribution $D_{\phi}$; original parameter $\phi_0$; $N$ tasks per batch, $\nu\in[0,1)$ of them sampled from the original $D_{\phi_0}$; CEM quantile $\beta\in(0,1)$\;
 \BlankLine
 {\bf Initialize:}\;
 $\quad\ \ \phi\leftarrow {\phi_0}, \quad $
 $N_o\leftarrow \lfloor \nu N \rfloor, \quad $
 $N_s\leftarrow \lceil (1-\nu) N \rceil$\;
 \BlankLine
 \While{not finished training}{
 
	\tcp{Sample tasks}
	Sample $\{z_{o,i}\}_{i=1}^{N_o} \sim D_{{\phi_0}},\quad\{z_{\phi,i}\}_{i=1}^{N_s} \sim D_{\phi}$\; \label{line:sample}
    $z \leftarrow (z_{o,1},\ldots, z_{o,N_o}, z_{\phi,1},\ldots, z_{\phi,N_s})$\;  % \text{shuffle}
    
	\tcp{Rollouts and meta-learning step}
    $\{\Lambda_{i}\}_{i=1}^N \leftarrow \text{rollout}(\{z_{i}\}_{i=1}^{N})$\; \label{line:rollout}
    $ R_{i} \leftarrow \text{return}(\Lambda_{i}), \quad \forall i \in \{1,\ldots,N\} $\;
	$\mathtt{ML}(\{\Lambda_{i}\}_{i=1}^N)$\; \label{line:run}
	
	\tcp{Estimate reference quantile}
    $w_{o,i} \leftarrow 1, \qquad\qquad\ \forall i\in\{1,\ldots,N_o\}$\; 
    $w_{\phi,i} \leftarrow \frac{D_{\phi_0}(z_{\phi,i})}{ D_{\phi}(z_{\phi,i})}, \quad \forall i\in \{1,\ldots,N_s\}$\; \label{line:w_s}
    $w \leftarrow (w_{o,1},\ldots, w_{o,N_o}, w_{\phi,1},\ldots, w_{\phi,N_s})$\; \label{line:w_all}
    $\hat{q}_\alpha \leftarrow \text{weighted\_quantile}(\{R_{i}\},\, w,\, \alpha)$
	
	\tcp{Compute sample quantile}
	${q}_{\beta} \leftarrow \text{quantile}(\{R_{i}\},\, \beta)$\;
	
	\tcp{Update sampler}
	$q \leftarrow \max (\hat{q}_\alpha,\, {q}_{\beta})$\; \label{line:qref}
    $\phi \leftarrow \arg\!\max_{\phi^\prime}\!
    \sum_{i\leq N} \! w_i \;\pmb{1}_{R_i\le q} \log\! D_{\phi^\prime}(z_i)$\; \label{line:update}
 }
\end{algorithm}
\end{minipage}
\end{wrapfigure}

For a batch of $N$ tasks sampled from $D=D_{\phi_0}$, \cref{eq:CEM0} essentially chooses the $\alpha N$ tasks with the lowest returns, and updates $\phi$ to focus on these tasks.
This may be noisy unless $\alpha N \gg 1$.
Instead, the CEM chooses a larger number of tasks $\beta > \alpha$ for the update, where $\beta$ is a hyper-parameter.
$\phi$ is updated according to these $\beta N$ lowest-return tasks, and the next batch is sampled from $D_\phi \ne D_{\phi_0}$.
This repeats iteratively: every batch is sampled from $D_\phi$, where $\phi$ is updated according to the $\beta N$ lowest-return tasks of the former batch.
Each task return is also compared to the $\alpha$-quantile of the \textit{original} distribution $D_{\phi_0}$. If more than $\beta N$ tasks yield lower returns, the CEM permits more samples for the update step.
% Once $\beta N$ tasks or more yield lower returns than the $\alpha$-quantile of the \textit{original} distribution $D_{\phi_0}$, the CEM permits more samples for the update step.
The return quantile over $D_{\phi_0}$ can be estimated from $D_\phi$ at any point using importance sampling weights.
See more details about the CEM in \cref{app:cem_background}.

% Finding $\phi^*$ using \cref{eq:CEM0} directly from a batch of $N$ samples may be too noisy, unless $\alpha N \gg 1$.
% Instead, the CEM works iteratively with a parameter $\beta>\alpha$: first, a batch is sampled from $D_{\phi_0}$; then, every batch, $\phi$ is updated according to the $\beta$-tail of the current batch -- until the $\alpha$-tail of the original distribution is reached.
% Notice that the quantile of the original distribution can be estimated at any point using the importance sampling weights.
% See \cref{app:cem} for more details about the CEM.

% The baseline distribution $D_{\phi_0}$, whose tail is our target,
In our problem, the target distribution is the tail of $D_{\phi_0}$.
Since the tail is defined by the agent returns in these tasks,
% Since $D_{\phi_0}$ represents the agent returns,
it varies with the agent and is non-stationary throughout training.
Thus, we use the dynamic-target CEM of \citet{cross_entropy_method}.
To smooth the changes in the sampled tasks, the sampler is also regularized to always provide certain exposure to all the tasks: we force
% A potential concern is that by focusing on certain tasks, the agent might ``forget'' how to operate on previously-successful ones; to prevent this, we regularize the sampler by forcing
$\nu$ percent of every batch to be sampled from the original distribution $D = D_{\phi_0}$, and only $1-\nu$ percent from $D_\phi$.

Putting this together, we obtain the Robust Meta RL algorithm (\textbf{RoML}), summarized in \cref{algo:roml} and \cref{fig:roml}.
RoML does not require multiple meta-rollouts per update (parameter $M$ in \cref{algo:cvarml}), since it directly models high-risk tasks.
Similarly to CVaR-ML, RoML is a meta-algorithm and can operate on top of any meta-learning baseline (\cref{def:mrl_baseline}).
Given the baseline implementation, only the tasks sampling procedure needs to be modified, which makes RoML easy to implement.
% Notice that given an implementation of the baseline, RoML is usually straightforward to implement, as it only requires modification of the tasks sampling procedure.

\ra{\textbf{Limitations:}}
The CEM's adversarial tasks sampling relies on several assumptions. Future research may reduce some of these assumptions, while keeping the increased data efficiency of RoML.
% follow RoML in increasing data efficiency, while reducing some of these assumptions.

First, as mentioned above, we need at least partial control over the selection of training tasks.
This assumption is common in other RL frameworks \citep{PAIRED,robust_PLR}, and often holds in both simulations and the real world (e.g., choosing in which roads and hours to train driving).

Second, the underlying task distribution $D$ is assumed to be known, and the sample distribution is limited to the chosen parameterized family $\{D_\phi\}$. For example, if $\tau\sim U([0,1])$, the user may choose the family of Beta distributions $Beta(\phi_1,\phi_2)$ (where $Beta(1,1)\equiv U([0,1])$), as demonstrated in \cref{app:mujoco}.
The selected family expresses implicit assumptions on the task-space. For example, if the probability density function is smooth, close tasks will always have similar sample probabilities; and if the family is unimodal, high-risk tasks can only be over-sampled from around a single peak.
This approach is useful for generalization across continuous task-spaces -- where the CEM can never observe the infinitely many possible tasks. Yet, it may pose limitations in certain discrete task-spaces, if there is no structured relationship between tasks.

%%%%%%%%%%%%%%%%%%%%%%%%%%%%%%%%%%%%%%%%%%%%%%%%%
%%%%%%%%%%%%%%%%%%%%%%%%%%%%%%%%%%%%%%%%%%%%%%%%%

\section{Experiments}
\label{sec:experiments}

We implement \textbf{RoML} and \textbf{CVaR-ML} on top of two different risk-neutral MRL baselines -- \textbf{VariBAD}~\citep{varibad} and \textbf{PEARL}~\citep{pearl}.
% , and compare their test CVaR return to the corresponding baseline.
% We use \textbf{VariBAD}~\citep{varibad} and \textbf{PEARL}~\citep{pearl} as MRL baselines.
As a risk-averse reference for comparison, we use \textbf{CeSoR}~\citep{cesor}, an efficient sampling-based method for CVaR optimization in RL, implemented on top of PPO.
As another reference, we use the Unsupervised Environment Design algorithm \textbf{PAIRED}~\citep{PAIRED}, which uses regret minimization to learn robust policies on a diverse set of tasks.

\cref{sec:kd} demonstrates the mean/CVaR tradeoff, as our methods learn substantially different policies from their baseline. %, improving the CVaR at the expense of the mean return.
\cref{sec:mujoco} demonstrates the difficulty of the naive CVaR-ML in more challenging control benchmarks, and the RoML's efficacy in them.
\ra{The ablation test} in \cref{app:ablation_cem} demonstrates that RoML deteriorates significantly when the CEM is replaced by a naive adversarial sampler.
\cref{sec:sine} presents an implementation of CVaR-ML and RoML on top of \textbf{MAML}~\citep{maml} for \textit{supervised} meta-learning.
In all the experiments, the running times of RoML and CVaR-ML are indistinguishable from their baselines (RoML's CEM computations are negligible).

\ra{\textbf{Hyper-parameters:}}
To test the practical applicability of RoML as a meta-algorithm, in every experiment, \textbf{we use the same hyper-parameters for RoML, CVaR-ML and their baseline}.
In particular, we use the baseline's default hyper-parameters whenever applicable (\citet{varibad,pearl} in \cref{sec:mujoco}, and \citet{maml} in \cref{sec:sine}).
That is, we use the same hyper-parameters as originally tuned for the baseline, and test whether RoML improves the robustness without any further tuning of them.
As for the additional hyper-parameters of the meta-algorithm itself: in \cref{algo:cvarml}, we use $M=1$ meta-rollout per task; and in \cref{algo:roml}, we use $\beta=0.2,\, \nu=0$ unless specified otherwise (similarly to the CEM in \citet{cesor}).
For the references PAIRED and CeSoR, we use the hyper-parameters of \citet{PAIRED,cesor}.
Each experiment is repeated for 30 seeds.
% Each experiment is repeated for 30 seeds, using the same hyper-parameters and architectures for RoML, CVaR-ML and their baseline.
% % All experiments were repeated with 30 seeds.
% % In each experiment, CVaR-ML and RoML used the same hyper-parameters and architectures as the baseline.
% In particular, we follow the original hyper-parameters of \citet{varibad,pearl,PAIRED,cesor} (\cref{sec:mujoco}) and \citet{maml} (\cref{sec:sine}).
% In \cref{algo:cvarml} we use $M=1$ meta-rollout per task.
See more details in \cref{app:detailed_results}.
% The running times of CVaR-ML, RoML and their baselines were indistinguishable, as the additional computations of the CEM are negligible.
% In each experiment, all the algorithms were applied with the same hyper-parameters and architectures as the baseline: in \cref{sec:kd} we tuned them manually, and in \cref{sec:mujoco} and \cref{sec:sine} we followed the original hyper-parameters of \citet{maml} and \citet{varibad}, respectively.
The code is available in our repositories: \repo{VariBAD}, \repopearl{PEARL}, \repocesor{CeSoR}, \repopaired{PAIRED} and \repomaml{MAML}.

% For every method, we save the model with the best validation score throughout training.
% To assure a fair comparison, we align the validation criterion with the training objective (mean for risk-neutral methods and CVaR for risk-averse). Otherwise, selecting risk-neutral models according to CVaR validation score often led to selection of premature models with inferior returns in both mean and CVaR.

%%%%%%%%%%%%%%%%%%%%%%%%%%%%%%%%%%%%%%%%%%%%%%%%%

\subsection{Khazad Dum}
\label{sec:kd}

We demonstrate the tradeoff between mean and CVaR optimization in the Khazad Dum benchmark, visualized in \cref{fig:kd_trajs}. The agent begins at a random point in the bottom-left part of the map, and has to reach the green target as quickly as possible, without falling into the black abyss.
The bridge is not covered and thus is exposed to wind and rain, rendering its floor slippery and creating an additive action noise (\cref{fig:slippery}) -- to a level that varies with the weather.
Each task is characterized by the rain intensity, which is exponentially distributed. The CEM in RoML is allowed to modify the parameter of this exponential distribution.
Note that the agent is not aware of the current task (i.e, the weather), but may infer it from observations.
We set the target risk level to $\alpha=0.01$, and train each meta-agent for a total of $5\cdot10^6$ frames.
See complete details in \cref{app:kd}.

\begin{figure}[!t]
\centering
\begin{minipage}{0.47\textwidth}
% \begin{wrapfigure}[26]{R}{0.5\textwidth}
% \vspace{-15pt}
\centering
\begin{subfigure}{.45\linewidth}
  \centering
  \includegraphics[width=1.\linewidth]{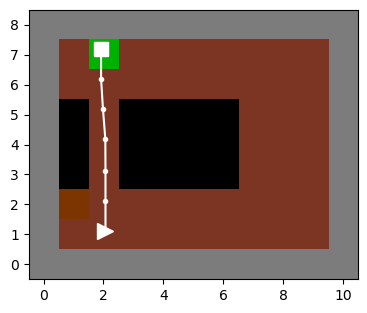}
  \caption{VariBAD}
\end{subfigure}
\begin{subfigure}{.45\linewidth}
  \centering
  \includegraphics[width=1.\linewidth]{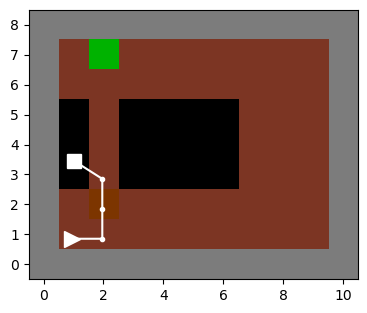}
  \caption{VariBAD}
  \label{fig:slippery}
\end{subfigure} \\
\begin{subfigure}{.45\linewidth}
  \centering
  \includegraphics[width=1.\linewidth]{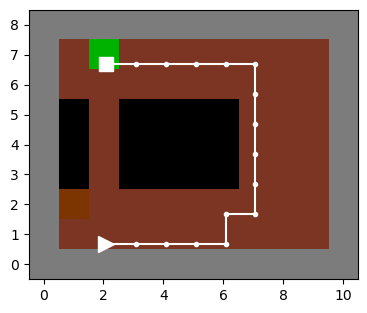}
  \caption{CVaR-ML}
\end{subfigure}
\begin{subfigure}{.45\linewidth}
  \centering
  \includegraphics[width=1.\linewidth]{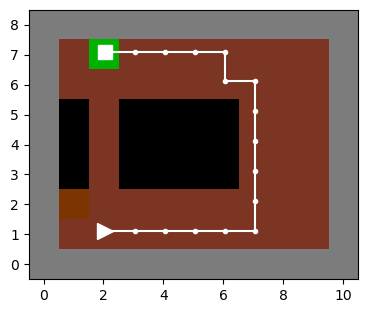}
  \caption{RoML}
\end{subfigure}
\begin{subfigure}{.49\linewidth}
  \centering
  \includegraphics[width=1.\linewidth]{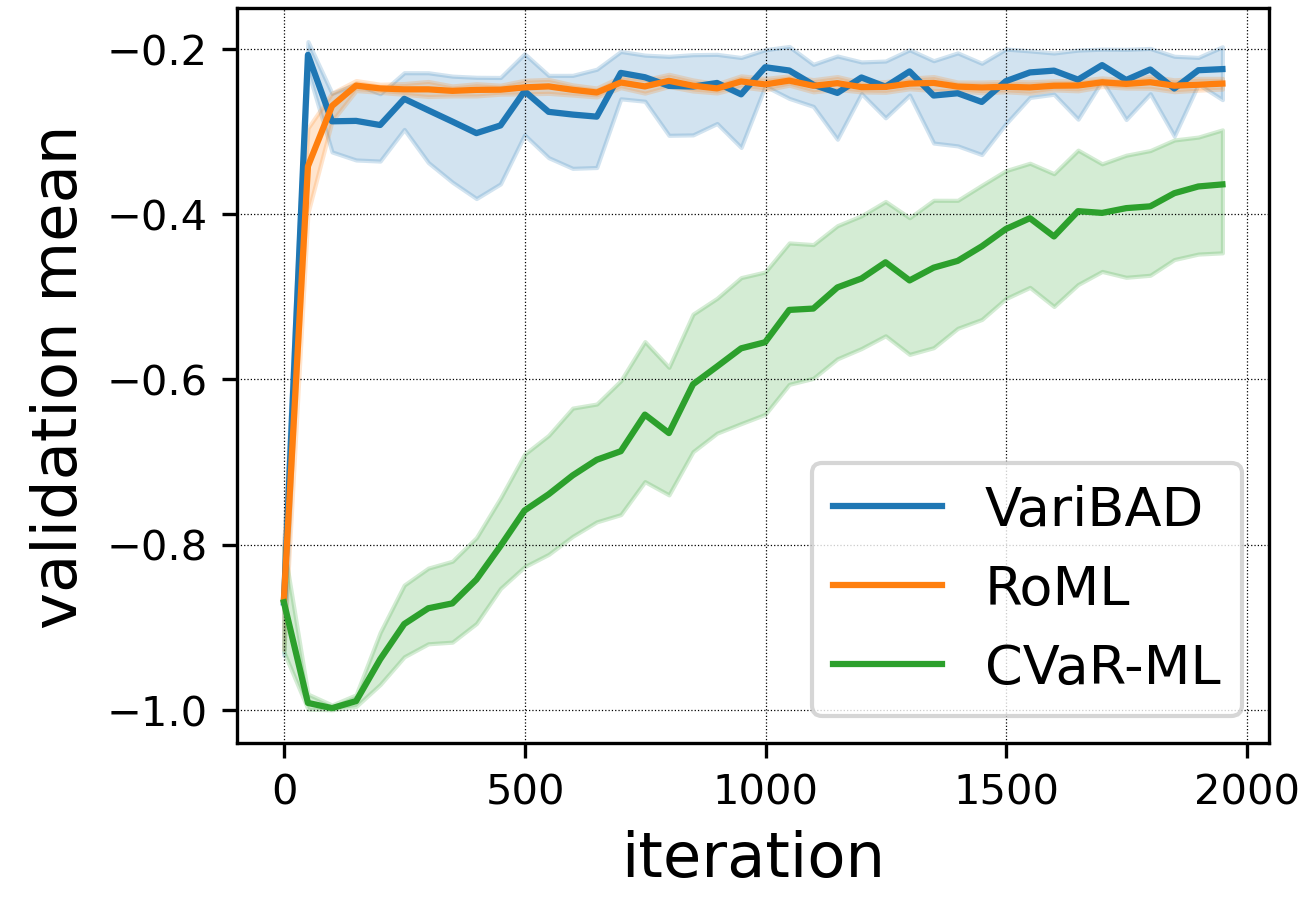}
  \caption{Mean}
  \label{fig:kd_valid_mean}
\end{subfigure}
\begin{subfigure}{.49\linewidth}
  \centering
  \includegraphics[width=1.\linewidth]{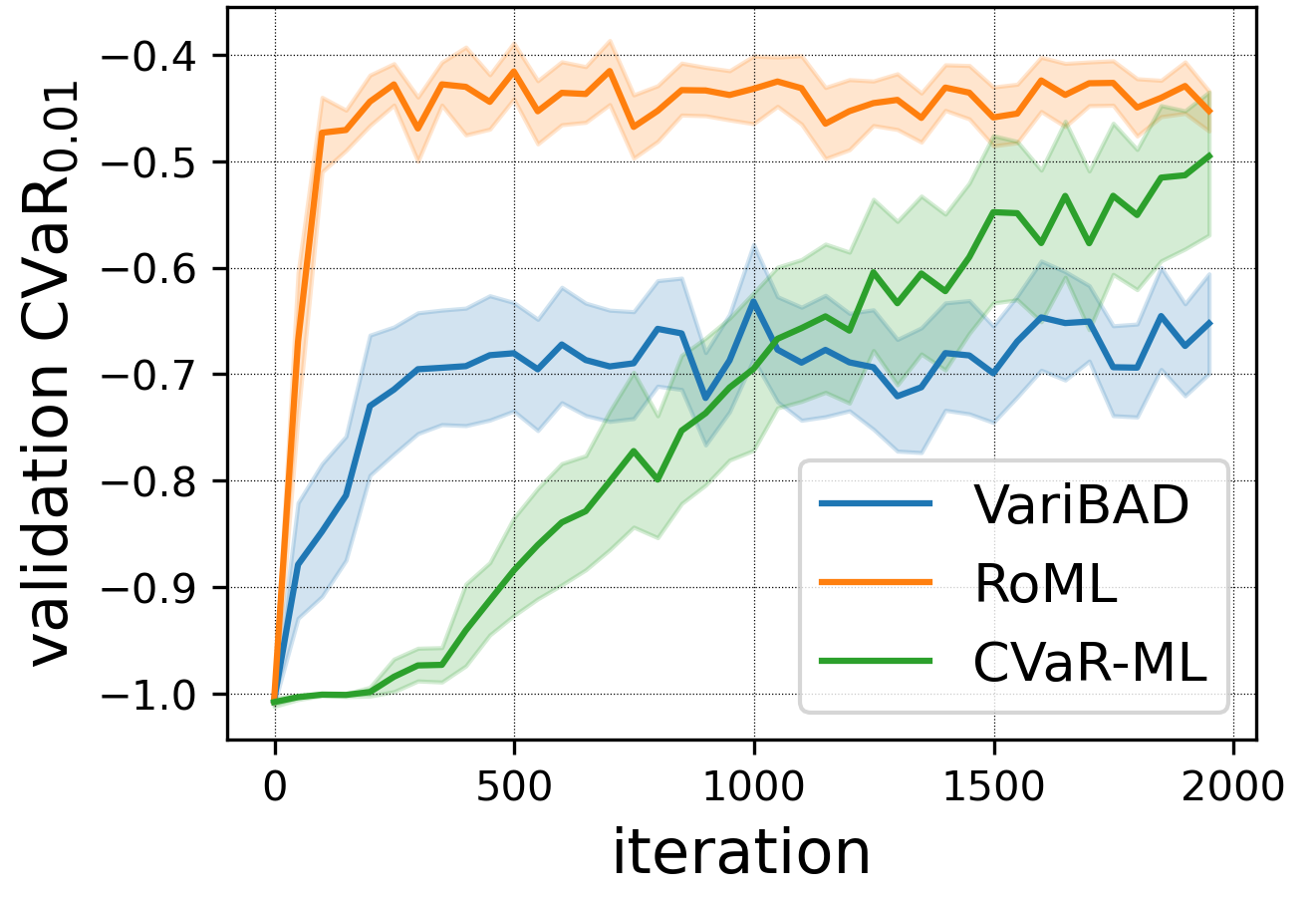}
  \caption{CVaR}
  \label{fig:kd_valid_cvar}
\end{subfigure}
\caption{\small Khazad-Dum: (a-d) Sample episodes. (e-f) Test return vs.~training iteration, with 95\% confidence intervals over 30 seeds.}
\label{fig:kd_trajs}
% \end{wrapfigure}
\end{minipage}
%%%%%%%%%%%%%%%%%%%%%%%
\hspace{0.03\textwidth}
%%%%%%%%%%%%%%%%%%%%%%%
\begin{minipage}{0.47\textwidth}
% \begin{wrapfigure}[21]{R}{0.4\textwidth}
% \begin{figure}%[h]
% \vspace{-1pt}
\centering
\begin{subfigure}{.28\linewidth}
  \centering
  \includegraphics[width=1.\linewidth]{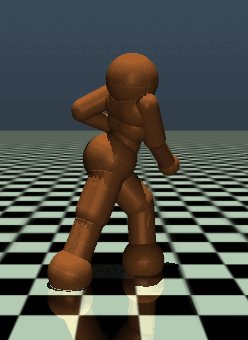}
  \caption{Hum-Mass (VariBAD)}
  \label{fig:humm_vis_varibad}
\end{subfigure}
\hspace{.06\linewidth}
\begin{subfigure}{.5\linewidth}
  \centering
  \includegraphics[width=1.\linewidth]{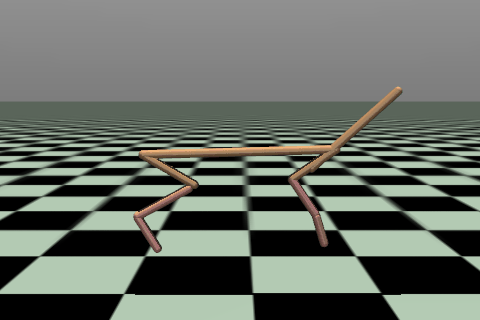}
  \caption{HalfCheetah-Body}%\\ \phantom{ }}
  \label{fig:hcb_vis}
\end{subfigure}
\begin{subfigure}{.28\linewidth}
  \centering
  \includegraphics[width=1.\linewidth]{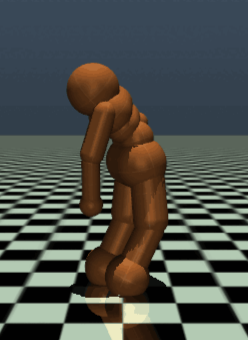}
  \caption{Hum-Mass (RoML)}
  \label{fig:humm_vis_roml}
\end{subfigure}
\begin{subfigure}{.62\linewidth}
  \centering
  \includegraphics[width=1.\linewidth]{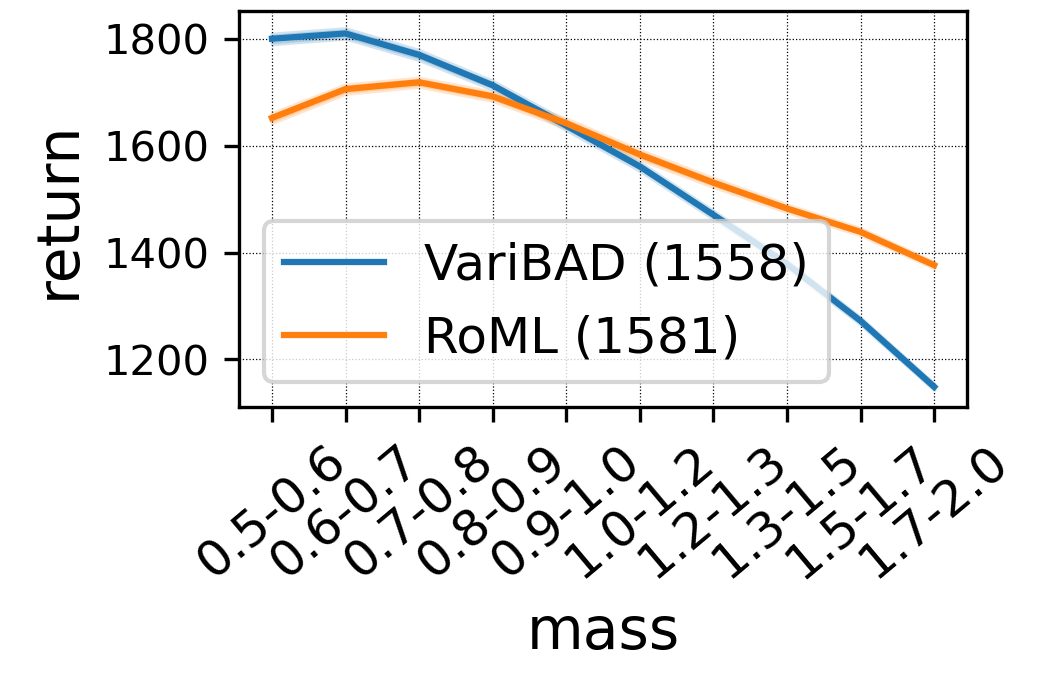}
  \caption{HalfCheetah-Mass}
  \label{fig:hcm_tasks}
\end{subfigure} %\hspace{-13pt}
\caption{\small MuJoCo: (a-c) Sample frames, where tasks vary in mass, head size and damping level. \ra{Notice} that in Humanoid, RoML handles the low-return tasks of large mass by leaning the center of mass forward, so that gravity pulls the humanoid forward. (d) Average return per range of tasks in HalfCheetah-Mass. RoML learns to act robustly: it is less sensitive to the task, and performs better on high-risk tasks.}
\label{fig:mujoco_vis}
% \end{figure}
% \end{wrapfigure}
\end{minipage}
\end{figure}

We implement CVaR-ML and RoML on top of VariBAD.
As shown in \cref{fig:kd_trajs}, VariBAD learns to try the short path -- at the risk of rare falls into the abyss.
By contrast, our CVaR-optimizing methods take the longer path and avoid risk.
This policy increases the cumulative cost of time-steps, but
% While the additional time steps required to reach the goal result in a slightly lower average return, this policy
leads to higher CVaR returns, as shown in \cref{fig:kd_valid_cvar}.
In addition to superior CVaR, RoML also provides competitive \textit{average} returns in this example (\cref{fig:kd_valid_mean}).
Finally, in accordance with \cref{prop:efficiency}, RoML learns significantly faster than CVaR-ML.

% \FloatBarrier

%%%%%%%%%%%%%%%%%%%%%%%%%%%%%%%%%%%%%%%%%%%%%%%%%

\subsection{Continuous Control}
\label{sec:mujoco}

We rely on standard continuous control problems from the MuJoCo framework~\citep{mujoco}: training a cheetah to run (\textbf{HalfCheetah}), and training a \textbf{Humanoid} and an \textbf{Ant} to walk.
% For continuous control problems, we consider the MuJoCo environments~\citep{mujoco} of HalfCheetah (where the goal is to train a cheetah to run) and Humanoid (where the goal is to train a humanoid to walk).
For each of the 3 environments, we create 3 meta-learning versions: % (6 environments in total):
(1) \textit{Goal} or \textit{Vel}~\citep{maml}, where each task corresponds to a different location or velocity objective, respectively; (2) \textit{Mass}, where each task corresponds to a different body mass; and (3) \textit{Body}, where each task corresponds to different mass, head size and physical damping level (similarly to \citet{model_based_mrl}).
In addition, to experiment with high-dimensional task spaces, we randomly draw 10 numeric variables from \textit{env.model} in HalfCheetah, and let them vary between tasks.
% We repeat this for 3 random sets of task variables, yielding different environments.
We define 3 such environments with different random sets of task variables (\textbf{HalfCheetah 10D-task a,b,c}).
For each of the 12 environments above, we set a target risk level of $\alpha=0.05$ and optimize for $K=2$ episodes per task.
Additional implementation details are specified in \cref{app:mujoco}.
% We train each agent over $6\cdot10^7$ frames in the \textit{Body} benchmarks and $3\cdot10^7$ in the other benchmarks.
% For each method, VariBAD-related hyper-parameters are set according to \citet{varibad}.
% The task distribution representations for the CEM in RoML are specified in \cref{app:mujoco}.

Interestingly, the naive approach of \textbf{CVaR-ML} consistently fails to meta-learn in all the cheetah environments. It remains unsuccessful even after large number of steps, indicating a difficulty beyond sample inefficiency.
A possible explanation is the effectively decreased batch size of CVaR-ML.
\textbf{PAIRED} and \textbf{CeSoR} also fail to adjust to the MRL environments, and obtain poor CVaR returns.
% \textbf{CeSoR} fails to adjust to the MRL environments, and obtains poor CVaR returns despite its risk-averse and efficient design~\citep{cesor}.

% 95\% confidence intervals ($1.96\cdot std$) are presented over 30 seeds

\begin{table}[]
\caption{CVaR$_{0.05}$ return over 1000 test tasks, for different models and MuJoCo environments. Standard deviation is presented over 30 seeds. Mean returns are displayed in \cref{tab:mujoco_mean}.}
\label{tab:mujoco}
\centering
\small\addtolength{\tabcolsep}{-1pt}
\begin{tabular}{|l|ccc|ccc|}
\toprule
% \multicolumn{1}{|c|}{\multirow{2}{*}{$\pmb{CVaR_{0.05}}$}}
& \multicolumn{3}{c|}{HalfCheetah}                              & \multicolumn{3}{c|}{HalfCheetah 10D-task}                      \\
\multicolumn{1}{|c|}{}                               & Vel                & Mass                & Body               & (a)                 & (b)                 & (c)                \\ \midrule
CeSoR                                          & $-2606 \pm 25$     & $902 \pm 36$        & $478 \pm 27$       & $637 \pm 26$                   & $981 \pm 31$                   & $664 \pm 26$  \\
PAIRED                                         & $-725 \pm 65$     & $438 \pm 37$        & $218 \pm 51$       & $229 \pm 59$                   & $354 \pm 53$                   & $81 \pm 65$  \\
CVaR-ML & $-897 \pm 23$      & $38 \pm 6$         & $76 \pm 5$        & $120 \pm 11$                   & $141 \pm 11$ & $81 \pm 4$           \\ %\hline
PEARL & $-1156 \pm 23$     & $1115 \pm 19$       & $800 \pm 5$       & $1140 \pm 33$       & $1623 \pm 23$       & $\pmb{1016 \pm 5}$      \\
VariBAD                & $-202 \pm 6$      & $1072 \pm 16$       & $835 \pm 30$       & $1126 \pm 6$       & $1536 \pm 39$       & $988 \pm 13$       \\
RoML (VariBAD)   & $\pmb{-184 \pm 4}$ & $\pmb{1259 \pm 19}$ & $\pmb{935 \pm 17}$ & $\pmb{1227 \pm 13}$ & $\pmb{1697 \pm 24}$ & $999 \pm 20$ \\ %\hline
RoML (PEARL)                                         & $-1089 \pm 31$     & $1186 \pm 34$       & $808 \pm 6$       & $1141 \pm 27$       & $\pmb{1657 \pm 18}$       & $\pmb{1024 \pm 6}$       \\ \midrule
% \hline
% \multicolumn{1}{|c|}{\multirow{2}{*}{$\pmb{CVaR_{0.05}}$}}
& \multicolumn{3}{c|}{Humanoid}                                 & \multicolumn{3}{c|}{Ant}       \\
    & Vel & Mass & Body & Goal & Mass & Body \\ \midrule
VariBAD                        & $801 \pm 10$      & $1283 \pm 18$       & $1290 \pm 19$       & $-500 \pm 9$       & $1370 \pm 6$      & $\pmb{1365 \pm 4}$       \\
RoML (VariBAD)                           & $\pmb{833 \pm 4}$ & $\pmb{1378 \pm 20}$ & $\pmb{1365 \pm 21}$ & $\pmb{-454 \pm 8}$ & $\pmb{1385 \pm 3}$ & $\pmb{1368 \pm 4}$ \\ \bottomrule
\end{tabular}
\end{table}

% \begin{figure}%[h]
% % \vspace{-13pt}
% \centering
% \begin{subfigure}{.38\linewidth}
%   \centering
%   \includegraphics[width=1.\linewidth]{Figs/Mujoco/roml_hc_cvar.png}
%   \caption{HalfCheetah}
%   \label{fig:hc_cvar}
% \end{subfigure}
% \begin{subfigure}{.25\linewidth}
%   \centering
%   \includegraphics[width=1.\linewidth]{Figs/Mujoco/roml_hum_cvar.png}
%   \caption{Humanoid}
%   \label{fig:hum_cvar}
% \end{subfigure}
% \caption{\small The difference between RoML CVaR-return and the baseline CVaR-return over 1000 test tasks, for different environments and baselines:
% (a) 3 HalfCheetah environments with VariBAD and PEARL baselines; (b) 3 Humanoid environments with VariBAD baseline.
% 95\% confidence intervals are calculated over 30 seeds.
% RoML improves the CVaR return compared to its baseline in all 9 experiments.
% % The CVaR-return advantage of RoML over its baseline is positive in all 9 experiments.
% CVaR-ML obtained very low returns in all environments and is not displayed in the figure.
% All absolute returns (mean and CVaR) are reported in \cref{app:mujoco}.}
% \label{fig:mujoco_cvar}
% \end{figure}

% achieves significantly higher CVaR returns in all benchmarks (\cref{fig:mujoco_cvar}) -- using the same hyper-parameters as VariBAD.

\textbf{RoML}, on the other hand, consistently improves the CVaR returns (\cref{tab:mujoco}) compared to its baseline (VariBAD or PEARL), while using the same hyper-parameters as the baseline.
% achieves consistently higher CVaR returns than the MRL baseline on which it is implemented.
% -- for \textit{all} environments and \textit{all} baselines (\cref{fig:mujoco_cvar}). % -- with high statistical significance in all experiments but one.
% The relative CVaR returns of RoML compared to its baselines are displayed in \cref{fig:mujoco_cvar}, and the absolute returns in \cref{app:mujoco}.
% Note that RoML runs on top each baseline (VariBAD or PEARL) without changing the original baseline's hyper-parameters, and improves the CVaR returns.
% In the 3 HalfCheetah environments, RoML runs on top of either VariBAD or PEARL -- without changing their original hyper-parameters from \citet{varibad} and \citet{pearl}, respectively.
The VariBAD baseline presents better returns and running times than PEARL on HalfCheetah, and thus is used for the 6 extended environments (Humanoid and Ant).
RoML improves the CVaR return in comparison to the baseline algorithm in all the 18 experiments (6 with PEARL and 12 with VariBAD). % over test episodes.
% Compared to PEARL, VariBAD presents higher absolute returns and shorter running times by an order of magnitude, thus VariBAD is also used as the baseline for the 9 extended experiments (Humanoid, Ant and high-dimensional tasks).

In 5 out of 18 experiments, RoML slightly improves the \textit{average} return compared to its baseline, and not only the CVaR (\cref{tab:mujoco_mean} in the appendix).
\ra{This indicates that low-return tasks can sometimes be improved at the cost of high-return tasks, but without hurting average performance. In addition, this may indicate that over-sampling difficult tasks forms a helpful learning curriculum.}

The robustness of RoML to the selected task is demonstrated in \cref{fig:hcm_tasks}.
% \cref{fig:mujoco_tasks} demonstrates the robustness of RoML to the selected task, as it improves the returns over the more difficult tasks. % of the learned meta-policy
In multi-dimensional task spaces, RoML learns to focus the sampling modification on the high-impact variables, as demonstrated in \cref{fig:mujoco_dist} in the appendix.
% In the \textit{Body} environments, where the task is multi-dimensional, RoML learns to sample tasks with higher masses and damping levels, but leaves the sampled head-size unchanged, as it turns out to be irrelevant to the task return (\cref{fig:hcb_dist} in the appendix).
Finally, qualitative inspection shows that RoML learns to handle larger masses, for example, by leaning forward and letting gravity pull the agent forward, as displayed in \cref{fig:humm_vis_roml} (see animations on \repo{GitHub}).

% \begin{wrapfigure}{R}{0.52\textwidth}
% % \begin{figure}[h]
% % \vspace{-13pt}
% \centering
% \hspace{-10pt}
% \begin{subfigure}{.49\linewidth}
%   \centering
%   \includegraphics[width=1.\linewidth]{Figs/Mujoco/hcm_tasks.png}
%   \caption{HalfCheetah-Mass}
%   \label{fig:hcm_tasks}
% \end{subfigure} \hspace{-13pt}
% \begin{subfigure}{.49\linewidth}
%   \centering
%   \includegraphics[width=1.\linewidth]{Figs/Mujoco/humm_tasks.png}
%   \caption{Humanoid-Mass}
%   \label{fig:humm_tasks}
% \end{subfigure}
% \hspace{-10pt}
% \caption{\small Average return per range of tasks in HalfCheetah-Mass and Humanoid-Mass (global average in parentheses). RoML learns to act robustly: it is less sensitive to the task, and in particular performs better on high-risk tasks. \textit{Vel} and \textit{Body} benchmarks are displayed in \cref{fig:mujoco_tasks2} in the appendix.}
% \label{fig:mujoco_tasks}
% \end{wrapfigure}

% \FloatBarrier

%%%%%%%%%%%%%%%%%%%%%%%%%%%%%%%%%%%%%%%%%%%%%%%%%

\subsection{Beyond RL: Robust Supervised Meta-Learning}
\label{sec:sine}

\begin{wrapfigure}[12]{R}{0.36\textwidth}
% \begin{figure}[h]
\vspace{-33pt}
% \centering
% \hspace{-10pt}
\includegraphics[width=1.\linewidth]{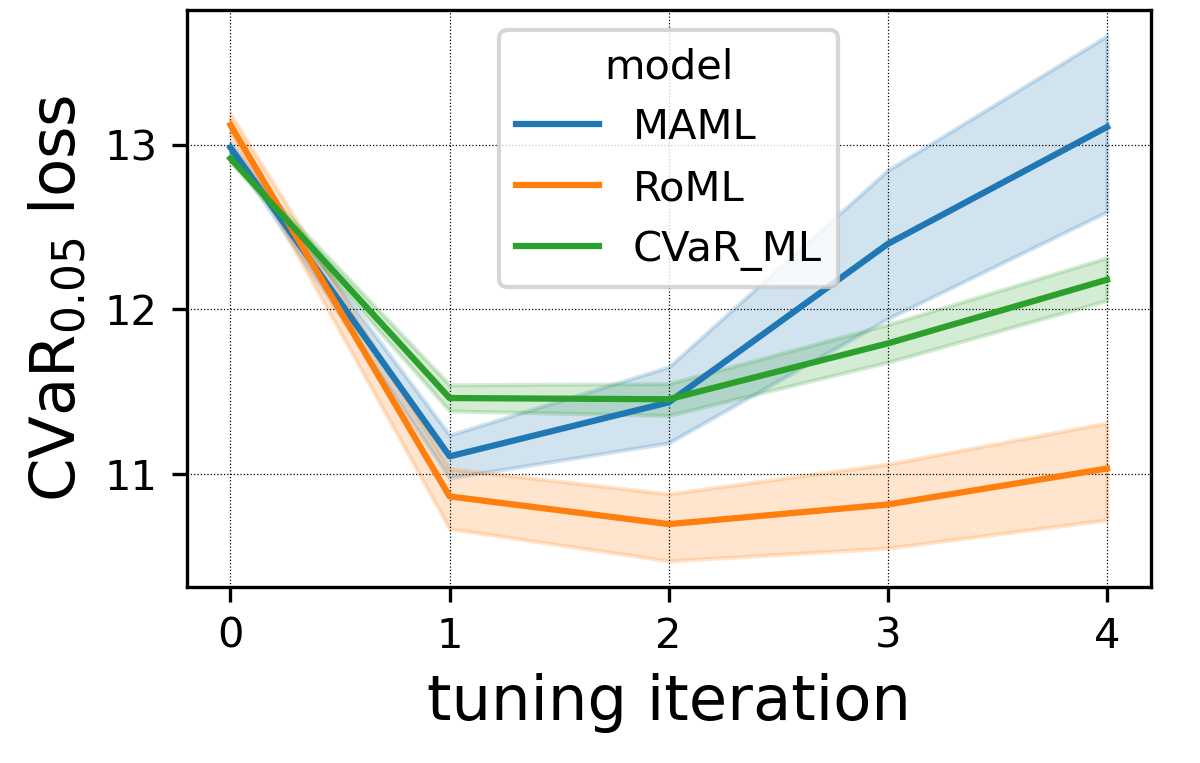}
\caption{\small Supervised Sine Regression: CVaR loss over 10000 test tasks, against the number of tuning gradient-steps at test time. The 95\% confidence intervals are calculated over 30 seeds.}
\label{fig:sine_cvar_only}
\end{wrapfigure}

Our work focuses on robustness in MRL.
However, the concept of training on harder data to improve robustness, as embodied in RoML, is applicable beyond the scope of RL.
As a preliminary proof-of-concept, we apply RoML to a \ra{toy supervised meta-learning} problem of sine regression, based  on \citet{maml}:
The input is $x\in[0,2\pi)$, the desired output is $y=A\sin(\omega x + b)$, and the task is defined by the parameters $\tau=(A,b,\omega)$.
% , distributed uniformly over $\Omega = [0.1,5]\times[0,2\pi]\times[0.3,3]$.
Similarly to \citet{maml}, the model is fine-tuned  for each task via a gradient-descent optimization step over 10 samples $\{(x_i,y_i)\}_{i=1}^{10}$, and is tested on another set of 10 samples. The goal is to find model weights that adapt quickly to new task data. %-- in expectation over tasks.

We implement CVaR-ML and RoML on top of MAML~\citep{maml}. %, with $\alpha=0.05$.
% For the CEM of RoML, we re-parameterize the uniform task distribution using Beta distributions (see details below).
% As shown in \cref{fig:sine_cem} in the appendix, RoML learns to focus on tasks (sine functions) with high amplitudes and slightly increased frequencies, without changing the phase distribution.
% \cref{fig:sine} displays the test losses over 30 seeds, after meta-training for 10000 tasks.
As shown in \cref{fig:sine_cvar_only}, RoML achieves better CVaR losses over tasks than both CVaR-ML and MAML.
The complete setting and results are presented in \cref{app:sine}.

% In this problem, the task space $\Omega$ induces a natural structure on the tasks, as we can tell which tasks are ``close'' to each other. The CEM takes advantage of this structure to characterize high-risk tasks -- even though it can never observe the infinitely many possible tasks.
% The applicability of CEM to supervised learning in general is discussed in \cref{app:cem_discussion}.

%%%%%%%%%%%%%%%%%%%%%%%%%%%%%%%%%%%%%%%%%%%%%%%%%
%%%%%%%%%%%%%%%%%%%%%%%%%%%%%%%%%%%%%%%%%%%%%%%%%

\section{Summary and Future Work}
% \section{Discussion}
\label{sec:summary}

We defined a robust MRL objective and derived the CVaR-ML algorithm to optimize it.
In contrast to its analogous algorithm in standard RL, we proved that CVaR-ML does not present biased gradients, yet it does inherit the same data inefficiency. To address the latter, we introduced RoML and demonstrated its advantage in sample efficiency and CVaR return.

% The main limitation of RoML is the need for at least partial control over the selection of training tasks.
% This assumption is common in other RL frameworks \citep{PAIRED,robust_PLR}, and often holds in both simulations and the real world (e.g., deciding where and when to train driving).
% Yet, future research may increase the data efficiency of CVaR-ML \textit{without} relying on such control.
% % Yet, collection of training data can often be directed in both simulations and the real world (e.g., deciding where and when to train driving), hence this assumption often holds.

% Our meta-algorithm can also be applied to other scenarios in which a natural task structure is present, and leverage the structure to improve task robustness. Future applications may include supervised learning and coordinate-based regression (e.g., \citet{coord_based_regression}).

\ra{Future} research may address the CEM-related limitations of RoML discussed at the end of \cref{sec:roml}.
Another direction for future work is extension of RoML to other scenarios, especially where a natural task structure can be leveraged to improve task robustness, e.g., in supervised learning and coordinate-based regression \citep{coord_based_regression}.

Finally, RoML is \repo{easy to implement}, operates agnostically as a meta-algorithm on top of existing MRL methods, and can be set to any desired robustness level.
We believe that these properties, along with our empirical results, make RoML a promising candidate for MRL in risk-sensitive applications.

\textbf{\ra{Acknowledgements}:} This work has received funding from the European Union's Horizon Europe Programme, under grant number 101070568.

%%%%%%%%%%%%%%%%%%%%%%%%%%%%%%%%%%%%%%%%%%%%%%%%%%%%%%%%%%%%
%%%%%%%%%%%%%%%%%%%%%%%%%%%%%%%%%%%%%%%%%%%%%%%%%%%%%%%%%%%%

\newpage
% \subsubsection*{Acknowledgements}

\bibliographystyle{plainnat}
% \nocite{*}
\bibliography{main}

\newpage
\appendix

%%%%%%%%%%%%%%%%%%%%%%%%%%%%%%%%%%%%%%%%%%%%%%%%%
%%%%%%%%%%%%%%%%%%%%%%%%%%%%%%%%%%%%%%%%%%%%%%%%%

% \section*{Table of Contents}
\setcounter{tocdepth}{1}
\tableofcontents
\newpage

%%%%%%%%%%%%%%%%%%%%%%%%%%%%%%%%%%%%%%%%%%%%%%%%%
%%%%%%%%%%%%%%%%%%%%%%%%%%%%%%%%%%%%%%%%%%%%%%%%%

\section{Policy Gradient for CVaR Optimization in Meta RL}
\label{app:pg}

In this section we provide the complete proof for \cref{theorem:mpg}.
For completeness, \cref{app:pg_rl} recaps of the analogous proof in standard (non-meta) RL, before moving on to the proof in \cref{app:pg_mrl}.
This allows us to highlight the differences between the two.

The substantial difference is that in RL, the CVaR is defined directly over the low-return trajectories, and the policy parameter $\theta$ affects the probability of each trajectory in the tail (\cref{eq:conservation_rl}).
In MRL (\cref{eq:obj_cvar}), on the other hand, the CVaR is defined over the low-return tasks, whose probability is not affected directly by $\theta$ (\cref{eq:conservation}). This allows a decoupling between $\theta$ and $\tau$, which results in \cref{theorem:mpg}.

Another high-level intuition is as follows.
In RL, the CVaR-PG gradient is invariant to successful strategies, hence must be explicitly negative for the unsuccessful ones (in order not to encourage them, see \cref{fig:pg}).
In MRL, within the tasks of interest, the gradient always encourages successful strategies on account of the unsuccessful ones (\cref{fig:mpg}).

%%%%%%%%%%%%%%%%%%%%%%%%%%%%%%%%%%%%%%%%%%%%%%%%%

\subsection{Recap: PG for CVaR in (non-meta) RL}
\label{app:pg_rl}

We briefly recap the calculation of Propositions 1 and 2 in \citet{gcvar} for CVaR policy gradient under the standard RL settings.

\begin{definition}[CVaR return in (non-meta) RL]
Consider an MDP $(S,A,P,\gamma,P_0)$ with the cumulative reward (i.e., return) $R\sim P^\theta$, whose $\alpha$-quantile is $q_\alpha^\theta(R)$.
Recall the CVaR objective defined in \cref{sec:preliminaries}:
$$ \tilde{J}_\alpha^\theta(R) = \mathtt{CVaR}^\alpha_{R\sim P^\theta}[R] = \int_{-\infty}^{q_\alpha^\theta(R)} x \cdot P^\theta(x) \cdot dx $$
\end{definition}

To calculate the policy gradient of $\tilde{J}_\alpha^\theta(R)$, we begin with the conservation of probability mass below the quantile $q_\alpha^\theta$:
$$ \int_{-\infty}^{q_\alpha^\theta(R)} P^\theta(x) dx \equiv \alpha . $$
Then, using the Leibniz integral rule, we have
\begin{align}
\label{eq:conservation_rl}
\begin{split}
0 &= \nabla_\theta \int_{-\infty}^{q_\alpha^\theta(R)} P^\theta(x) dx = \left[ \int_{-\infty}^{q_\alpha^\theta(R)} \nabla_\theta P^\theta(x) dx \right] \ + \ \left[ P^\theta(q_\alpha^\theta(R)) \cdot \nabla_\theta q_\alpha^\theta(R) \right] .
\end{split}
\end{align}
Notice that as a particular consequence of the conservation rule \cref{eq:conservation_rl}, positive gradients of $P^\theta(x)$ cause the quantile $q_\alpha^\theta(R)$ to decrease. This phenomenon makes the CVaR policy gradient sensitive to the selection of baseline, as visualized in \cref{fig:cvar_pg}.
In fact, the quantile $q_\alpha^\theta(R)$ itself is the only baseline that permits unbiased gradients:
\begin{align}
\label{eq:pg_rl2}
\begin{split}
\nabla_\theta \tilde{J}_\alpha^\theta(R) &= \nabla_\theta \int_{-\infty}^{q_\alpha^\theta(R)} x \cdot P^\theta(x) \cdot dx = \\ & \left[ \int_{-\infty}^{q_\alpha^\theta(R)} x \cdot \nabla_\theta P^\theta(x) dx \right] \ + \ \left[ q_\alpha^\theta(R) \cdot P^\theta(q_\alpha^\theta(R)) \cdot \nabla_\theta q_\alpha^\theta(R) \right] \overbrace{=}^{\cref{eq:conservation_rl}} \\
& \left[ \int_{-\infty}^{q_\alpha^\theta(R)} x \cdot \nabla_\theta P^\theta(x) dx \right] \ - \ \left[ q_\alpha^\theta(R) \cdot \int_{-\infty}^{q_\alpha^\theta(R)} \nabla_\theta P^\theta(x) dx \right] = \\
& \int_{-\infty}^{q_\alpha^\theta(R)} (x - q_\alpha^\theta(R)) \cdot \nabla_\theta P^\theta(x) dx,
\end{split}
\end{align}
which gives us \cref{eq:PG_rl}.

\begin{figure}[h]
% \vspace{-13pt}
\centering
\begin{subfigure}{.5\linewidth}
  \centering
  \includegraphics[width=1.\linewidth]{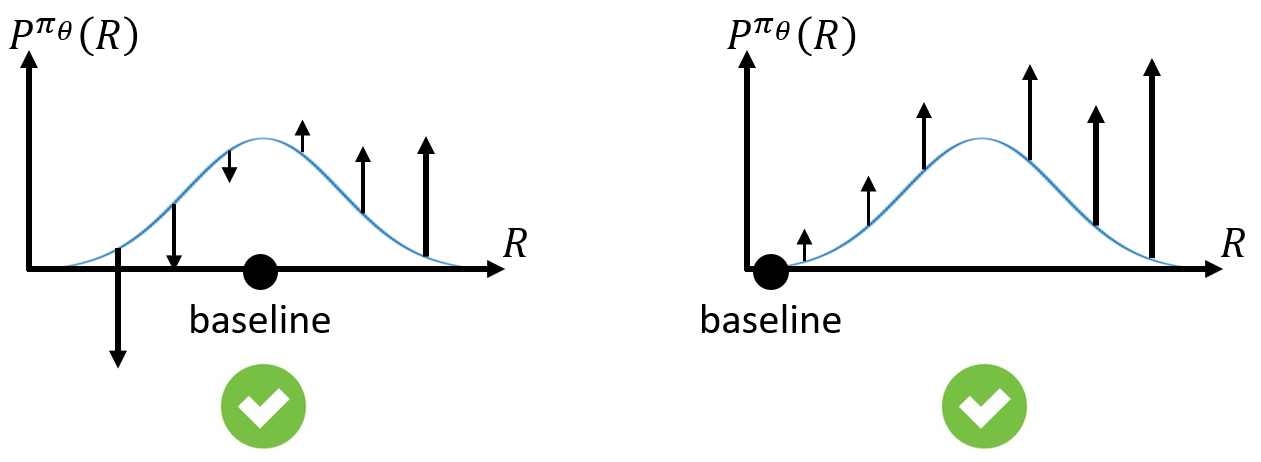}
  \caption{Mean PG}
  \label{fig:mean_pg}
\end{subfigure} \\
\begin{subfigure}{.5\linewidth}
  \centering
  \includegraphics[width=1.\linewidth]{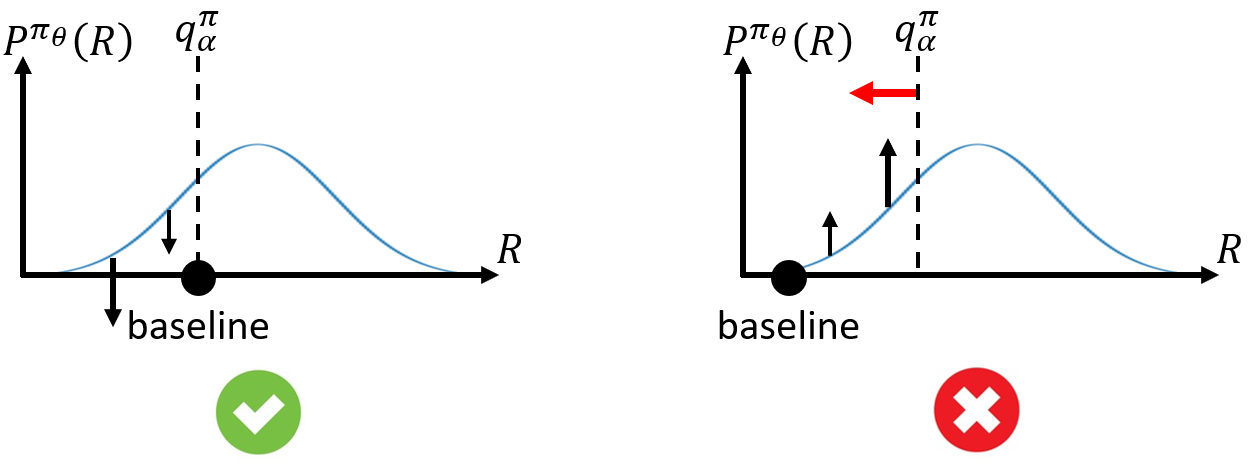}
  \caption{CVaR PG, \cref{eq:PG_rl}}
  \label{fig:cvar_pg}
\end{subfigure}
\caption{\small Illustration of the policy gradient estimation in standard RL. (a) In Mean PG, the expected gradient is independent of the baseline: even if most of the distribution $P^{\pi_\theta}$ seems to be "pushed upwards", its normalization to a total probability of 1 forces the probability of low returns to decrease for that of high returns will increase. (b) In CVaR PG, due to the effect of \cref{eq:conservation_rl}, any baseline except for $q_\alpha^\theta(R)$ leads to biased gradients.}
\label{fig:pg}
\end{figure}

%%%%%%%%%%%%%%%%%%%%%%%%%%%%%%%%%%%%%%%%%%%%%%%%%

\subsection{PG for CVaR in Meta RL}
\label{app:pg_mrl}

We now turn to PG for CVaR optimization in \textit{MRL}.
We rely on the definitions and notations of \cref{sec:preliminaries}, as well as \cref{def:omega_alpha} and \cref{assumption} in \cref{sec:cvar_mrl}.
Notice that $J_\alpha^\theta(R)$ of \cref{eq:obj_cvar} can be written in integral form as
$$ J_\alpha^\theta(R) = \mathtt{CVaR}_{\tau\sim D}^\alpha \left[ \mathbb{E}_{R\sim P_\tau^\theta}[R] \right] = \int_{\Omega_\alpha^\theta} D(z) \int_{-\infty}^{\infty} x \cdot P_z^\theta(x) \cdot dx \cdot dz . $$

In \cref{eq:pg_rl2} above, the boundary of the integral over the $\alpha$-tail is simply the scalar $q_\alpha^\theta$. In MRL, this is replaced by the boundary of the set $\Omega_\alpha^\theta$, defined in a general topological space. Thus, we begin by characterizing this boundary.

\begin{lemma}[The boundary of $\Omega_\alpha^\theta$]
\label{lemma:boundary}
Under \cref{assumption},
% $ \partial\Omega_\alpha^\theta \subseteq \{ z\in\Omega \,|\, V_z^\theta = q_\alpha(V_\tau^\theta) \} $
$ \forall z\in\partial\Omega_\alpha^\theta:\ \ \int_{-\infty}^{\infty} x \cdot P_z^\theta(x) \cdot dx = q_\alpha(V_\tau^\theta) $.
\end{lemma}
\begin{proof}
% https://math.stackexchange.com/questions/2215563/continuous-functions-and-the-boundary
Since $v(z) = V_z^\theta$ is a continuous function between topological spaces, by denoting $B=(-\infty,\,q_\alpha(V_\tau^\theta)]$ we have
$$ \partial\Omega_\alpha^\theta = \partial v^{-1}(B) \overbrace{\subseteq}^{\text{continuous } v} v^{-1}(\partial B) = v^{-1}(\{q_\alpha(V_\tau^\theta)\}) = \{ z\in\Omega \,|\, V_z^\theta = q_\alpha(V_\tau^\theta) \}, $$
hence $\forall z\in\partial\Omega_\alpha^\theta:\, V_z^\theta = q_\alpha(V_\tau^\theta)$.
Notice that $V_z^\theta = \int_{-\infty}^{\infty} x P_z^\theta(x) \cdot dx$.
\end{proof}

\begin{figure}%[h]
% \vspace{-13pt}
\centering
\includegraphics[width=.95\linewidth]{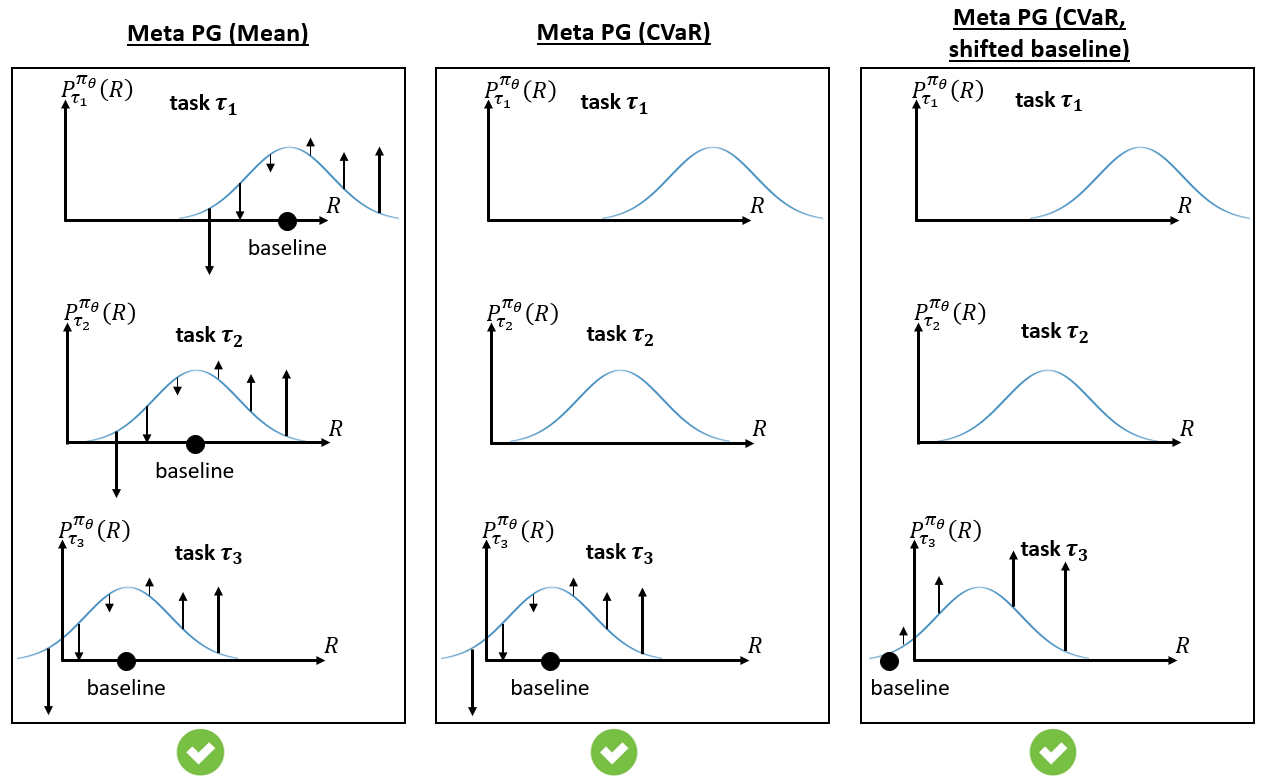}
\caption{\small Illustration of the meta policy gradient estimation. In Mean meta-PG (left, \cref{eq:MPG}), the PG is estimated for all tasks. In CVaR meta-PG (center, \cref{eq:cvar_mpg}), it is only calculated over the low-return tasks.
Since the returns distribution is decoupled from the tasks distribution, shifting the baseline (right) does not insert bias to the PG, in accordance with \cref{theorem:mpg}.}
\label{fig:mpg}
\end{figure}

Finally, we can prove \cref{theorem:mpg}.

\begin{proof}[Proof of \cref{theorem:mpg}]
First, we consider the conservation of probability:
$$ \int_{\Omega_\alpha^\theta} D(z) dz \equiv \alpha $$
The gradient of this integral can be calculated using a high-dimensional generalization of the Leibniz integral rule, named Reynolds Transport Theorem (RTT,~\citet{vector_calculus}):
% https://en.wikipedia.org/wiki/Reynolds_transport_theorem
%
\begin{align}
\label{eq:conservation}
\begin{split}
0 \overbrace{=}^{\substack{\text{derivative of} \\ \text{a constant}}} \nabla_\theta \int_{\Omega_\alpha^\theta} D(z) dz \overbrace{=}^{\text{RTT}}& \int_{\Omega_\alpha^\theta} \nabla_\theta D(z) dz + \int_{\partial\Omega_\alpha^\theta} D(z) \cdot (v_b \cdot n) \cdot dA \\
\overbrace{=}^{\nabla_\theta D(z)\equiv0}& \int_{\partial\Omega_\alpha^\theta} D(z) \cdot (v_b \cdot n) \cdot dA ,
\end{split}
\end{align}
where $n(z,\theta)$ is the outward-pointing unit vector that is normal to the surface $\partial\Omega_\alpha^\theta$, and $v_b(z,\theta)$ is the velocity of the area element on the surface with area $dA$.
Notice that $\nabla_\theta D(z) \equiv 0$ in the last equality expresses a substantial difference from the standard RL settings in \cref{eq:conservation_rl}: there, we have $\nabla_\theta P^\theta(x)$, which does not necessarily vanish.

Next, we turn to the meta policy gradient itself, again using Reynolds Transport Theorem:
\begin{align}
\label{eq:meta_pg}
\begin{split}
\nabla_\theta J_\alpha^\theta(R) &= \nabla_\theta \int_{\Omega_\alpha^\theta} D(z) \int_{-\infty}^{\infty} x P_z^\theta(x) \cdot dx \cdot dz \overbrace{=}^{\text{RTT}} \\
& \left[ \int_{\Omega_\alpha^\theta} D(z) \int_{-\infty}^{\infty} x \nabla_\theta P_z^\theta(x) \cdot dx \cdot dz \right] + \left[ \int_{\partial\Omega_\alpha^\theta} D(z) \left( \int_{-\infty}^{\infty} x P_z^\theta(x) \cdot dx \right) \cdot (v_b \cdot n) \cdot dA \right] \overbrace{=}^{\cref{lemma:boundary}} \\
& \left[ \int_{\Omega_\alpha^\theta} D(z) \int_{-\infty}^{\infty} x \nabla_\theta P_z^\theta(x) \cdot dx \cdot dz \right] + \left[ q_\alpha(V_\tau^\theta) \int_{\partial\Omega_\alpha^\theta} D(z) \cdot (v_b \cdot n) \cdot dA \right] \overbrace{=}^{\cref{eq:conservation}} \\
& \int_{\Omega_\alpha^\theta} D(z) \int_{-\infty}^{\infty} x \nabla_\theta P_z^\theta(x) \cdot dx \cdot dz .
\end{split}
\end{align}

Finally, we show that any $\theta$-independent additive baseline (highlighted in the equation) does not change the gradient calculation:
\begin{align}
\label{eq:meta_pg_baseline}
\begin{split}
\int_{\Omega_\alpha^\theta} D(z) \int_{-\infty}^{\infty} (x{\color{red}-b}) \nabla_\theta P_z^\theta(x) \cdot dx \cdot dz &= \\
\left[ \int_{\Omega_\alpha^\theta} D(z) \int_{-\infty}^{\infty} x \nabla_\theta P_z^\theta(x) \cdot dx \cdot dz \right] &- \left[ \int_{\Omega_\alpha^\theta} D(z) \int_{-\infty}^{\infty} {\color{red} b} \nabla_\theta P_z^\theta(x) \cdot dx \cdot dz \right] = \\
\left[ \int_{\Omega_\alpha^\theta} D(z) \int_{-\infty}^{\infty} x \nabla_\theta P_z^\theta(x) \cdot dx \cdot dz \right] &- \left[ \int_{\Omega_\alpha^\theta} D(z) \cdot {\color{red} b} \cdot \nabla_\theta \left( \int_{-\infty}^{\infty} P_z^\theta(x) \cdot dx \right) \cdot dz \right] \overbrace{=}^{\int_{-\infty}^{\infty} P_z^\theta(x) dx \equiv 1} \\
\left[ \int_{\Omega_\alpha^\theta} D(z) \int_{-\infty}^{\infty} x \nabla_\theta P_z^\theta(x) \cdot dx \cdot dz \right] &- 0 = \nabla_\theta J_\alpha^\theta(R) ,
\end{split}
\end{align}
which completes the proof.
Notice that we used the identity $\nabla_\theta \int_{-\infty}^{\infty} P_z^\theta(x) dx = \nabla_\theta 1 = 0$; this does not hold for the analogous term in the standard RL settings in \cref{eq:conservation_rl}, $\nabla_\theta \int_{-\infty}^{q_\alpha^\theta(x)} P^\theta(x) dx$, whose gradient depends on $\nabla_\theta q_\alpha^\theta(x)$ according to the Leibniz integral rule.

\end{proof}

%%%%%%%%%%%%%%%%%%%%%%%%%%%%%%%%%%%%%%%%%%%%%%%%%

% \subsection{An Illustrative Example}
% \label{app:pg_example}

% a trivial continuous bandits problem where r=a and the policies-class is ${N(theta,1)}_theta$. the gradient is simply 1. however, calculation with a negative baseline (or 0 baseline where theta>0) pulls us back. However, on MRL, $r=tau+a$ where $tau~N(0,1)$. then, the decoupling between tau and a always gives us 1 in the gradient calculation, regardless of the baseline, even if the policy depends on tau (it shouldn't since tau is unknown, but can show that to stress the decoupling, and mention that in practice the dependence might be indirect through observations).

%%%%%%%%%%%%%%%%%%%%%%%%%%%%%%%%%%%%%%%%%%%%%%%%%
%%%%%%%%%%%%%%%%%%%%%%%%%%%%%%%%%%%%%%%%%%%%%%%%%

\section{Proof of \cref{prop:efficiency}}
\label{app:efficiency}

\begin{proof}

Recall that by \cref{eq:PG_est_mrl}, $G = \frac{1}{\alpha N} \sum_{i=1}^N \pmb{1}_{R_i\le \hat{q}_\alpha^\theta} \sum_{m=1}^M g_{i,m}$.
Denoting $G_i = \sum_{m=1}^M g_{i,m}$ and substituting $\hat{q}_\alpha^\theta={q}_\alpha^\theta$, we have
$$ G = \frac{1}{N} \sum_{i=1}^N \alpha^{-1} \pmb{1}_{R_i\le {q}_\alpha^\theta} G_i. $$
\paragraph{Expectation:}
Since $\{G_i\}$ are i.i.d, and using the law of total probability, we obtain
\begin{align*}
    \mathbb{E}_D [\alpha^{-1} \pmb{1}_{R_i\le {q}_\alpha^\theta} G_i] = &\alpha \cdot \left( \alpha^{-1} \cdot 1 \cdot \mathbb{E}_D [G_1 \,|\, R_1\le {q}_\alpha^\theta] \right) + (1-\alpha) \cdot \left( \alpha^{-1} \cdot 0 \cdot \mathbb{E}_D [G_1 \,|\, R_1> {q}_\alpha^\theta] \right) \\
    = &\mathbb{E}_D [G_1 \,|\, R_1\le {q}_\alpha^\theta] .
\end{align*}
By switching the task sample distribution to $D_\alpha^\theta$, and using the definition of $D_\alpha^\theta$, we simply have
$$ \mathbb{E}_{D_\alpha^\theta} [\alpha^{-1} \pmb{1}_{R_i\le {q}_\alpha^\theta} G_i] = \alpha^{-1} \mathbb{E}_{D_\alpha^\theta} [G_1] = \alpha^{-1} \mathbb{E}_D [G_1 \,|\, R_1\le {q}_\alpha^\theta] . $$
Together, we obtain $\mathbb{E}_{D_\alpha^\theta} [\alpha G] = \mathbb{E}_D [G]$ as required.

\paragraph{Variance:}
For the original distribution, since $\{G_i\}$ are i.i.d, we have
\begin{align*}
    N \cdot \mathrm{Var}_D (G)
    &= \mathrm{Var}_D (\alpha^{-1} \pmb{1}_{R_1\le {q}_\alpha^\theta} G_1) \\
    &= \mathbb{E}_D [\alpha^{-2} \pmb{1}_{R_1\le {q}_\alpha^\theta} G_1^2] - \mathbb{E}_D [\alpha^{-1} \pmb{1}_{R_1\le {q}_\alpha^\theta} G_1]^2 \\
    &= \mathbb{E}_D [\alpha \alpha^{-2} G_1^2 \,|\, R_1\le {q}_\alpha^\theta] - \mathbb{E}_D [\alpha \alpha^{-1} G_1 \,|\, R_1\le {q}_\alpha^\theta]^2 \\
    &= \alpha^{-1} \mathbb{E}_{D_\alpha^\theta} [ G_1^2] - \mathbb{E}_{D_\alpha^\theta} [G_1]^2 \\
    &\ge \alpha^{-1} (\mathbb{E}_{D_\alpha^\theta} [ G_1^2] - \mathbb{E}_{D_\alpha^\theta} [G_1]^2) \\
    &= \alpha^{-1} \mathrm{Var}_{D_\alpha^\theta} (G_1) .
\end{align*}
For the tail distribution $D_\alpha^\theta$, however,
\begin{align*}
    N \cdot \mathrm{Var}_{D_\alpha^\theta} (\alpha G) = \alpha^2 \mathrm{Var}_{D_\alpha^\theta} (\alpha^{-1} \pmb{1}_{R_1\le {q}_\alpha^\theta} G_1) = \mathrm{Var}_{D_\alpha^\theta} (G_1) ,
\end{align*}
which completes the proof.

\end{proof}

%%%%%%%%%%%%%%%%%%%%%%%%%%%%%%%%%%%%%%%%%%%%%%%%%
%%%%%%%%%%%%%%%%%%%%%%%%%%%%%%%%%%%%%%%%%%%%%%%%%

\section{The Cross Entropy Method}
\label{app:cem}

\subsection{Background}
\label{app:cem_background}

The Cross Entropy Method (CEM,~\citet{CE_tutorial}) is a general approach to rare-event sampling and optimization.
In this work, we use its sampling version to sample high-risk tasks from the tail of $D$.
As described in \cref{sec:roml}, the CEM repeatedly samples from the parameterized distribution $D_\phi$, and updates $\phi$ according to the $\beta$-tail of the sampled batch.
Since every iteration focuses on the tail of its former, we intuitively expect exponential convergence to the tail of the original distribution.
While theoretical convergence analysis does not guarantee the exponential rate \citep{RareEE}, practically, the CEM often converges within several iterations.
For clarity, we provide the pseudo-code for the basic CEM in \cref{algo:cem}.
In this version, the CEM repeatedly generates samples from the tail of the given distribution $D_{\phi_0}$.

\begin{algorithm}
% \vspace{-12pt}
% \hspace{5pt}
\caption{The Cross Entropy Method (CEM)}
\label{algo:cem}
% \setstretch{1.1}
\DontPrintSemicolon
\SetAlgoNoLine
\SetNoFillComment

{\bf Input}: distribution $D_{\phi_0}$; score function $R$; target level $q$; batch size $N$; CEM quantile $\beta$.\;
 \BlankLine
 \BlankLine
 $\phi\leftarrow {\phi_0}$\;
 
 \While{true}{
    \tcp{Sample}
    Sample $z \sim D_{\phi}^N$\;
	$w_i \leftarrow D_{{\phi_0}}(z_i) / D_{\phi}(z_i) \quad (1\le i\le N)$\;
	Print $z$\;
    \tcp{Update}
    $q^\prime \leftarrow \max\left(q,\ q_\beta\left(\{R(z_i)\}_{i=1}^N\right) \right)$\;
	$\phi \leftarrow \mathrm{argmax}_{\phi^{\prime}} \sum_{i=1}^N w_i \pmb{1}_{R(z_i)\le q^\prime} \log D_{\phi^{\prime}}(z_i)$\;
 }
\end{algorithm}

%%%%%%%%%%%%%%%%%%%%%%%%%%%%%%%%%%%%%%%%%%%%%%%%%

\subsection{Discussion}
\label{app:cem_discussion}

The CEM is the key to the flexible robustness level of RoML (\cref{algo:roml}): it can learn to sample not only the single worst-case task, but all the $\alpha$ tasks with the lowest returns. %In fact, this is the distinction between the sampling version and the optimization version of the CEM.
% RoML relies on the cross entropy method (CEM) to sample the high-risk tasks.
% The CEM can fit the whole tail of the task distribution, and not only the single task with the lowest return; this is the key to the flexible robustness level of RoML.
% Unlike common adversary methods, the CEM does not search for the one task with the lowest-return, but can rather learn all the $\alpha$ lowest-return tasks. This is the key to the optimization of the CVaR objective with the flexible robustness level $\alpha$.

The CEM searches for a task distribution within a parametric family of distributions. This approach can handle infinite task spaces, and learn the difficulty of tasks in the entire task space from a mere finite sample of tasks.
For example, assume that the tasks correspond to environment parameters that take continuous values within some bounded box (as in \cref{sec:mujoco} and \cref{sec:sine}). The CEM can fit a distribution over an entire subset of the box -- from a mere finite batch of tasks.
% In MRL problems the structure is often natural, e.g., a parameter of the environment that takes different values in different tasks.
This property lets the CEM learn the high-risk tasks quickly and accelerates the meta-training, as demonstrated in \cref{sec:experiments} and \cref{app:cem_res}.

On the other hand, this approach relies on the structure in the task space.
If the tasks do not have a natural structure like the ones in the bounded box, it is not trivial to define the parametric family of distributions.
This is the case in certain supervised meta learning problems.
For example, in the common meta-classification problem \citep{maml}, the task space consists of subsets of classes to classify. This is a discrete space without a trivial metric between tasks. Hence, it is difficult for the CEM to learn the $\alpha$ lowest-return tasks from a finite sample.
Thus, while RoML is applicable to supervised meta learning as well as MRL, certain task spaces require further adjustment, such as a meaningful embedding of the task space.
This challenge is left for future work.

%%%%%%%%%%%%%%%%%%%%%%%%%%%%%%%%%%%%%%%%%%%%%%%%%
%%%%%%%%%%%%%%%%%%%%%%%%%%%%%%%%%%%%%%%%%%%%%%%%%

\FloatBarrier
\section{Experiments: Detailed Settings and Results}
\label{app:detailed_results}

\subsection{Khazad Dum}
\label{app:kd}

\begin{quote}
\textit{At the end of the hall the floor vanished and fell to an unknown depth.
The outer door could only be reached by a slender bridge of stone, without
kerb or rail, that spanned the chasm with one curving spring of fifty feet.
It was an ancient defence of the Dwarves against any enemy that might capture
the First Hall and the outer passages. They could only pass across it in single
file.}~\citep{tolkien}
\end{quote}

% "At the end of the hall the floor vanished and fell to an unknown depth.
% The outer door could only be reached by a slender bridge of stone, without
% kerb or rail, that spanned the chasm with one curving spring of fifty feet.
% It was an ancient defence of the Dwarves against any enemy that might capture
% the First Hall and the outer passages. They could only pass across it in single
% file."
%     Tolkien, 1954
% KhazadDum: a gym environment where the agent needs to cross an abyss using either
% a near narrow bridge or a far wide bridge. The former is faster but has a higher
% probability of falling due to the control-noise.
% Shall the agent pass?

The detailed settings of the Khazad Dum environment presented in \cref{sec:kd} are as follows.
Every task is carried for $K=4$ episodes of $T=32$ times steps.
The return corresponds to the undiscounted sum of the rewards ($\gamma=1$).
Every time step has a cost of $1/T$ points if the L1-distance of the agent from the target is larger than 5; and for distances between 0 and 5, the cost varies linearly between 0 and $1/T$.
By reaching the destination, the agent obtains a reward of $5/T$, and has no more costs for the rest of the episode.
By falling to the abyss, the agent can no longer reach the goal and is bound to suffer a cost of $1/T$ for every step until the end of the episode.
Every step, the agent observes its location (represented using a soft one-hot encoding, similarly to \citet{cesor}) and chooses whether to move left, right, up or down. If the agent attempts to move into a wall, it remains in place.

The tasks are characterized by the rain intensity, distributed $\tau\sim Exp(0.1)$.
The rain only affects the episode when the agent crosses the bridge: then, the agent suffers from an additive normally-distributed action noise $\mathcal{N}(0,\tau^2)$, in addition to a direct damage translated into a cost of $3\cdot\tau$.
For RoML, the CEM is implemented over the exponential family of distributions $Exp(\phi)$ with $\phi_0=0.1$ and $\beta=0.05$. In this toy benchmark we use no regularization ($\nu=0$). 

In addition to the test returns throughout meta-training shown in \cref{fig:kd_trajs}, \cref{fig:kd} displays the final test returns at the end of the meta-training, over 30 seeds and 3000 test tasks per seed.

% \begin{figure}[h]
% % \vspace{-13pt}
% \centering
% \begin{subfigure}{.32\linewidth}
%   \centering
%   \includegraphics[width=1.\linewidth]{Figs/KD/kd_f_mean.png}
%   \caption{$Mean$}
%   \label{fig:kd_mean}
% \end{subfigure}
% \begin{subfigure}{.32\linewidth}
%   \centering
%   \includegraphics[width=1.\linewidth]{Figs/KD/kd_f_cvar.png}
%   \caption{$\mathtt{CVaR}_{0.01}$}
%   \label{fig:kd_cvar}
% \end{subfigure}
% \caption{\small Khazad-Dum: Mean and CVaR returns over 30 seeds and 3000 test tasks.
% % RoML achieves the best test CVaR, and VariBAD the best mean.
% }
% \label{fig:kd}
% \end{figure}

% \begin{figure}[h]
% % \vspace{-13pt}
% \centering
% \begin{subfigure}{.32\linewidth}
%   \centering
%   \includegraphics[width=1.\linewidth]{Figs/KD/kd_f2_mean.png}
%   \caption{$Mean$}
%   \label{fig:kd_mean}
% \end{subfigure}
% \begin{subfigure}{.32\linewidth}
%   \centering
%   \includegraphics[width=1.\linewidth]{Figs/KD/kd_f2_cvar.png}
%   \caption{$\mathtt{CVaR}_{0.01}$}
%   \label{fig:kd_cvar}
% \end{subfigure}
% \caption{\small Khazad-Dum: Mean and CVaR returns over 30 seeds and 3000 test tasks.
% % RoML achieves the best test CVaR, and VariBAD the best mean.
% }
% \label{fig:kd}
% \end{figure}

\begin{figure}[!h]
\centering
\includegraphics[width=.64\linewidth]{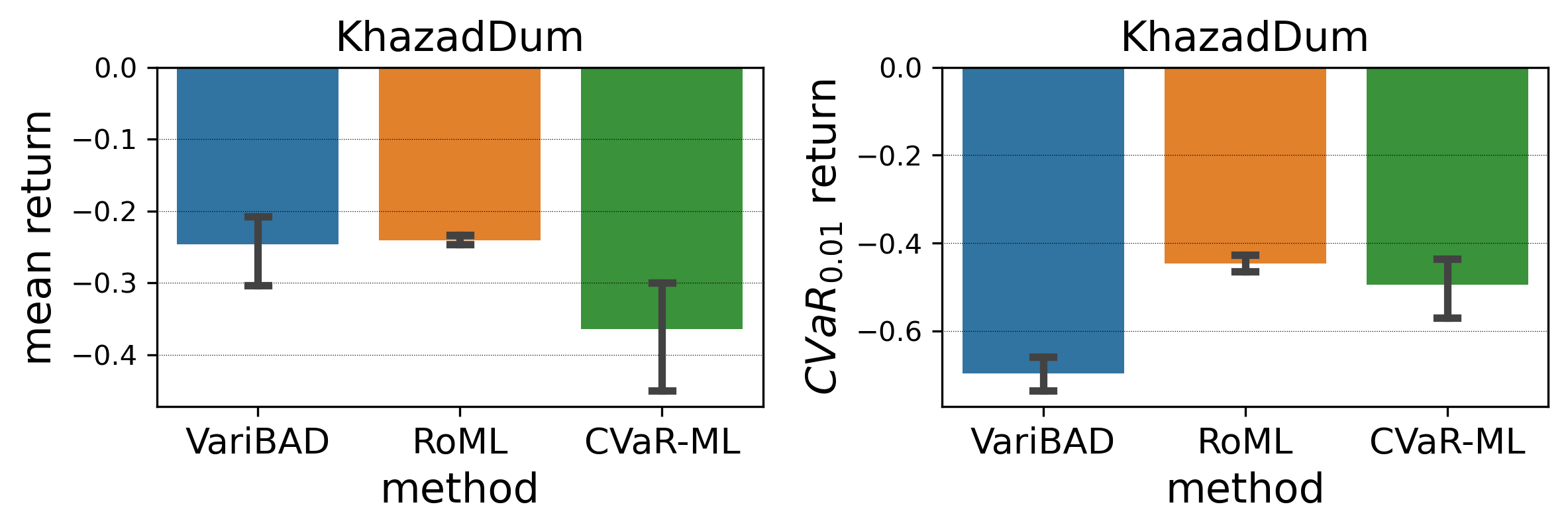}
\caption{\small Khazad-Dum: Mean and CVaR returns over 30 seeds and 3000 test tasks.}
\label{fig:kd}
\end{figure}

%%%%%%%%%%%%%%%%%%%%%%%%%%%%%%%%%%%%%

\FloatBarrier
\subsection{Continuous Control}
\label{app:mujoco}

In all the MuJoCo benchmarks introduced in \cref{sec:mujoco}, \ra{each task's meta-rollout consists of 2 episodes} $\times$ 200 time-steps per episode.
Below we describe the task distributions, as well as their parameterization for the CEM. In each benchmark, we used CEM quantile $\beta=0.2$ and regularization of $\nu=0.2$ samples per batch.
% With the VariBAD baseline, we meta-trained each agent over $6\cdot10^7$ frames in the \textit{Body} benchmarks and $3\cdot10^7$ in the other benchmarks. With the PEARL baseline, we meta-trained each benchmark for 500 iterations.
%
\begin{itemize}
    \item \textbf{HalfCheetah-Vel}: The original task distribution is uniform $\tau\sim U([0,7])$ in HalfCheetah (the task space $[0,7]$ was extended in comparison to \citet{varibad}, to create a more significant tradeoff between tasks). We rewrite it as $\tau=7\tilde{\tau},\ \tilde{\tau}\sim Beta(2\phi,2-2\phi)$ with $\phi_0=0.5$ (leading to $Beta(1,1)$, which is indeed the uniform distribution). The CEM learns to modify $\phi$. Notice that this parameterization satisfies $\mathbb{E}_{D_\phi}[\tilde{\tau}] = \phi$.
    \item \textbf{Humanoid-Vel}: Same as HalfCheetah-Vel, with task space $[0,2.5]$ instead of $[0,7]$.
    \item \textbf{Ant-Goal}: The target location is random within a circle of radius 5. We represent the target location in polar coordinates, and write $r \sim Beta(2\phi_1,2-2\phi_1)$ and $\theta \sim Beta(2\phi_2,2-2\phi_2)$ (up to multiplicative factors $5$ and $2\pi$). The original distribution parameter is $\phi_0=(0.5,0.5)$, and the CEM learns to modify it.
    \item \textbf{HalfCheetah-Mass, Humanoid-Mass, Ant-Mass}: The task $\tau\in[0.5,2]$ corresponds to the multiplicative factor of the body mass (e.g., $\tau=2$ is a doubled mass). The original task distribution is uniform over the log factor, i.e., $\log_2 \tau \sim U([-1,1])$. Again, we re-parameterize the uniform distribution as $Beta$, and learn to modify its parameter.
    \item \textbf{HalfCheetah-Body, Humanoid-Body, Ant-Body}: The 3 components of the task correspond to multiplicative factors of different physical properties, and each of them is distributed independently and uniformly in log, i.e., $\forall 1\le j\le3:\ \log_2 \tau_j \sim U([-1,1])$. We re-parameterize this as 3 independent $Beta$ distributions with parameters $\phi=(\phi_1,\phi_2,\phi_3)$.
    \item \textbf{HalfCheetah 10D-task}: Again, the task components correspond to multiplicative factors of different properties of the model. This time, there are 10 different properties (i.e., the task space is 10-dimensional), but each of them varies in a smaller range: $\log_2 \tau_j \sim U([-0.5,0.5])$. Each such property is a vector, and is multiplied by $\tau_j$ when executing the task $\tau$. The 10 properties are selected randomly, among all the variables of type \textit{float ndarray} in \textit{env.model}. We generate 3 such MRL environments -- with 3 different random sets of 10 task-variables each. Some examples for properties are inertia, friction and mass.
\end{itemize}

In the experiments, we rely on the official implementations of VariBAD~\citep{varibad} and PEARL~\citep{pearl}, both published under the MIT license.
CVaR-ML and RoML are implemented on top of these baseline, and their running times are indistinguishable from the baselines.
All experiments were performed on machines with Intel Xeon 2.2 GHZ CPU and NVIDIA's V100 GPU.
% All the experiments were run on DGX-1 machines with...
Each experiment (meta-training and testing) required 12-72 hours, depending on the environment and the baseline algorithm.

\cref{tab:mujoco_mean} and \cref{fig:mujoco_tasks2} present detailed results for our MuJoCo experiments, in addition to the results presented in \cref{sec:mujoco}.

\begin{table}[]
\caption{Mean return over 1000 test tasks, for different models and MuJoCo environments. Standard deviation is presented over 30 seeds. CVaR returns are displayed in \cref{tab:mujoco}.}
\centering
\label{tab:mujoco_mean}
\small\addtolength{\tabcolsep}{-1pt}
\begin{tabular}{|l|ccc|ccc|}
\toprule
\multicolumn{1}{|c|}{\multirow{2}{*}{}} & \multicolumn{3}{c|}{HalfCheetah}                              & \multicolumn{3}{c|}{HalfCheetah 10D-task}                        \\
\multicolumn{1}{|c|}{}                  & Vel               & Mass                & Body                & (a)                 & (b)                  & (c)                 \\ \midrule
CeSoR                                   & $-1316 \pm 18$    & $1398 \pm 31$       & $1008 \pm 34$       & $1222 \pm 23$       & $1388 \pm 20$       & $1274 \pm 32$        \\
PAIRED                                  & $-545 \pm 55$     & $662 \pm 30$        & $492 \pm 51$       & $551 \pm 53$         & $706 \pm 36$        & $561 \pm 65$  \\
CVaR-ML                                 & $-574 \pm 22$     & $113 \pm 8$        & $193 \pm 6$        & $263 \pm 15$        & $250 \pm 11$        & $192 \pm 5$       \\
PEARL                                   & $-534 \pm 15$     & $\pmb{1726 \pm 13}$ & $\pmb{1655 \pm 6}$ & $1843 \pm 9$       & $1866 \pm 13$       & $1425 \pm 6$        \\
VariBAD                                 & $\pmb{-82 \pm 2}$ & $1558 \pm 32$       & $1616 \pm 28$       & $\pmb{1893 \pm 6}$ & $\pmb{1984 \pm 67}$ & $\pmb{1617 \pm 12}$ \\
RoML (VariBAD)                          & $-95 \pm 3$       & $1581 \pm 32$       & $1582 \pm 21$       & $1819 \pm 8$       & $\pmb{1950 \pm 20}$        & $\pmb{1616 \pm 13}$       \\ %\hline
RoML (PEARL)                            & $-519 \pm 15$     & $1553 \pm 18$       & $1437 \pm 8$       & $1783 \pm 7$       & $1859 \pm 10$        & $1399 \pm 8$       \\ \midrule
\multirow{2}{*}{}                       & \multicolumn{3}{c|}{Humanoid}                                 & \multicolumn{3}{c|}{Ant}                                         \\
                                        & Vel               & Mass                & Body                & Goal                & Mass                 & Body                \\ \midrule
VariBAD                                 & $\pmb{880 \pm 4}$       & $\pmb{1645 \pm 22}$ & $\pmb{1678 \pm 17}$ & $\pmb{-229 \pm 3}$        & $\pmb{1473 \pm 3}$         & $\pmb{1476 \pm 1}$  \\
RoML (VariBAD)                          & $\pmb{883 \pm 4}$ & $1580 \pm 17$       & $1618 \pm 18$       & $\pmb{-224 \pm 3}$  & $\pmb{1475 \pm 2}$   & $1472 \pm 1$        \\ \bottomrule
\end{tabular}
\end{table}

\begin{figure}%[t]
% \vspace{-13pt}
\centering
\begin{subfigure}{.24\linewidth}
  \centering
  \includegraphics[width=1.\linewidth]{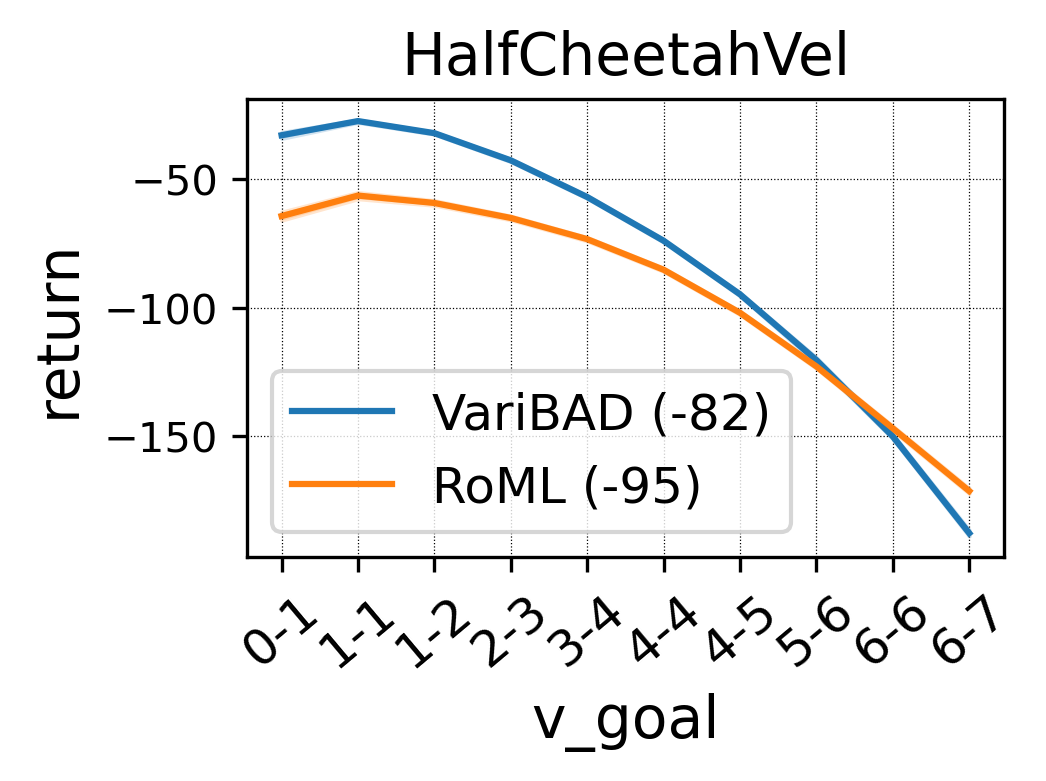}
  \caption{HalfCheetah-Vel}
  \label{fig:hcv_tasks}
\end{subfigure}
\begin{subfigure}{.74\linewidth}
  \centering
  \includegraphics[width=1.\linewidth]{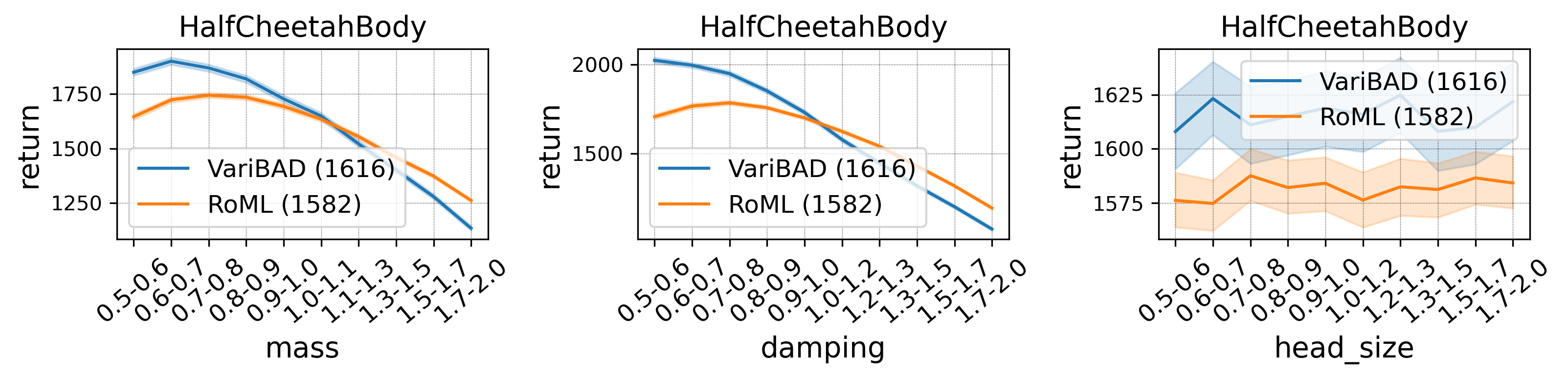}
  \caption{HalfCheetah-Body (mass, damping and head-size)}
  \label{fig:hcb_tasks}
\end{subfigure}
\begin{subfigure}{.24\linewidth}
  \centering
  \includegraphics[width=1.\linewidth]{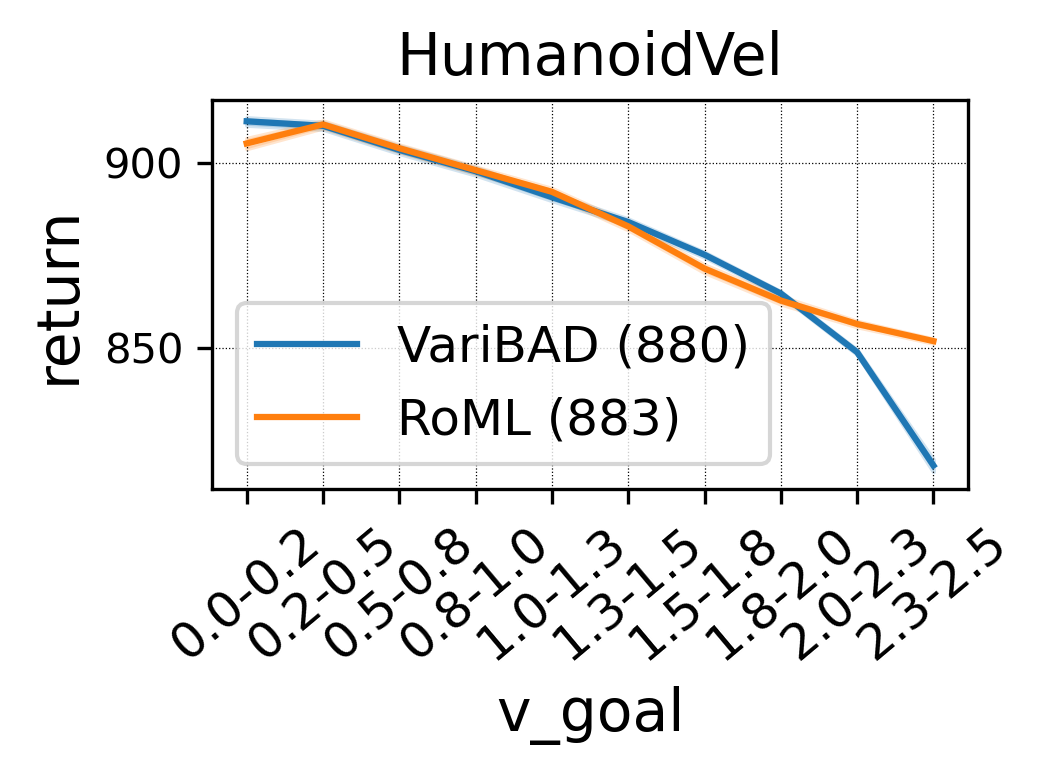}
  \caption{Humanoid-Vel}
  \label{fig:humv_tasks}
\end{subfigure}
\begin{subfigure}{.74\linewidth}
  \centering
  \includegraphics[width=1.\linewidth]{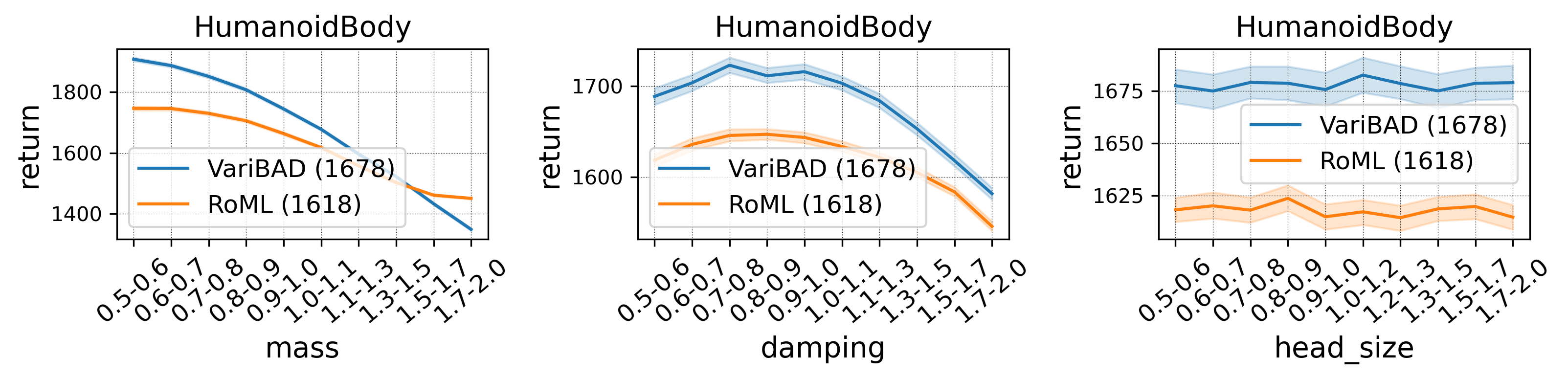}
  \caption{Humanoid-Body (mass, damping and head-size)}
  \label{fig:humb_tasks}
\end{subfigure}
\begin{subfigure}{.24\linewidth}
  \centering
  \includegraphics[width=1.\linewidth]{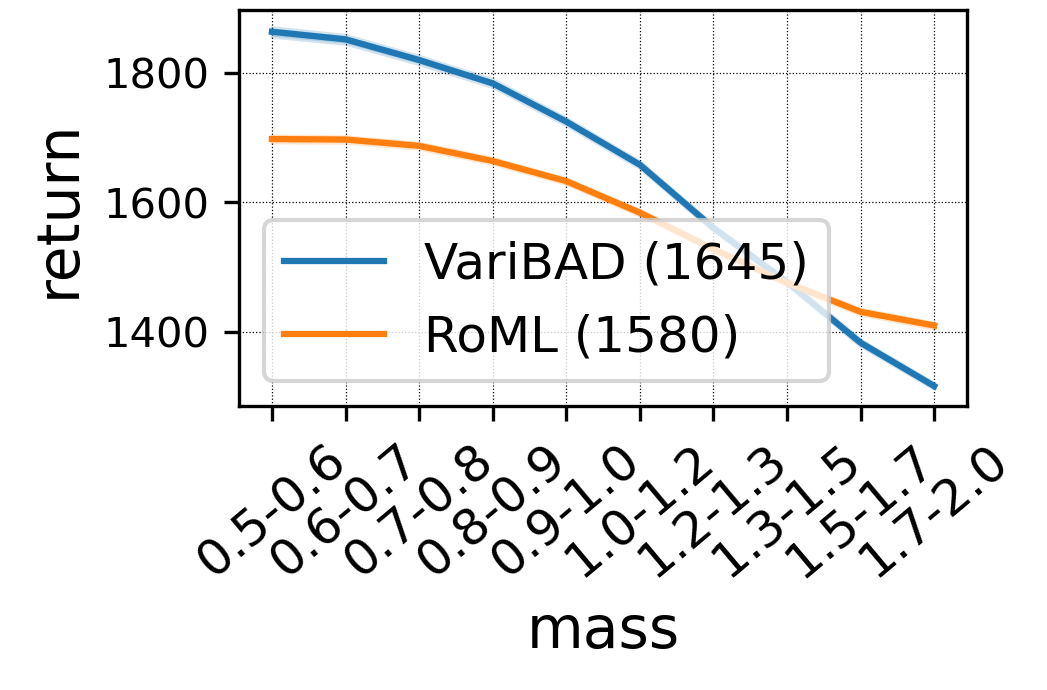}
  \caption{Humanoid-Mass}
  \label{fig:humm_tasks}
\end{subfigure}
\begin{subfigure}{.24\linewidth}
  \centering
  \includegraphics[width=1.\linewidth]{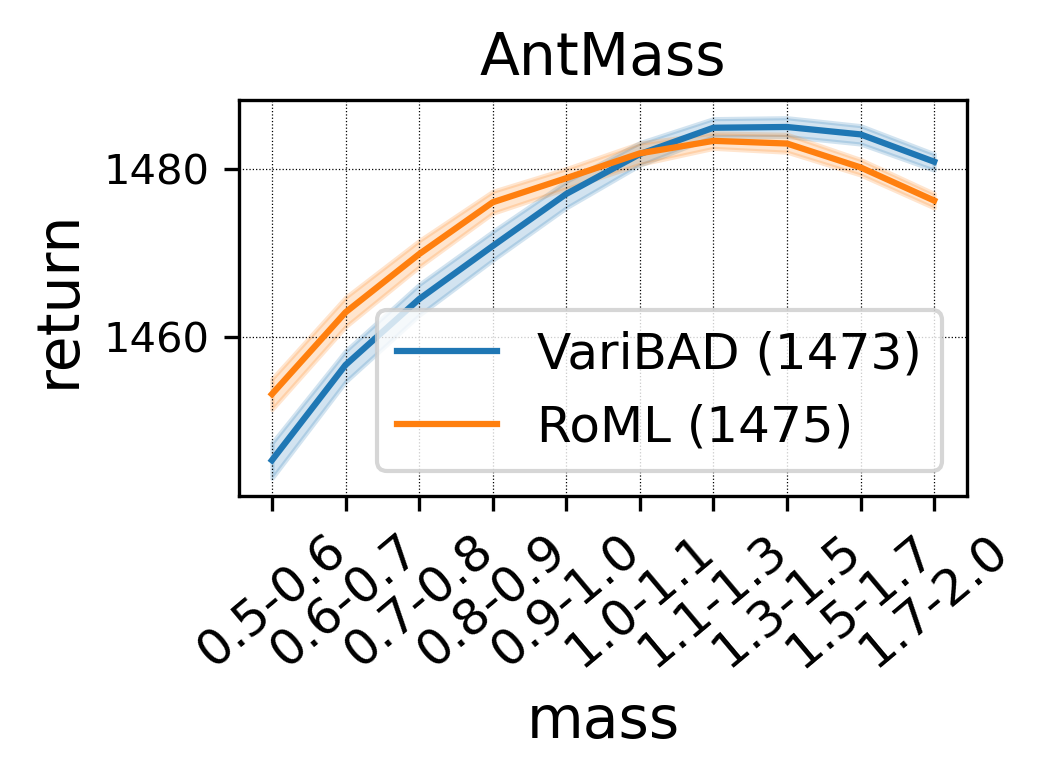}
  \caption{Ant-Mass}
  \label{fig:antm_tasks}
\end{subfigure}
\begin{subfigure}{.49\linewidth}
  \centering
  \includegraphics[width=1.\linewidth]{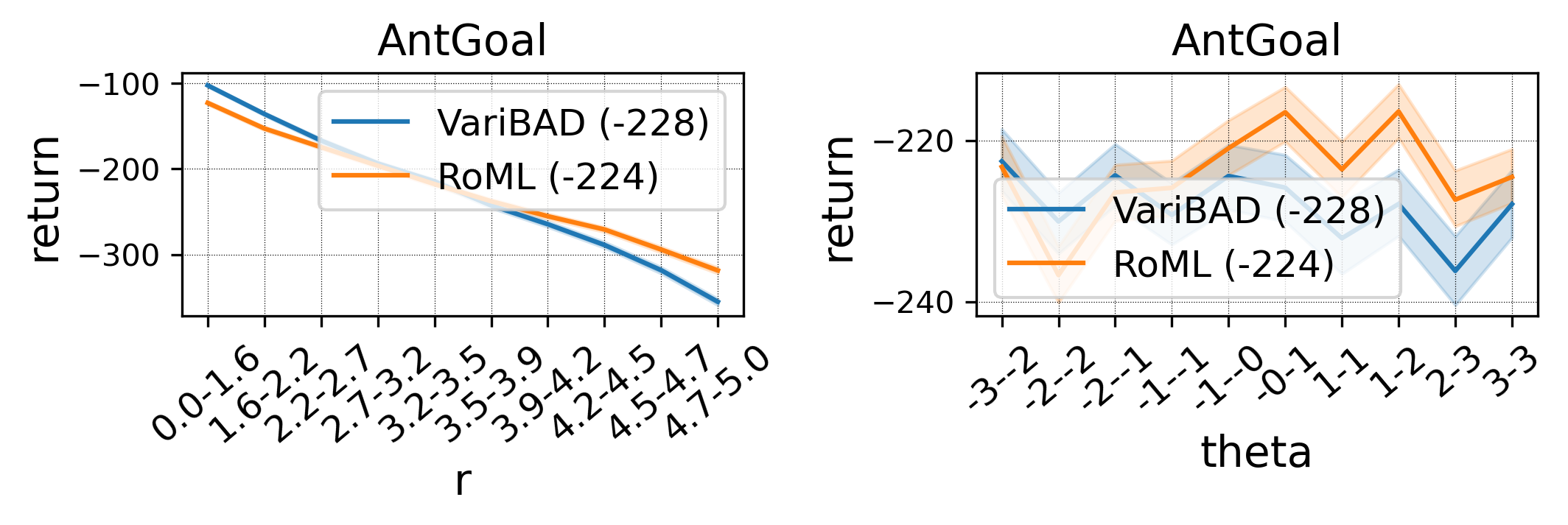}
  \caption{Ant-Goal}
  \label{fig:antg_tasks}
\end{subfigure}
\caption{\small Average return per range of tasks in the various MuJoCo environments (global average in parentheses). HalfCheetah-Mass is displayed in \cref{fig:hcm_tasks}.}
\label{fig:mujoco_tasks2}
\end{figure}

%%%%%%%%%%%%%%%%%%%%%%%%%%%%%%%%%%%%%

\FloatBarrier
\subsection{The Cross Entropy Method}
\label{app:cem_res}

We implemented RoML using the Dynamic Cross Entropy Method \cempypi{implementation} of \citet{cross_entropy_method}.
Below we concentrate results related to the CEM functionality in all the experiments:
\begin{itemize}
    \item \textbf{Learned sample distribution}: One set of figures corresponds to the learned sample distribution, as measured via $\phi$ throughout the meta-training.
    \item \textbf{Sample returns}: A second set of figures corresponds to the returns over the sampled tasks (corresponding to $D_\phi$): ideally, we would like them to align with the $\alpha$-tail of the reference returns (corresponding to $D_{\phi_0}$). Thus, the figures present the mean sample return along with the mean and CVaR reference returns (the references are estimated from the sample returns using Importance Sampling weights). In most figures, we see that the sample returns (in green) shift significantly from the mean reference return (in blue) towards the CVaR reference return (orange), at least for part of the training. Note that in certain environments, the distinction between difficulty of tasks can only be made after the agent has already learned a basic meaningful policy, hence the sample returns do not immediately deviate from the mean reference.
\end{itemize}

\begin{figure}[!h]
% \vspace{-13pt}
\centering
\begin{subfigure}{.28\linewidth}
  \centering
  \includegraphics[width=1.\linewidth]{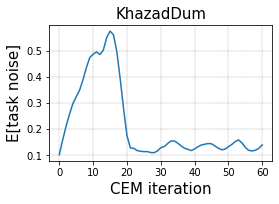}
  \caption{Learned sample distribution}
  %\label{fig:kd_noise}
\end{subfigure}
\begin{subfigure}{.38\linewidth}
  \centering
  \includegraphics[width=1.\linewidth]{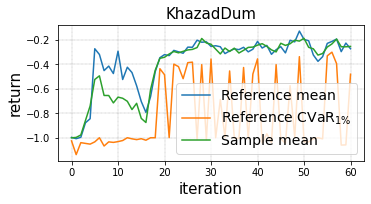}
  \caption{Sample returns}
  %\label{fig:kd_cem}
\end{subfigure}
\caption{\small The CEM in Khazad-Dum. Note that the effect of the CEM concentrates at the first half of the meta-training; once the meta-policy learns to focus on the long path, the agent becomes invariant to the sampled tasks, and the sampler gradually returns to the original task distribution $\phi\approx\phi_0$.}
\label{fig:kd_cem}
\end{figure}

\begin{figure}[!h]
% \vspace{-13pt}
\centering
\begin{subfigure}{.24\linewidth}
  \centering
  \includegraphics[width=1.\linewidth]{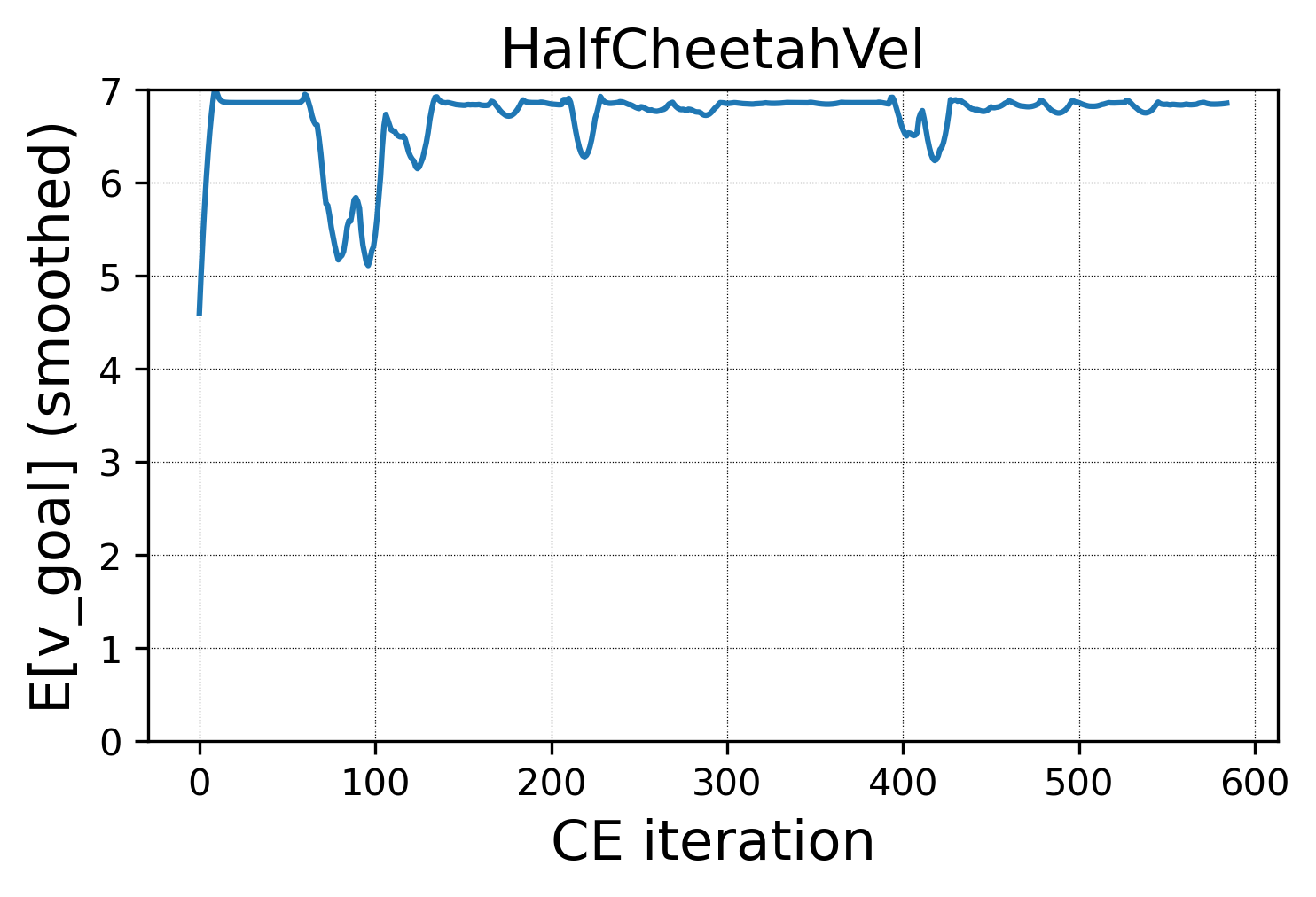}
  \caption{HalfCheetah-Vel}
  \label{fig:hcv_dist}
\end{subfigure}
\begin{subfigure}{.24\linewidth}
  \centering
  \includegraphics[width=1.\linewidth]{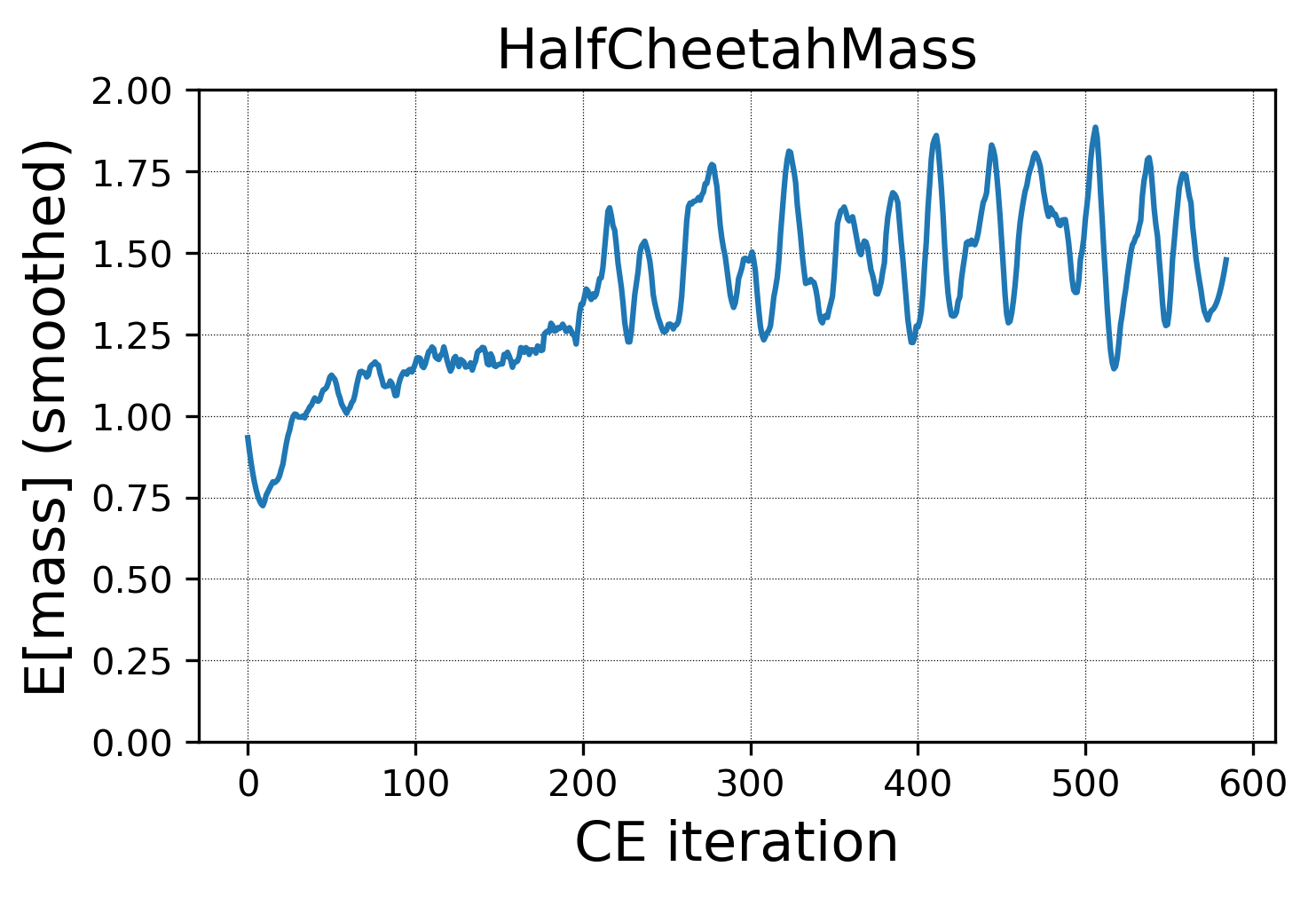}
  \caption{HalfCheetah-Mass}
  \label{fig:hcm_dist}
\end{subfigure}
\begin{subfigure}{.24\linewidth}
  \centering
  \includegraphics[width=1.\linewidth]{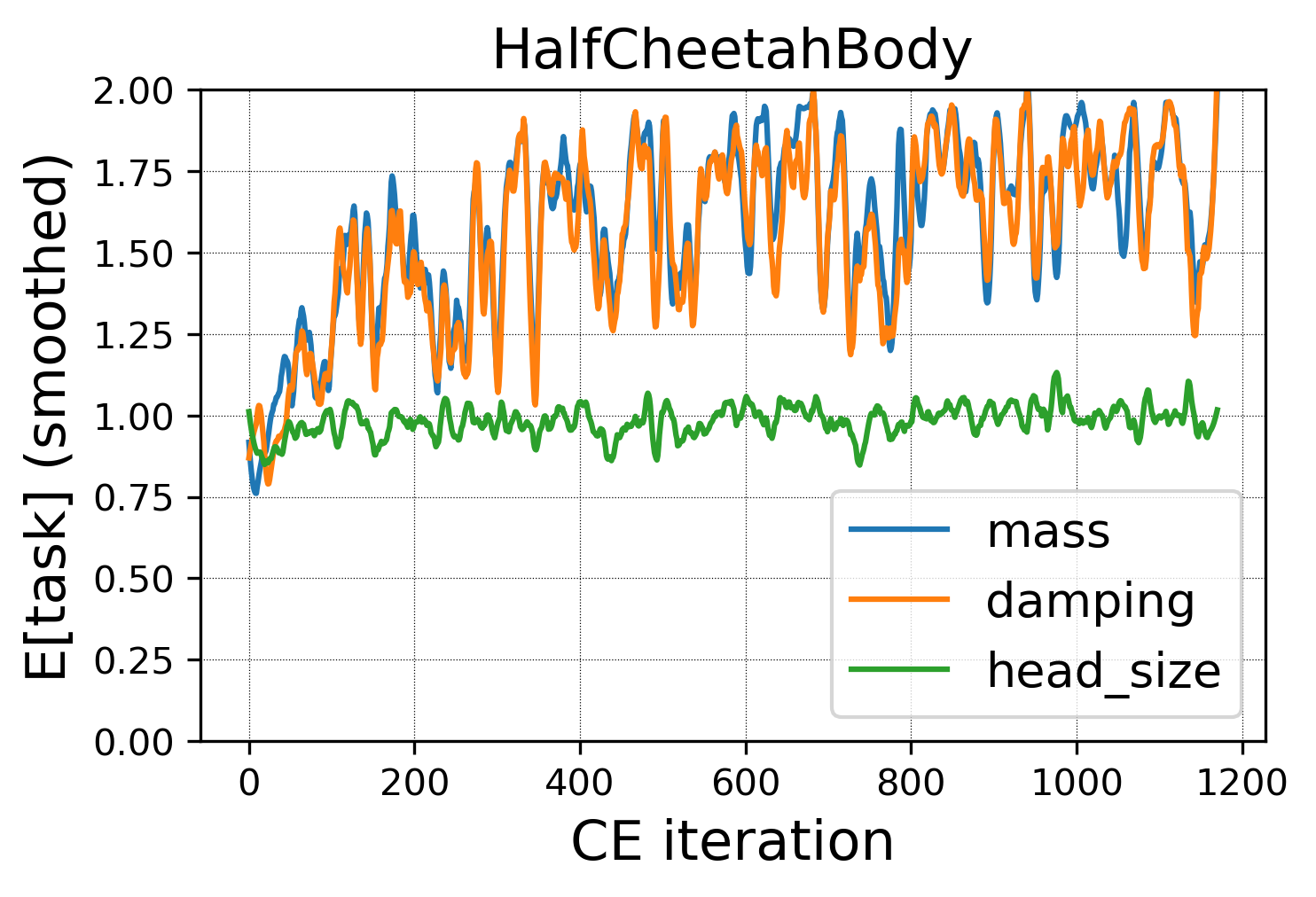}
  \caption{HalfCheetah-Body}
  \label{fig:hcb_dist}
\end{subfigure}
\begin{subfigure}{.24\linewidth}
  \centering
  \includegraphics[width=1.\linewidth]{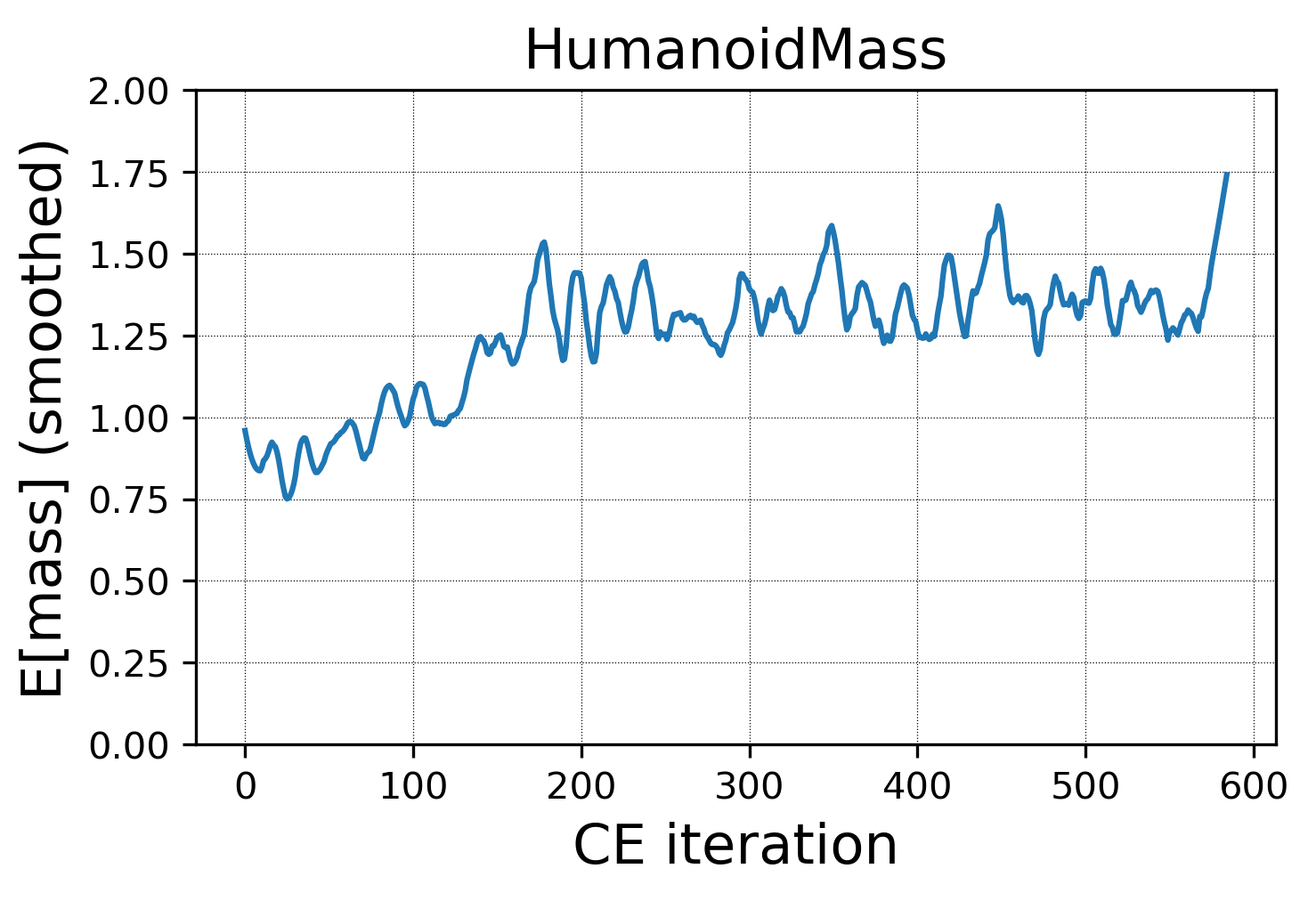}
  \caption{Humanoid-Mass}
  \label{fig:humm_dist}
\end{subfigure}
\caption{\small Learned sample distribution in MuJoCo benchmarks. Notice that for HalfCheetah-Body, the CEM has to control 3 different task parameters simultaneously.}
\label{fig:mujoco_dist}
\end{figure}

\begin{figure}[!h]
% \vspace{-13pt}
\centering
\begin{subfigure}{.24\linewidth}
  \centering
  \includegraphics[width=1.\linewidth]{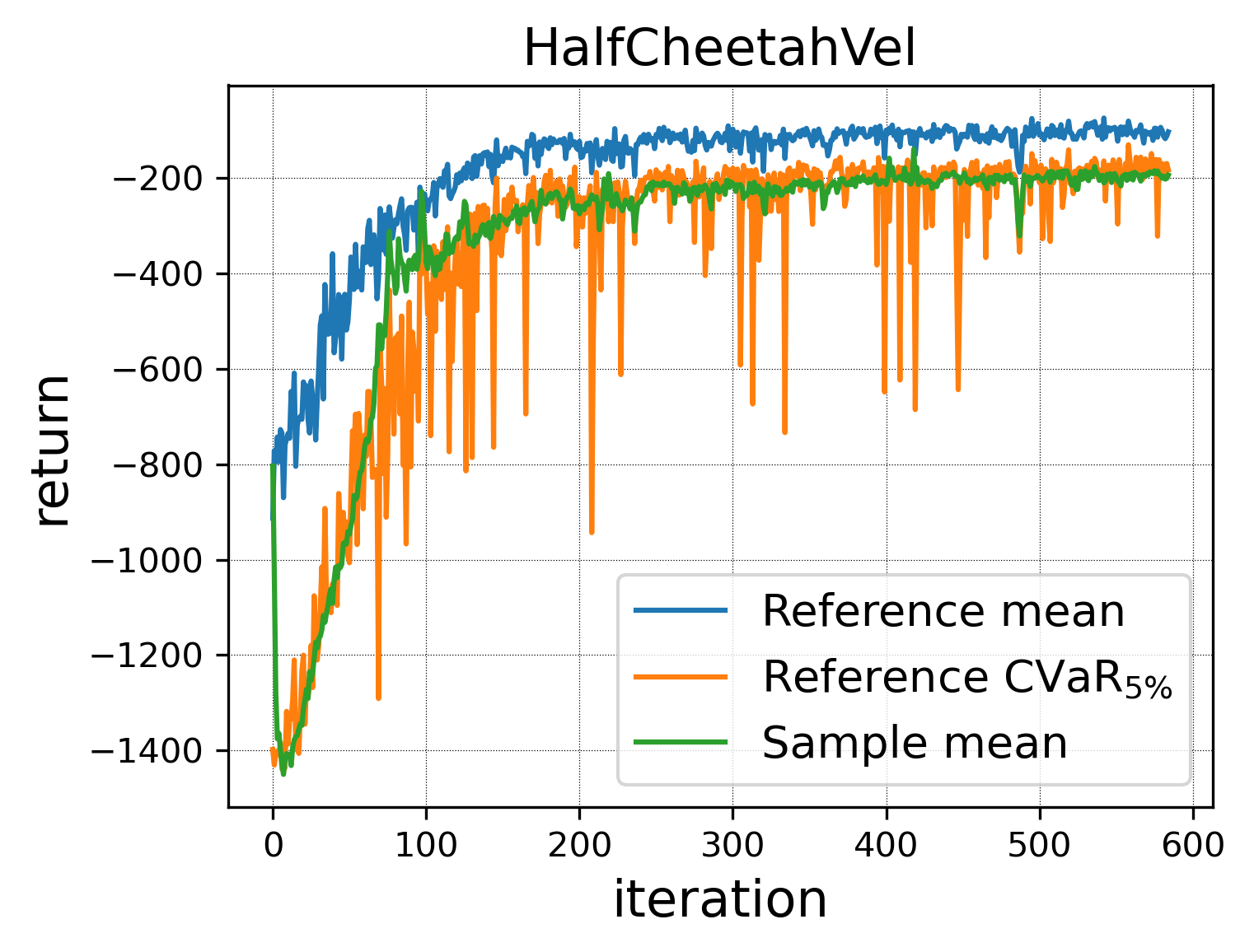}
  \caption{HalfCheetah-Vel}
  \label{fig:hcv_cem}
\end{subfigure}
\begin{subfigure}{.24\linewidth}
  \centering
  \includegraphics[width=1.\linewidth]{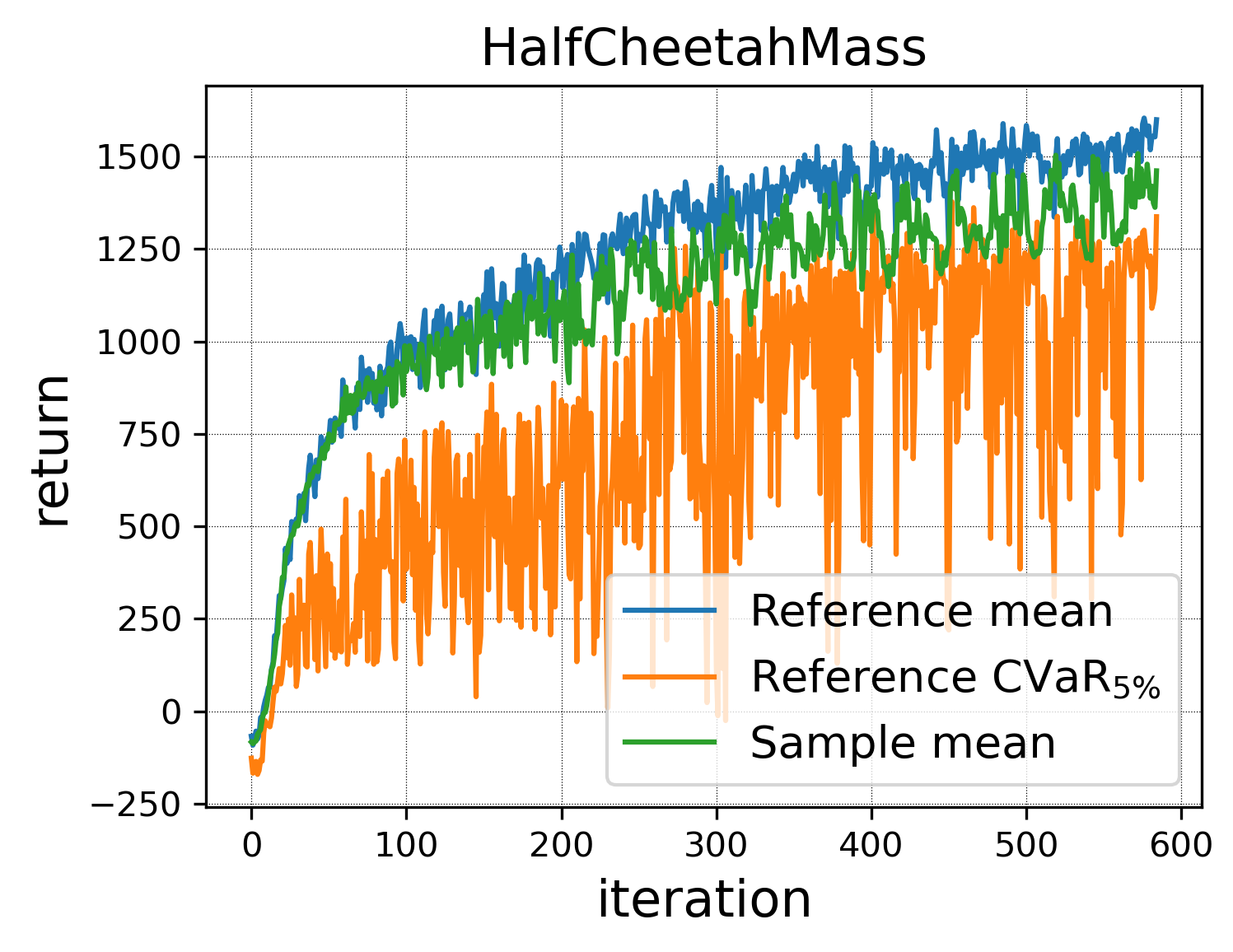}
  \caption{HalfCheetah-Mass}
  \label{fig:hcm_cem}
\end{subfigure}
\begin{subfigure}{.24\linewidth}
  \centering
  \includegraphics[width=1.\linewidth]{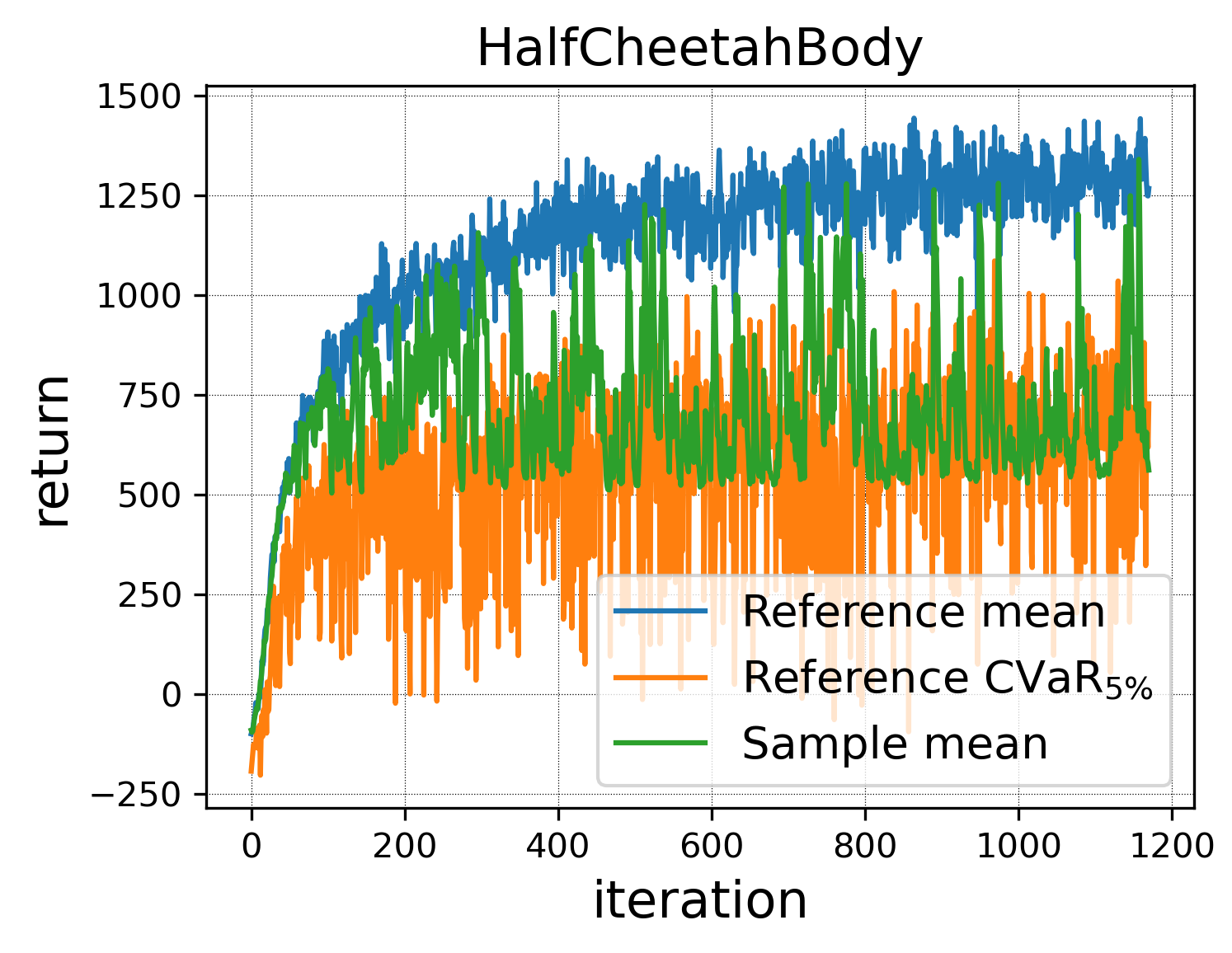}
  \caption{HalfCheetah-Body}
  \label{fig:hcb_cem}
\end{subfigure}
\begin{subfigure}{.24\linewidth}
  \centering
  \includegraphics[width=1.\linewidth]{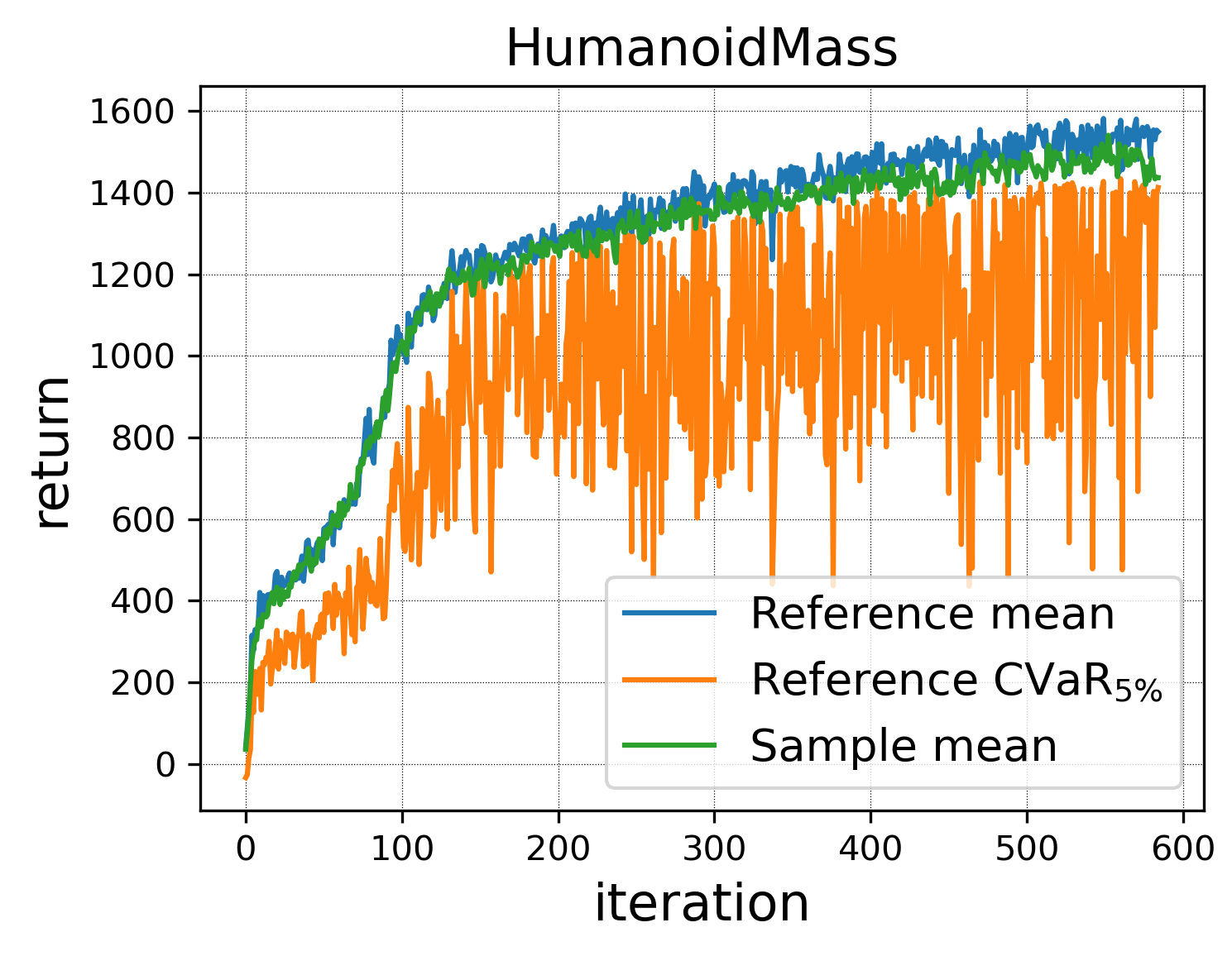}
  \caption{Humanoid-Mass}
  \label{fig:humm_cem}
\end{subfigure}
\caption{\small Sample returns in MuJoCo benchmarks.}
\label{fig:mujoco_cem}
\end{figure}

\begin{figure}[!h]
  \centering
  \includegraphics[width=1.\linewidth]{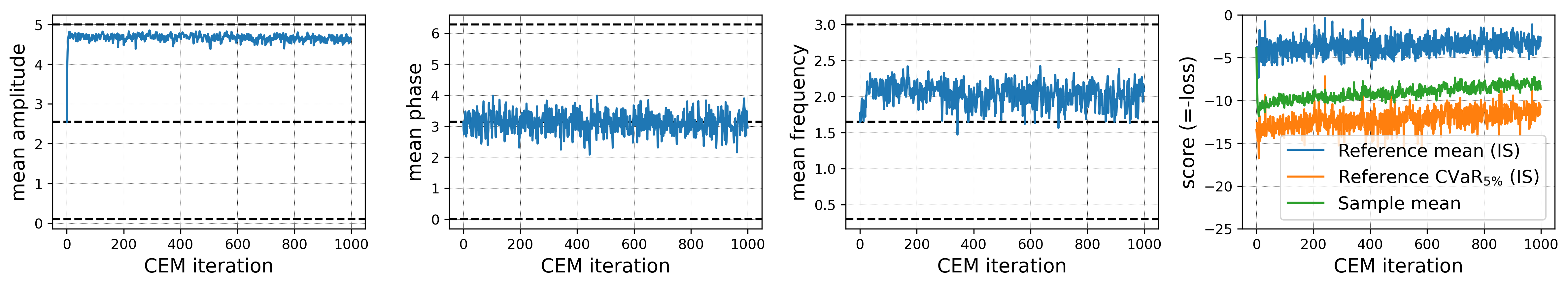}
  \caption{\small Sine Regression: The 3 left figures show the learned sample distribution, corresponding to average amplitudes, phases and frequencies. The CEM immediately learns that the amplitude has the strongest effect on test loss, whereas frequency has a moderate effect and phase has none. The right figure shows in green the mean sample scores (in supervised learning these are the negative losses instead of the returns).}
  \label{fig:sine_cem}
\end{figure}

%%%%%%%%%%%%%%%%%%%%%%%%%%%%%%%%%%%%%%%%%%%%%%%%%

\FloatBarrier

\subsection{Ablation Test: RoML with a Naive Sampler}
\label{app:ablation_cem}

RoML learns to be robust to high-risk tasks by using the CEM to over-sample them throughout training.
The dynamic CEM presented in \cref{sec:roml} is designed to identify and sample high-risk tasks from a possibly-infinite task-space, where the task values change with the evolving policy throughout training.

In this section, we conduct an ablation test to demonstrate the importance of the CEM to this task.
To that end, we implement a naive adversarial task sampler.
The first $\mathcal{M}=100$ tasks are sampled from the original distribution $D$, and the naive sampler memorizes the $\alpha \mathcal{M}$ lowest-return tasks, and samples randomly from them for the rest of the training.

As displayed in \cref{tab:ablation}, switching to the naive sampler decreases the CVaR returns significantly.

\begin{table}[h]
\caption{Ablation test: CVaR$_{0.05}$ return, compared to the naive sampler baseline.}
\label{tab:ablation}
\centering
\small\addtolength{\tabcolsep}{-1pt}
\begin{tabular}{|l|c|ccc|}
\toprule
% \multicolumn{1}{|c|}{\multirow{2}{*}{$\pmb{CVaR_{0.05}}$}}
& \multicolumn{1}{c|}{HalfCheetah}                              & \multicolumn{3}{c|}{HalfCheetah 10D-task}                      \\
\multicolumn{1}{|c|}{}                               & Body               & (a)                 & (b)                 & (c)                \\ \midrule
VariBAD                & $835 \pm 30$       & $1126 \pm 6$       & $1536 \pm 39$       & $988 \pm 13$       \\
Naive sampler   & $839 \pm 20$ & $1056 \pm 34$ & $1340 \pm 57$ & $978 \pm 12$ \\ %\hline
RoML (VariBAD)   & $935 \pm 17$ & $1227 \pm 13$ & $1697 \pm 24$ & $999 \pm 20$ \\ %\hline
% \hline
% \multicolumn{1}{|c|}{\multirow{2}{*}{$\pmb{CVaR_{0.05}}$}}
\bottomrule
\end{tabular}
\end{table}

%%%%%%%%%%%%%%%%%%%%%%%%%%%%%%%%%%%%%%%%%%%%%%%%%

\FloatBarrier
\section{Supervised Meta-Learning}
\label{app:sine}
% \label{sec:sine}

Below we provide the complete details for the toy Sine Regression experiment of \cref{sec:sine}.
% Our work focuses on robustness to difficult tasks in MRL.
% However, the concept of training on harder data to improve robustness, as embodied in RoML, is applicable beyond the scope of RL.
% As a preliminary proof-of-concept, we apply RoML to a toy supervised meta-learning problem of sine regression, based  on \citet{maml}.
The input in the problem is $x\in[0,2\pi)$, the desired output is $y=A\sin(\omega x + b)$, and the task is defined by the parameters $\tau=(A,b,\omega)$, distributed uniformly over $\Omega = [0.1,5]\times[0,2\pi]\times[0.3,3]$.
Similarly to \citet{maml}, the model is fine-tuned  for each task via a gradient-descent optimization step over 10 samples $\{(x_i,y_i)\}_{i=1}^{10}$, and is tested on another set of 10 samples. The goal is to find model weights that adapt quickly to new task data. % -- in expectation over tasks.

CVaR-ML and RoML are implemented with robustness level of $\alpha=0.05$, on top of MAML.
For the CEM of RoML, we re-parameterize the uniform task distribution using Beta distributions (see CEM details below).

As shown in \cref{fig:sine_cem}, RoML learns to focus on tasks (sine functions) with high amplitudes and slightly increased frequencies, without changing the phase distribution.
\cref{fig:sine} displays the test losses over 30 seeds, after meta-training for 10000 tasks.
Similarly to the MRL experiments, again RoML achieves better CVaR losses than both CVaR-ML and the baseline.

% In this problem, the task space $\Omega$ induces a natural structure on the tasks, as we can tell which tasks are ``close'' to each other. The CEM takes advantage of this structure to characterize high-risk tasks -- even though it can never observe the infinitely many possible tasks.
% The applicability of CEM to supervised learning in general is discussed in \cref{app:cem_discussion}.

% Notice that in this problem, the task space $\Omega$ corresponds to a bounded box. This induces a natural structure on the tasks, as we can tell which tasks are ``close'' to each other in each dimension. As demonstrated above, the CEM takes advantage of this structure to characterize high-risk tasks -- even though it can never observe the infinitely many possible tasks.
% However, certain supervised meta learning problems do not have such a naturally-structured task space. This may pose difficulties for the CEM, as discussed in \cref{app:cem_discussion}.

\begin{figure}[h]
% \vspace{-13pt}
\centering
\hspace{-10pt}
\begin{subfigure}{.36\columnwidth}
  \centering
  \includegraphics[width=1.\linewidth]{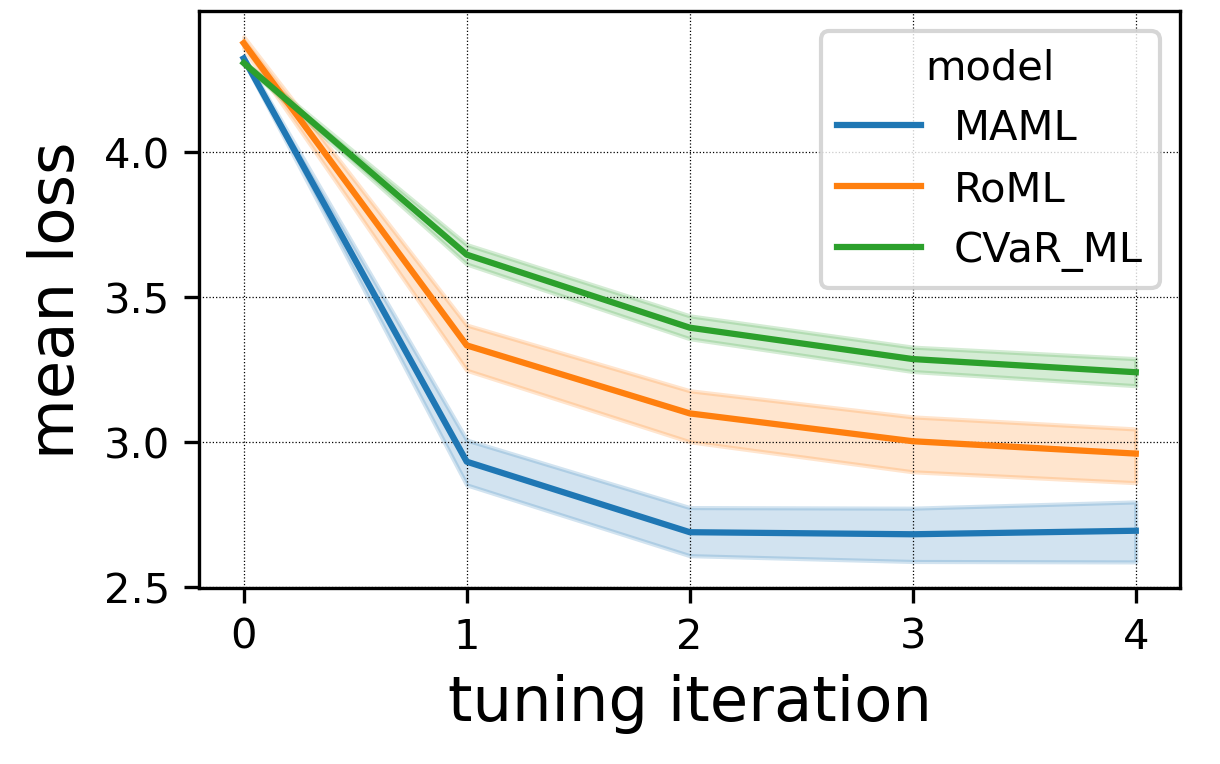}
  \caption{$\mathrm{Mean}$}
  \label{fig:sine_mean}
\end{subfigure}
\hspace{-7pt}
\begin{subfigure}{.36\columnwidth}
  \centering
  \includegraphics[width=1.\linewidth]{Figs/Sine/sine_cvar.png}
  \caption{$\mathtt{CVaR}_{0.05}$}
  \label{fig:sine_cvar}
\end{subfigure}
\hspace{-10pt}
\caption{\small Sine Regression: Mean and CVaR losses over 10000 test tasks, against the number of tuning gradient-steps at test time. The 95\% confidence intervals are calculated over 30 seeds.}
\label{fig:sine}
\end{figure}

\paragraph{CEM implementation details:}
In comparison to \citet{maml}, we added the sine frequency as a third parameter in the task space, since the original problem was to simplistic to pose a mean/CVaR tradeoff.
The tasks are distributed uniformly, and we reparameterize them for the CEM using the Beta distribution, similarly to \cref{app:mujoco}: $\tau=\{\tau_j\}_{j=1}^3,\ \tau_j \sim Beta(2\phi_j, 2-2\phi_j)$.
On top of this, we add a linear transformation from the Beta distribution domain $[0,1]$ to the actual task range ($[0.1,5]$ for amplitude, $[0,2\pi]$ for phase and $[0.3,3]$ for frequency).
Note that the original uniform distribution is recovered by $\phi_0=(0.5,0.5,0.5)$.
The parameter $\phi_j\in[0,1]$, which is controlled by the CEM, equals the expected task $\mathbb{E}_{\tau\sim D_\phi}[\tau_j]$.
The other hyper-parameters of RoML are set to $\beta=0.2$ (CEM quantile) and $\nu=0$ (no regularization).

%%%%%%%%%%%%%%%%%%%%%%%%%%%%%%%%%%%%%%%%%%%%%%%%%
%%%%%%%%%%%%%%%%%%%%%%%%%%%%%%%%%%%%%%%%%%%%%%%%%

\end{document}